\documentclass[11pt]{article}
\usepackage{latexsym,amssymb,amsmath} 
\usepackage{path}
\usepackage{url}
\usepackage{verbatim}

\usepackage{natbib}
\usepackage{amsfonts}
\usepackage{amsfonts}
\usepackage{amsfonts}
\usepackage{array}
\usepackage{mathrsfs}
\usepackage{amsfonts}
\usepackage{amsmath}
\usepackage{amssymb}
\usepackage{amsthm}
\usepackage{bm}
\usepackage{appendix}

\usepackage{graphicx} 
\usepackage{subfigure}

\usepackage[lined,boxed,ruled, commentsnumbered]{algorithm2e}


\numberwithin{equation}{section}
\numberwithin{table}{section}
\numberwithin{figure}{section}

\newtheorem{lemma}{Lemma}
\newtheorem{theorem}{Theorem}
\newtheorem{corollary}{Corollary}
\newtheorem{proposition}{Proposition}

\newtheorem{remark}{Remark}
\newtheorem{assumption}{Assumption}

\newcommand{\supp}{\text{supp}}
\newcommand{\argmin}{\mathop{\arg\min}}

\newcommand{\Tr}{\text{Tr}}

\newcommand{\s}{\setminus s}

\newcommand{\var}{\text{Var}}
\newlength{\fwtwo} 

\DeclareGraphicsExtensions{.jpg}

\title{Learning Pairwise Graphical Models with Nonlinear Sufficient Statistics}
\author{
  Xiao-Tong Yuan$^{1,2}$,\ \  Ping Li$^{2,3}$, \ \ Tong Zhang$^{2}$\\\\
  1. Department of Statistical Science, Cornell University \\
  Ithaca, New York, 14853, USA \\
  \and
  2. Department of Statistics \& Biostatistics, Rutgers University \\
  Piscataway, New Jersey, 08854, USA \\
  \and
  3. Department of Computer Science, Rutgers University \\
  Piscataway, New Jersey, 08854, USA \\  \\
  E-mail: \{\texttt{xtyuan1980@gmail.com}, \texttt{pingli@stat.rutgers.edu}, \texttt{tzhang@stat.rutgers.edu}\}
  }
\date{}

\begin{document}

\maketitle

\begin{abstract}
We investigate a generic problem of learning pairwise exponential
family graphical models with pairwise sufficient statistics defined
by a global mapping function, e.g., Mercer kernels. This subclass of
pairwise graphical models allow us to flexibly capture complex
interactions among variables beyond pairwise product. We propose two
$\ell_1$-norm penalized maximum likelihood estimators to learn the
model parameters from i.i.d. samples. The first one is a joint
estimator which estimates all the parameters simultaneously. The
second one is a node-wise conditional estimator which estimates the
parameters individually for each node. For both estimators, we show
that under proper conditions the extra flexibility gained in our
model comes at almost no cost of statistical and computational
efficiency. We demonstrate the advantages of our model over
state-of-the-art methods on synthetic and real datasets.
\end{abstract}


\newpage


\section{Introduction}
As an important class of statistical models for exploring the
interrelationship among a large number of random variables,
undirected graphical models (UGMs) have enjoyed popularity in a wide
range of scientific and engineering domains, including statistical
physics, computer vision, data mining, and computational biology.
Let $X = [X_1,...,X_p]$ be a $p$-dimensional random vector with each
variable $X_i$ taking values in a set $\mathcal {X}$. Suppose
$G=(V,E)$ is an undirected graph consists of a set of vertices
$V=\{1,...,p\}$ and a set of unordered pairs $E$ representing edges
between the vertices. The UGMs over $X$ corresponding to $G$ are a
set of distributions which satisfy Markov independence assumptions
with respect to the graph $G$: $X_s$ is independent of $X_t$ given
$\{X_u : u \neq s, t\}$ if and only if $(s, t) \notin E$. According
to the Hammersley-Clifford theorem~\citep{Clifford-1990}, the
general form for a (strictly positive) probability density encoded
by an undirected graph $G$ can be written as the following
exponential family distribution~\citep{Wainwright-2008}:
\[
\mathbb{P}(X;\theta) \propto \exp\left\{ \sum_{c \in
\texttt{Cliques}(G)} \theta_c f_c(X_c)\right\},
\]
where the sum is taken over all cliques, or fully connected subsets
of vertices of the graph $G$, $\{f_c\}$ are the clique-wise
sufficient statistics and $\theta = \{\theta_c\}$ are the weights
over the sufficient statistics. Learning UGMs from data within this
exponential family framework can be reduced to estimating the
weights $\theta$. Particularly, the cliques of \emph{pairwise} UGMs
consist of the set of nodes $V$ and the set of edges $E$, so that
\begin{equation}\label{prob:ggm_pairwise}
\mathbb{P}(X; \theta) \propto \exp\left\{ \sum_{s \in V} \theta_s
f_s(X_s) + \sum_{(s,t) \in E} \theta_{st} f_{st} (X_s, X_t)\right\}.
\end{equation}
In such a pairwise model, $(X_s, X_t)$ are conditionally independent
(given the rest variables) if and only if the weight $\theta_{st}$
is zero. A fundamental issue that arises is to specify
sufficient statistics, i.e., $\{f_s(X_s),
f_{st}(X_s,X_t)\}$, for modeling the interactions among variables. The
most popular instances of pairwise UGMs are Gaussian graphical
models (GGMs) and Ising (or Potts) models. GGMs use the node-wise
values and pairwise product of variables, i.e., $\{X_s, X_sX_t\}$,
as sufficient statistics and these are useful for modeling
real-valued
data~\citep{Speed-1986,Banerjee-2008,Rothman-2008,Yuan-Lin-2007}.
However, the multivariate normal distributional assumption imposed
by GGMs is quite stringent because this implies the marginal
distribution of any variable must also be Gaussian. In the case of binary or finite
nominal discrete random variables, Ising models are popular choices which also use pairwise product as
sufficient statistics to define the interactions among
variables~\citep{Ravikumar-AoS-2010,Jalali-AISTAT-2011}. This
subclass of models, however, are not suitable for modeling
count-valued variables such as non-negative integers. To find a
broader class of parametric graphical models,
\citet{Yang-GMGLM-NIPS-2012,Yang-EFGM-2013} proposed
\emph{exponential family graphical models} (EFGMs) as a unified
framework to learn UGMs with node-wise conditional distributions
arising from generalized linear models (GLMs). The distribution of
EFGMs is given by
\begin{equation}\label{equat:efgm_distr}
\mathbb{P}(X;\theta) \propto \exp\left\{ \sum_{s \in V}\theta_s
f(X_s) + \sum_{(s,t) \in E} \theta_{st} X_sX_t \right\},
\end{equation}
where $f(\cdot)$ is the base measure function which defines the
node-wise sufficient statistics. It is a special case of
distribution~\eqref{prob:ggm_pairwise} with $f_s(X_s) \equiv f(X_s)$
and $f_{st}(X_s,X_t)=X_sX_t$. An important merit of this model is
its flexibility in deriving multivariate graphical model
distributions from univariate exponential family distributions, such
as the Gaussian, binomial/multinomial, Poisson, exponential
distributions, etc..

\subsection{Motivation}
It is noteworthy that the extra gain of
flexibility in EFGMs mostly attributes to the node-wise base measure
$f(\cdot)$ which characterizes the node-conditional distributions.
The pairwise sufficient statistics, however, are still the pairwise
product as used for GGMs and Ising models. This is clearly
restrictive in the scenarios where the underlying pairwise
interactions of variables could be highly nonlinear. To
illustrate this restriction, we consider a
special case of the distribution~\eqref{prob:ggm_pairwise} in which
$f_s(X_s)\equiv 0$ and $f_{st}(X_s,X_t) = \exp\{|X_s - X_t|^2\}$,
i.e.,
\begin{equation}\label{equat:example_distr}
\mathbb{P}(X;\theta) \propto \exp\left\{\sum_{(s, t) \in E}
\theta_{st} \exp\{|X_s - X_t|^2\}\right\}.
\end{equation}
Assume that the underlying graph has a block structure as shown in
Figure~\ref{fig:example}(a) with parameters $\theta_{st} = 1$ for
connected pairs $(X_s,X_t)$. Let $p=50$ and each variate $X_s$ take
values in the real interval $[-10,10]$. Using Gibbs sampling, we
generate $10$ data samples from this graphical model.
Figure~\ref{fig:example}(b) shows the recovered graph structure by
fitting the data with the GGMs~\eqref{equat:efgm_distr}. It can be
clearly seen that GGMs fail when applied to this synthetic data with
highly nonlinear interactions among variables. This example
motivates us to investigate an important subclass of pairwise
graphical models in which the underlying exponential family employs
sufficient statistics beyond pairwise product.

\begin{figure}[h!]
\centering \subfigure[Truth structure]{
\includegraphics[width=50mm]{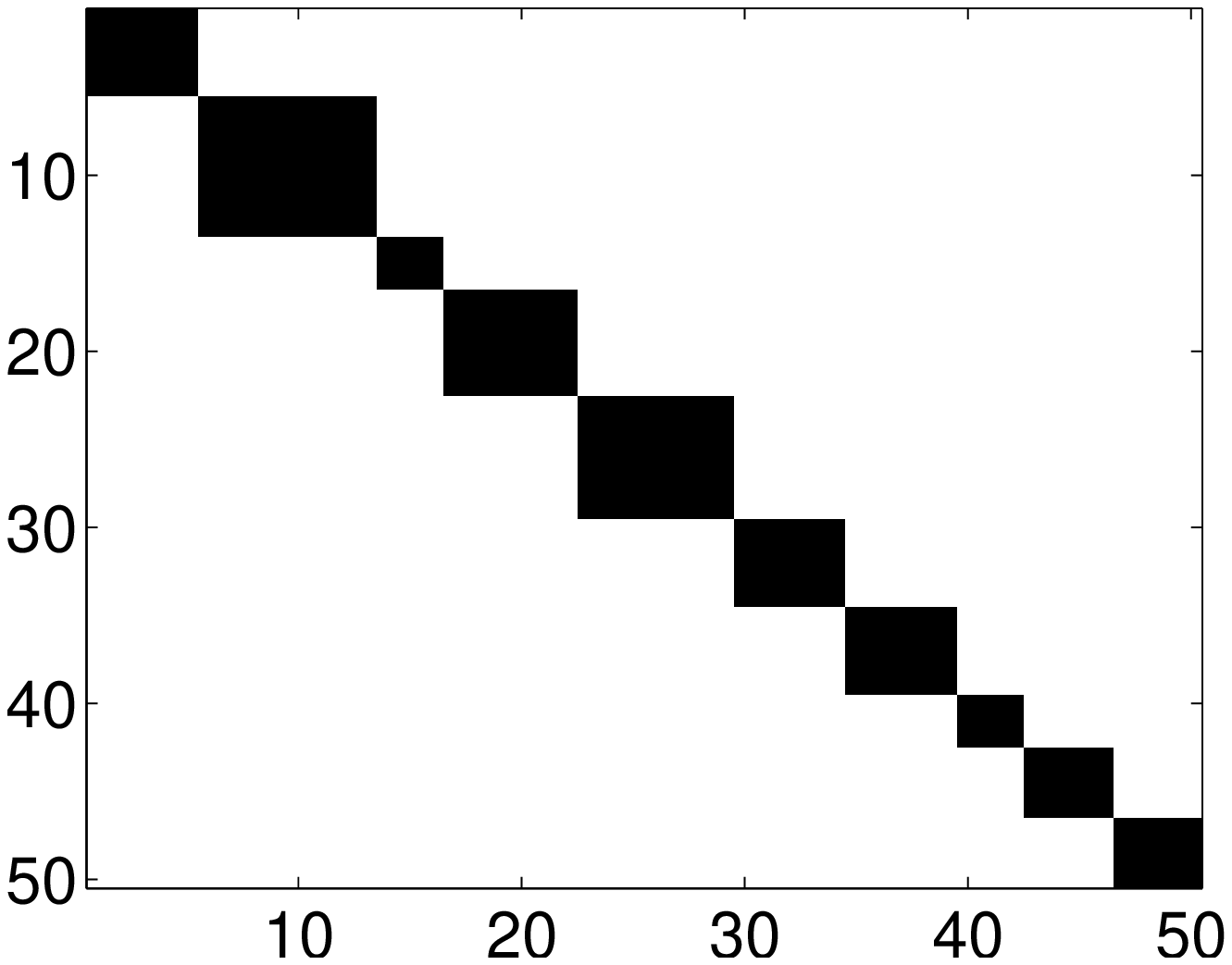}\label{fig:example_gt}
} \subfigure[GGMs]{
\includegraphics[width=50mm]{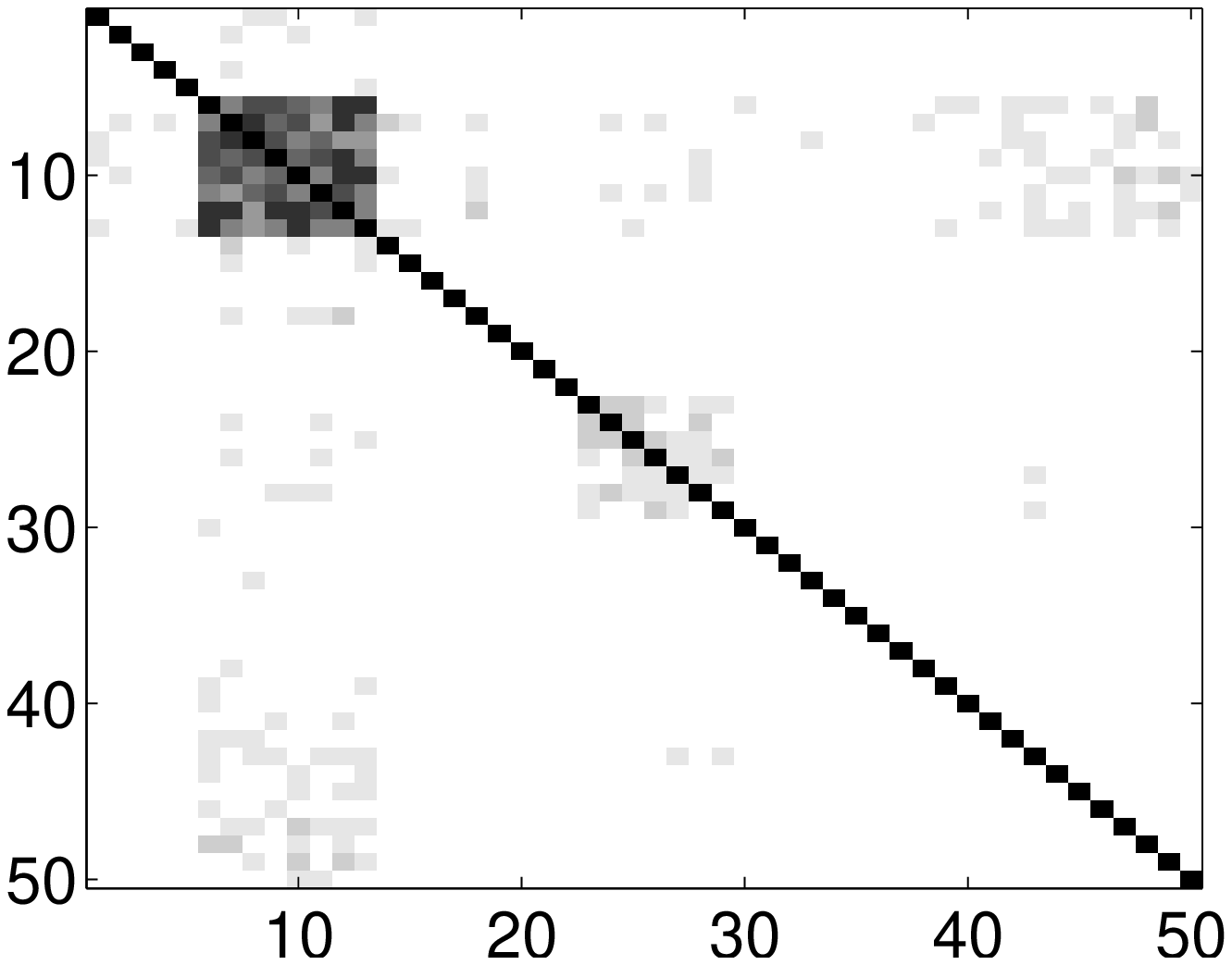}\label{fig:example_ggms}
} \subfigure[Our approach]{
\includegraphics[width=50mm]{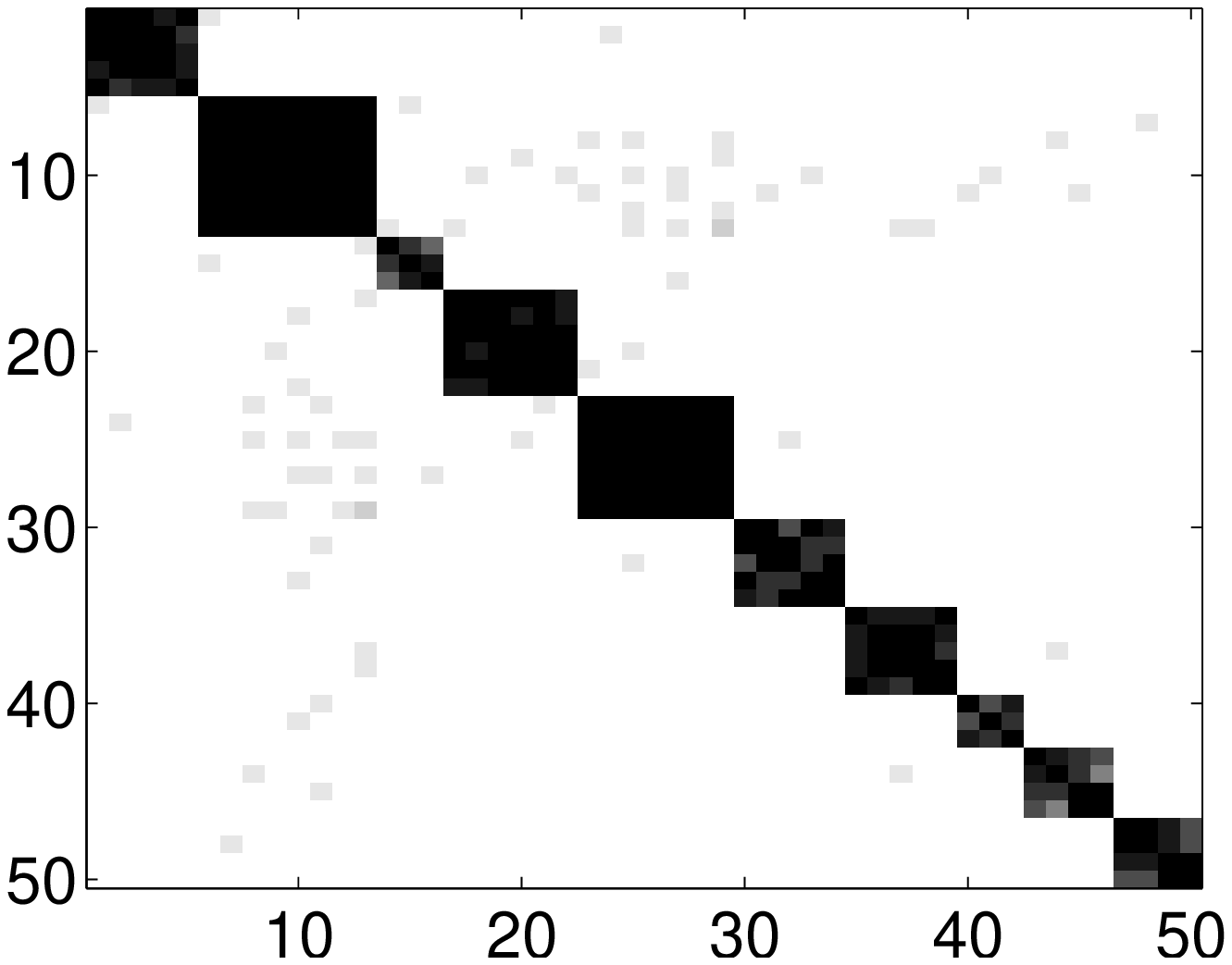}\label{fig:example_nlggm}
} \caption{An illustrative example to justify the importance of
modeling nonlinear interactions among random variables in a
graphical model. The intensity of each entry is the frequency of nonzeros identified for this entry out of 10 replications. White indicates 10 zeros identified out of 10 runs, and black indicates 0/10.} \label{fig:example}
\end{figure}

\subsection{Our Contribution}

In this paper, we address the problem of learning pairwise UGMs with
pairwise sufficient statistics defined by
$\{f_{st}(X_s,X_t):=\phi(X_s,X_t)\}$ where $\phi$ is the global
parametric function. This subclass of UGMs allow us to
model highly nonlinear interactions among the variables. In
contrast, most existing parametric/semiparametric graphical models
use pairwise product of variables (or properly transformed
variables) as sufficient
statistics~\citep{Ravikumar-EJS-2011,Liu-Nonparanormal,Yang-GMGLM-NIPS-2012},
and thus in their nature are unable to capture underlying complex
interactions among the variables. We propose two estimators to learn
the weights in the proposed UGMs. The first estimator is formulated
as $\ell_1$-norm penalized maximum likelihood estimation (MLE). The
second estimator is formulated as $\ell_1$-norm penalized node-wise
conditional MLE. The parameters of the global mapping $\phi$ can be
selected through, e.g., cross-validations.

One contribution of this paper is the statistical
efficiency analysis of the proposed estimators in terms of parameter
estimation error. We prove that under proper conditions the joint
MLE estimator achieves convergence rate $O(\sqrt{|E|\ln p /n})$
where $|E|$ is the number of edges. For the node-wise conditional
estimator, we prove that under proper conditions it achieves
convergence rate $O(\sqrt{d \ln p/n})$ in which $d$ is the degree of
the underlying graph $G$. For GGMs, these convergence rates are
known to be minimax optimal~\citep{Yuan-JMLR-2010,Cai-CLIME-2011}.
We have also analyzed the computational efficiency of the proposed
estimators. Particularly, when the mapping $\phi$ is a Mercer
kernel, we show that with proper relaxation the joint MLE estimator
reduces to a log-determinant program and the conditional MLE
estimator reduces to a Lasso program. These relaxed estimators can
be  efficiently optimized using off-the-shelf algorithms.

We conduct careful numerical studies on simulated and real data to
support our claims. Our simulation results show that, when the data
are drawn from an underlying UGMs with highly nonlinear sufficient
statistics, our approach significantly outperforms GGMs and
Nonparanormal estimators in most cases. The experimental results on
a stock price data show that our method recovers more accurate links
than  GMMs and Nonparanormal estimators. Continuing
with the aforementioned illustrative example,
Figure~\ref{fig:example}(c) shows the graph structure recovered
by our proposed semi-parametric model with heat kernel
$\phi(X_s,X_t)=\exp\{|X_s-X_t|^2/\sigma^2\}$. It can be clearly seen
that our model performs well while the GGMs fail on this example.

\subsection{Related Work}
In order to model random variables beyond parametric UGMs such as
GGMs and Ising models, researchers recently investigated
semiparametric/nonparametric extensions of these parametric models.
The Nonparanormal~\citep{Liu-Nonparanormal} and copula-based
methods~\citep{Dobra-2011} are semiparametric graphical models which assume that data is Gaussian after applying a monotone transformation. The network structure of these models can be recovered by fitting GGMs over the transformed variables. This class of models are also known as Gaussian copula family~\citep{Klaassen-1997,Tsukahara-2005}. More broadly,
one could learn transformations of the variables and then fit any
parametric UGMs (e.g., EFGMs) over the transformed variables. In two
recent papers~\citep{Liu-RankSemiparaGGM,Zou-RankSemiparaGGM}, the
rank-based estimators were used to estimate correlation matrix $\hat
S$ and then fit the following model GGMs.
\begin{equation}\label{equat:nonparanormal_distr}
\mathbb{P}(X;\Omega) = \frac{1}{\sqrt{(2\pi)^p (\det \Omega)^{-1}}}
\exp\left\{ - \Tr (\Omega \hat S) \right\},
\end{equation}
where $\Omega\succ 0$ is the precision matrix to be estimated. The
sufficient statistics used in this model are encoded in the
correlation matrix $\hat S$. Very recently, \citet{Gu-SINICA2013}
proposed a functional minimization framework to estimate the
nonparametric model~\eqref{prob:ggm_pairwise} over a Reproducing
Hilbert Kernel Space (RKHS). In this framework, to infer geometric
structure, a ``hypothesis testing'' method is used to eliminate
those weak interaction terms (edges). The forest density
estimation~\citep{Lafferty-SparseNonParaGM} is a fully nonparametric
method for estimating UGMs with structure restricted to be a forest.
Combinatorial approaches were proposed by~\citet{Lauritzen-1996,Bishop-2007} for fitting graphical models over multivariate count data. A kernel method
was proposed by~\citet{Bach-NIPS-2002} for learning the structure of
graphical models by treating variables as Gaussians in a mapped
high-dimensional feature space.

\vspace{-0.05in}
\subsection{Notation and Outline}

\noindent\textbf{Notation.} In the following, $\theta =
(\theta_i)\in \mathbb{R}^p$ is a vector; $\Theta = (\Theta_{ij}) \in
\mathbb{R}^{p\times p}$ is a matrix. The following notations will be
used in the text.

\begin{itemize}
\item $\supp(\theta)$: the support (set of nonzero elements) of $\theta$.
\item $\|\theta\|_0$: the number of nonzero of $\theta$.
\item $\|\theta\|_1 =\sum_{i=1}^p |\theta_i|$: the $\ell_1$-norm of vector $\theta$.
\item $\|\theta\|_2=\sqrt{\theta^\top \theta}$: the Euclidean norm of vector $\theta$.
\item $|\Theta|_1 = \sum_{i=1}^p\sum_{j=1}^q |\Theta_{ij}|$: the element-wise $\ell_1$-norm of matrix $\Theta$.
\item  $\|\Theta\|_{2,\infty} = \max_{j} \sqrt{\sum_i \Theta^2_{ij}}$: the max column Euclidean norm of $\Theta$.
\item $\supp(\Theta)=\{(i,j):\Theta_{ij}\neq 0\}$: the support of $\Theta$.
\item $\Tr(\Theta)$: the trace (sum of diagonal elements) of a matrix $\Theta$.
\item $\Theta^{-}$: the off-diagonals of $\Theta$.
\item  $\Theta \succeq 0$ ($\Theta \succ 0$): $\Theta$ is a positive
semi-definite (positive definite) matrix.
\item $I_{m\times m}$: $m$-by-$m$ identity matrix.
\item $\bar{S}$: the complement of an index set $S$.
\end{itemize}

\noindent\textbf{Outline.} The remaining of this paper is organized
as follows: \S\ref{sect:model} introduces the semi-parametric
pairwise UGMs with nonlinear sufficient statistics.
\S\ref{sect:parameter_estimation} presents two maximum likelihood
estimators for learning model parameters. The statistical guarantees
of the proposed estimators are analyzed in \S\ref{sect:analysis}.
Monte-Carlo simulations and experimental results on real data are
presented in \S\ref{sect:experiment}. Finally, we conclude this
paper in \S\ref{sect:conclusion}.

\section{Pairwise UGMs with Nonlinear Sufficient Statistics}
\label{sect:model}

Given a univariate parametric mapping $f: \mathcal {X} \rightarrow
\mathbb{R}$ and a bivariate parametric mapping $\phi(\cdot,\cdot):
\mathcal {X}^2 \rightarrow \mathbb{R}$ (for notation clarity purpose
we do not explicitly write out the parameters in $f$ and $\phi$), we
assume that the joint density of $X$ is given by the following
Semiparametric Exponential Family Graphical Models (Semi-EFGMs)
distribution:
\begin{equation}\label{prob:kefgm_distr}
\mathbb{P}(X;\theta) = \exp\left\{\sum_{s \in V}\theta_{s} f(X_s) +
\sum_{(s, t) \in E} \theta_{st} \phi(X_s,X_t)  - A(\theta) \right\},
\end{equation}
where
\[
A(\theta):= \log \int_{\mathcal {X}^p} \exp\left\{\sum_{s \in V}
\theta_s f(X_s) + \sum_{(s,t) \in E}\theta_{st} \phi(X_s,X_t)
\right\} d X
\]
is the log-partition function. We require the condition $A(\theta) <
\infty$ holds so that the definition of probability is valid. The
node-wise sufficient statistics $\{f(X_s)\}$ reflect the strength of
individual nodes. The pairwise sufficient statistics
$\{\phi(X_s,X_t)\}$ characterize the interactions between the nodes.
Specially, when $\phi(x,y) = xy$, Semi-EFGMs reduce to the standard
EFGMs with distribution~\eqref{equat:efgm_distr}. By using proper
nonlinear $\phi$, Semi-EFGMs is able to capture more complex
interactions among variables than EFGMs. Particularly, if the
mapping function $\phi$ is chosen as a Mercer kernel\footnote{ A
Mercer kernel on a space $\mathcal {X}$ is a function
$k(\cdot,\cdot): \mathcal {X}^2 \rightarrow \mathbb{R}$ such that
for any set of points $\{x^{(1)},...,x^{(n)}\}$ in $\mathcal {X}$,
the $n \times n$ matrix $K$, defined by $K_{ij} =
k(x^{(i)},x^{(j)})$, is positive semidefinite. Some popular Mercer
kernels in machine learning include polynomial kernels where
$k(x,y)=(c+x^\top y)^d$ with $c>0, d \in \mathbb{N}$ and radial
basis function kernels where $k(x,y) =
\exp\left\{-\frac{\|x-y\|^2}{2\sigma^2}\right\}$.} and $f(X)=
\phi(X,X)$, then the distribution of Semi-EFGMs is written by
\[
\mathbb{P}(X;\theta) = \exp\left\{ \sum_{s \in V}\theta_s \phi(X_s,
X_s) + \sum_{(s, t) \in E} \theta_{st} \phi(X_s,X_t) - A(\theta)
\right\}.
\]
In this case, Semi-EFGMs can be regarded as a kernel extension of
the Gaussian graphical models by replacing coefficient matrix with
kernel matrix $\Phi$ whose entries are given by $\Phi_{st} =
\phi(X_s,X_t)$. Different from GGMs, it is difficult to find a close-form
log-partition function $A(\theta)$ in the above distribution. It is
interesting to note that when using kernel mapping $\phi$,
Semi-EFGMs allow each random variate to be vector-valued. This
property is particularly useful in scenarios where each random
variate is described by different modalities of features. In the
current model, up to tunable parameters, the bivariate mapping
$\phi$ is assumed to be known. This is analogous to kernel methods
in which the kernels are conventionally assumed to be known.

\section{Parameter Estimation}\label{sect:parameter_estimation}

We are interested in the problem of learning the graph structure of
an underlying Semi-EFGM given i.i.d. samples. Suppose we have $n$
independent samples $\mathbb{X}_n = \{X^{(i)}\}_{i=1}^n$ drawn from
a Semi-EFGM with true parameters $\theta^*$:
\begin{equation}\label{prob:kdfgm_distr_*}
\mathbb{P}(X;\theta^*) = \exp\left\{\sum_{s \in V} f(X_s)
 + \sum_{(s, t) \in E} \theta^*_{st} \phi(X_s,X_t) - A(\theta^*)
\right\}.
\end{equation}
For the sake of notation simplicity in the analysis to follow, we
have assumed here that $\theta_s^*= 1$, noting that our algorithm
and analysis generalize straightforwardly to the cases where
$\theta_{ss}^*$ are also varying. An important goal of graphical
model learning is to estimate the true parameters $\theta^*$ from
the observed data $\mathbb{X}_n$. The more accurate parameter
estimation is, the more accurate we are able to recover the
underlying graph structure. In this section, we will propose two
maximum likelihood estimation (MLE) methods, the $\ell_1$-norm
penalized joint MLE and the $\ell_1$-norm penalized node-conditional
MLE, to estimate the model parameters.

\subsection{Joint Parameter Estimation}
\label{ssect:joint_parameter_estimation}

Given $n$ independent samples $\mathbb{X}_n = \{X^{(i)}\}_{i=1}^n$,
we can write the log-likelihood of the joint
distribution~\eqref{prob:kdfgm_distr_*} as:
$$
L(\theta; \mathbb{X}_n) = -\frac{1}{n}\sum_{i=1}^n \log
\mathbb{P}(X^{(i)}; \theta) = -\frac{1}{n}\sum_{i=1}^n
\left\{f(X_s^{(i)}) + \sum_{s\neq t} \theta_{st}
\phi(X^{(i)}_s,X^{(i)}_t) \right\}+ A(\theta).
$$
With a bit algebra we can obtain the following standard result which
shows that the first two derivatives of $L (\theta;\mathbb{X}_n)$
yield the cumulants of the random variables $\phi(X_s,X_t)$ and $L
(\theta;\mathbb{X}_n)$ is convex with respect to
$\theta$~\citep[see also, e.g.,][]{Wainwright-2008}.
\begin{proposition}\label{prop:joint_derivatives}
The likelihood function $L(\theta;\mathbb{X}_n)$ has the following
first two derivatives:
\begin{eqnarray}
\frac{\partial L (\theta;\mathbb{X}_n)}{\partial \theta_{st}} &=&
-\frac{1}{n}\sum_{i=1}^n\phi(X^{(i)}_s,X^{(i)}_t) +
\mathbb{E}_{\theta}[\phi(X_s,X_t)], \label{equat:joint_derivatives} \\
 \frac{\partial^2
L(\theta;\mathbb{X}_n)}{\partial \theta_{st}\partial \theta_{uv}}
&=& \mathbb{E}_{\theta}[\phi(X_s,X_t)\phi(X_u,X_v)] -
\mathbb{E}_{\theta}[\phi(X_s,X_t)]
\mathbb{E}_{\theta}[\phi(X_u,X_v)], \label{equat:joint_hessian}
\end{eqnarray}
where the expectation $\mathbb{E}_{\theta}[\cdot]$ is taken over the
joint distribution~\eqref{prob:kefgm_distr}. Moreover,
$L(\theta;\mathbb{X}_n)$ is a convex function with respect to
$\theta$.
\end{proposition}
In order to estimate the parameters, we consider the following
$\ell_1$-norm penalized MLE:
\begin{equation}\label{prob:glm_structure_1}
\hat\theta_n = \argmin_{\theta} \left\{ L(\theta; \mathbb{X}_n) +
\lambda_n \|\theta\|_1 \right\},
\end{equation}
where $\lambda_n>0$ is the regularization strength parameter
dependent on $n$. By Proposition~\ref{prop:joint_derivatives}, the
M-estimator~\eqref{prob:glm_structure_1} is strongly convex, and
thus admits a unique global minimizer. The solution can be found by
some off-the-shelf first-order iterative algorithms such as proximal
gradient descent~\citep{Nesterov-2005,Tseng-2008,Beck-2009}. At each
iteration, we need to evaluate the gradient $\nabla
L(\theta;\mathbb{X}_n) $ which is given
in~\eqref{equat:joint_derivatives}. Note that the major
computational overhead is to calculate the expectation term
$\mathbb{E}_{\theta}[\phi(X_s,X_t)]$. In general, this term has no
close-form for exact calculation. We have to resort to sampling
methods for approximate estimation. The multivariate sampling
methods, however, typically suffer from high computational cost
especially when dimension $p$ is large. We next consider the
node-wise parameter estimation method which only requires univariate
sampling for computing the expectation terms involved in the
gradient.

\subsection{Node-wise Parameter Estimation}
\label{ssect:nodewise_parameter_estimation}

Recent state of the art methods for learning GGMs, Ising models and
exponential family
models~\citep{Meinshausen-NSLasso-2006,Ravikumar-AoS-2010,Yang-GMGLM-NIPS-2012}
suggest a natural procedure for deriving multivariate graphical models
from univariate distributions. The key idea in those methods
is to learn the MRF graph structure by estimating
node-neighborhoods, or by fitting node-conditional distributions of
each node conditioned on the rest of the nodes. Indeed, these
node-wise fitting methods have been shown to have strong
computational as well as statistical guarantees. Following these
approaches, we propose an alternative estimator which estimates the
weights of sufficient statistics associated with each individual
node. Given the joint distribution~\eqref{prob:kdfgm_distr_*}, it is
easy to show that the conditional distribution of $X_s$ given the
rest variables, $X_{\s}$, is written by:
\begin{equation}\label{equat:kernel_conditional_distr}
\mathbb{P}(X_s \mid X_{\s};\theta^*_s) = \exp\left\{f(X_s) + \sum_{t
\in N(s)}\theta^*_{st} \phi(X_s, X_t) -
D(X_{\s};\theta^*_s)\right\},
\end{equation}
where with slight abuse of notations, we denote $\theta^*_s :=
\{\theta^*_{st}\}_{t\in N(s)}$, and
\[
D(X_{\s};\theta^*_s):= \log \int_{\mathcal {X}} \exp\left\{f(X_s) +
\sum_{t \in N(s)} \theta^*_{st} \phi(X_s,X_t) \right\} d X_s
\]
is the log-partition function which ensures normalization. Indeed,
the marginal distribution of $X_{\s}$ is
\[
\mathbb{P}(X_{\s};\theta^*)=\int_{\mathcal {X}}
\mathbb{P}(X;\theta^*) dX_s =\exp\left\{\sum_{u \in V\s} f(X_u) +
\sum_{(u,t) \in E, u\in V\s} \theta^*_{ut} \phi(X_u,X_t) -
A(\theta^*) + D(X_{\s};\theta^*_s)\right\}.
\]
The conditional distribution
\eqref{equat:kernel_conditional_distr} is then  $\mathbb{P}(X_{s} \mid X_{\s}; \theta^*)
= \mathbb{P}(X;\theta^*) / \mathbb{P}(X_{\s};\theta^*) $. We note
that $A(\theta^*)<\infty$ implies $D(X_{\s};\theta^*_s) < \infty$.

In order to estimate the parameters associated with any node, we
consider using the sparsity constrained conditional maximum
likelihood estimation. Given $n$ independent samples $\mathbb{X}_n$,
we can write the log-likelihood of the conditional
distribution~\eqref{equat:kernel_conditional_distr} as:
\[
\tilde L(\theta_s; \mathbb{X}_n) = -\frac{1}{n}\sum_{i=1}^n \log
\mathbb{P}(X^{(i)}_s \mid X^{(i)}_{\s}; \theta_s) =
\frac{1}{n}\sum_{i=1}^n \left\{ - f(X^{(i)}_s) - \sum_{t \neq
s}\theta_{st} \phi(X^{(i)}_s, X^{(i)}_t) + D(X^{(i)}_{\s};\theta_s)
\right\}, \nonumber
\]
where $\theta_s = (\theta_{st})_{t\neq s} \in \mathbb{R}^{p-1}$ is
the set of parameter to be estimated. Analogous to
Proposition~\ref{prop:joint_derivatives}, the following proposition
gives the first two derivatives of the likelihood function $\tilde
L(\theta_s; \mathbb{X}_n)$ and establishes the convexity of $\tilde
L(\theta_s; \mathbb{X}_n)$.
\begin{proposition}\label{prop:derivatives}
The likelihood function $\tilde L(\theta_s; \mathbb{X}_n)$ has the
following first two derivatives:
\begin{eqnarray}
\frac{\partial \tilde L(\theta_s; \mathbb{X}_n)}{\partial
\theta_{st}} &=& \frac{1}{n}\sum_{i=1}^n
\left\{-\phi(X_s^{(i)},X_t^{(i)}) +
\mathbb{E}_{\theta_s}[\phi(X_s , X_t^{(i)})\mid X^{(i)}_{\s}]\right\}, \label{equat:derivatives}\\
\frac{\partial^2 \tilde L(\theta_s; \mathbb{X}_n)}{\partial
\theta_{st}\partial\theta_{su}} &=& \frac{1}{n}\sum_{i=1}^n \left\{
\mathbb{E}_{\theta_s}[\phi(X_s,X_t^{(i)})\phi(X_s,X_u^{(i)})\mid
X^{(i)}_{\s}] \right.\nonumber \\
&& \left. - \mathbb{E}_{\theta_s}[\phi(X_s,X_t^{(i)})\mid
X^{(i)}_{\s}] \mathbb{E}_{\theta_s}[\phi(X_s,X_u^{(i)}) \mid
X^{(i)}_{\s}]\right\},
\end{eqnarray}
where the expectation $\mathbb{E}_{\theta_s}[\cdot \mid X_{\s}]$ is
taken over the node-wise conditional
distribution~\eqref{equat:kernel_conditional_distr}.  Moreover,
$\tilde L(\theta_s; \mathbb{X}_n)$ is a convex function with respect
to $\theta_s$.
\end{proposition}
Let us consider the following $\ell_1$-norm penalized conditional
MLE formulation associated with the variable $X_s$:
\begin{equation}\label{prob:glm_structure_0}
\hat\theta^n_s= \argmin_{\theta_s} \left\{ \tilde L(\theta_s;
\mathbb{X}_n) + \lambda_n \|\theta_s\|_1 \right\},
\end{equation}
where $\lambda_n>0$ is the regularization strength parameter
dependent on $n$. By Proposition~\ref{prop:derivatives}, the above
M-estimator is strongly convex, and thus admits a unique global
minimizer. We can use standard first-order methods such as proximal
gradient descent algorithms to optimize the
estimator~\eqref{prob:glm_structure_0}. At each iteration, we need
to evaluate the gradient $\nabla \tilde L(\theta_s; \mathbb{X}_n)$
which is given by~\eqref{equat:derivatives}. Note that the major
computational overhead of~\eqref{equat:derivatives} is to calculate
the expectation term $\mathbb{E}_{\theta_s}[\phi(X_s ,
X_t^{(i)})\mid X^{(i)}_{\s}]$. When $X_s$ is finite and discrete,
this term can be computed exactly via summation. For count-valued or real-valued
variables, however, this term is typically lack of a
close-form for exact calculation. We may resort to some standard
univariate sampling methods, e.g., importance sampling and
MCMC~\citep{Bishop-PRML-2006}, to approximately estimate this
expectation term. The univariate sampling process required by the
node-wise estimator~\eqref{prob:glm_structure_0} is much more
computational efficient than the multivariate sampling process
required by the joint estimator~\eqref{prob:glm_structure_1}.

\section{Statistical Analysis}\label{sect:analysis}

We now provide some parameter estimation error
bounds for the joint MLE
estimator~\eqref{prob:glm_structure_1} and the node-conditional
estimator~\eqref{prob:glm_structure_0}. In large picture, we use the
techniques from~\citep{Negahban-2012,Zhang-Zhang-2012} to analyze
our model by specifying the conditions under which these techniques
can be applied to the model.

\subsection{Analysis of the Joint Estimator}

For the joint estimator~\eqref{prob:glm_structure_1}, we study the
convergence rate of the parameter estimation error $\|\hat\theta_n -
\theta^*\|$ as a function of sample size $n$. Intuitively, as
$n\rightarrow \infty$, we expect $\hat\theta_n \rightarrow \theta^*$
and thus $\nabla L(\hat\theta_n;\mathbb{X}_n) \rightarrow \nabla
L(\theta^*;\mathbb{X}_n)$. Since $\hat\theta_n$ is the minimizer
of~\eqref{prob:glm_structure_1}, we have $\nabla
L(\hat\theta_n;\mathbb{X}_n) \rightarrow 0$ as $n \rightarrow
\infty$ and $\lambda_n \rightarrow 0$. Therefore it is desired that
$\nabla L(\theta^*;\mathbb{X}_n)$ approaches zero as $n$ approaches
infinity. Inspired by this intuition and
Proposition~\ref{prop:joint_derivatives}, we are interested in the
concentration bound of the random variables defined by
\[
Z_{st}:= \phi(X_s,X_t) - \mathbb{E}_{\theta^*}[\phi(X_s,X_t)],
\]
where the expectation $\mathbb{E}_{\theta^*}[\cdot]$ is taken over
the underlying true distribution~\eqref{prob:kdfgm_distr_*}. By the
``law of the unconscious statistician'' we know that $\mathbb{E}
[Z_{st}] = \mathbb{E}_{\theta^*}[\phi(X_s, X_t)] -
\mathbb{E}_{\theta^*} [\phi(X_s, X_t)]=0$. That is, $Z_{st}$ are
zero-mean random variables. We introduce the following technical
condition on $Z_{st}$ which we will show that guarantees the
gradient $\nabla L(\theta^*;\mathbb{X}_n)$ vanishes exponentially
fast, with high probability, as sample size increases.
\begin{assumption}\label{assump:tail_1}
For all $(s,t)$, we assume that there exist constants $\sigma>0$ and
$\zeta>0$ such that for all $|\eta| \le \zeta$,
\[
\mathbb{E}[\exp\{\eta Z_{st}\}] \le
\exp\left\{\sigma^2\eta^2/2\right\}.
\]
\end{assumption}
This assumption essentially imposes an exponential-type bound on the
moment generating function of the random variables $Z_{st}$.
Equivalently, from the definition of $Z_{st}$ and the ``law of the
unconscious statistician'' we know that this assumptions requires:
\[
\mathbb{E}_{\theta^*}[\exp\{\eta(\phi(X_s,X_t) -
\mathbb{E}_{\theta^*}[\phi(X_s,X_t)])\}] \le
 \exp\left\{\sigma^2\eta^2/2\right\}.
\]
The following result indicates that under
Assumption~\ref{assump:tail_1}, $Z_{st}$ satisfy a large deviation
inequality.
\begin{lemma}\label{lemma_tail}
If Assumption~\ref{assump:tail_1} holds, then for all index pairs
$(s,t)$ and any $\varepsilon \le \sigma^2\zeta$ we have
\[
\mathbb{P}\left(\left|\frac{1}{n}\sum_{i=1}^n
\phi(X^{(i)}_s,X^{(i)}_t) -
\mathbb{E}_{\theta^*}[\phi(X_s,X_t)]\right|
> \varepsilon\right) \le 2
\exp\left\{-\frac{n\varepsilon^2}{2\sigma^2}\right\}.
\]
\end{lemma}
A proof of this lemma is given in
Appendix~\ref{append:proof_lemma_tail}. As we will see in the
analysis to follow that Lemma~\ref{lemma_tail} plays a key role in
the deviation of the convergence rate of the joint MLE
estimator~\eqref{prob:glm_structure_1}. Before proceeding, we give a
few remarks on the conditions under which the random variables
$Z_{st}$ satisfy Assumption~\ref{assump:tail_1} such that
Lemma~\ref{lemma_tail} holds.
\begin{remark}[$Z_{st}$ are sub-Gaussian]\label{remark:sub_gaussian}
We call a zero-mean random variable $Z$ sub-Gaussian if there exists
a constant $\sigma>0$ such that $\mathbb{E}[\exp\{\eta Z\}] \le
\exp\left\{\sigma^2\eta^2/2\right\}, \text{for all } \eta \in
\mathbb{R}$. It is straightforward to see that
Assumption~\ref{assump:tail_1} holds when $Z_{st}$ are sub-Gaussian
random variables. For zero-mean Gaussian random variable $Z \sim
N(0,\sigma^2)$, it can be verified that $\mathbb{E}[\exp\{\eta Z\}]
= \exp\{\sigma^2\eta^2/2\}$. Based on the Hoeffding's Lemma, for any
random variable $Z \in [a,b]$ and $\mathbb{E}[z]=0$, we have
$\mathbb{E}[\exp\{\eta Z\}] \le \exp\{\eta^2(b-a)^2/8\}$. Therefore,
Assumption~\ref{assump:tail_1} holds when $Z_{st}$ are zero-mean
Gaussian or zero-mean bounded random variables. For an instance,
Assumption~\ref{assump:tail_1} is valid when the heat kernel mapping
$\phi(X_s,X_t) = \exp\{-|X_s - X_t|^2\}$ is used.
\end{remark}
\begin{remark}[$Z_{st}$ are sub-exponential]\label{remark:sub_exponential}
We call a random variable $Z$ sub-exponential if there exist
constants $c_1, c_2>0$ such that $\mathbb{P}(|X| > \eta) \le
\exp\left\{c_1-\eta/c_2\right\}, \text{for all } t \in \mathbb{R}$.
Using the result in~\citep[Lemma 5.15]{Vershynin-SubExp-2011}, we
can verify that Assumption~\ref{assump:tail_1} holds when $Z_{st}$
are sub-exponential random variables. One connection between
sub-Gaussian and sub-exponential random variables is: a random
variable $Z$ is sub-Gaussian if and only if $Z^2$ is
sub-exponential~\citep[Lemma 5.14]{Vershynin-SubExp-2011}. From this
connection and the fact that the sum of sub-exponential random
variables is still exponential, we know that if $Z_{st}$ are
Chi-square random variables (sum of square of Gaussian), then
$Z_{st}$ are sub-exponential and thus Assumption~\ref{assump:tail_1}
holds.
\end{remark}
\begin{remark}\label{remark:general}
More generally, consider that $\phi$ is a Mercer kernel satisfying
the condition:
\begin{equation}\label{inequat:integer_bound}
\int_{\mathcal {X}} \exp\{-c\phi(X,X)\}dX < \infty, \quad \text{for
any } c>0.
\end{equation}
Obviously, this condition holds when $\phi(X,X)$ grows faster than
$X^q$ for some $q>0$. If $f(X_s) = - \phi(X_s,X_s)$ and
$|\theta^*_{st}|<0.25$, then we claim that $Z_{st}$ are
sub-exponential. Indeed, since $\phi$ is Mercer kernel, we have that
$\phi(X_s,X_s)\ge0$, $\phi(X_s,X_t)=\phi(X_t, X_s)$ and
$|\phi(X_s,X_t)| \le (\phi(X_s,X_s) + \phi(X_t,X_t))/2$. Thus,
\begin{equation}\label{equat:example_3}
\mathbb{P}(|\phi(X_s,X_t)| > \eta) \le \mathbb{P}(\phi(X_s,X_s) +
\phi(X_t,X_t)> 2\eta).
\end{equation}
For any $s \in V$, from the joint
distribution~\eqref{prob:kdfgm_distr_*} and $|\theta^*_{st}|<0.25$
we know that the marginal distribution of $X_s$ is bounded by
$\mathbb{P} (X_s) \le c_1 \exp\left\{ -\phi(X_s,X_s) /2\right\}$ for
some absolute constant $c_1$. Therefore, by using Markov inequality
and~\eqref{inequat:integer_bound} we obtain that for any $\eta>0$
\[
\mathbb{P}(\phi(X_s,X_s)>\eta) \le
\frac{\mathbb{E}[\exp\{\phi(X_s,X_s)/4\}]}{\exp\{\eta/4\}} \le
\frac{c_1 \int_{\mathcal {X}} \exp\{-\phi(X_s,X_s)/4\}dX_s
}{\exp\{\eta/4\}} \propto \exp\{-\eta/4\},
\]
which implies that $\phi(X_s,X_s)$ is sub-exponential. By combining
the fact that the sum of sub-exponential random variables is still
sub-exponential and~\eqref{equat:example_3} we obtain that
$\phi(X_s,X_t)$ are sub-exponential, and so are $Z_{st}$. Obviously,
the above claim is applicable to the case of multivariate Gaussian
where $\phi(X_s,X_t) = X_sX_t$ and $f(X_s)=-X_s^2$. Similar results
have also been proved in previous work on
GGMs~\citep{Rothman-2008,Ravikumar-EJS-2011}.
\end{remark}

Let us define $\gamma_n:=\|\nabla L
(\theta^*;\mathbb{X}_n)\|_\infty$. The following lemma indicates
that under Assumption~\ref{assump:tail_1}, with overwhelming
probability, $\gamma_n$ approaches zero at the rate of $O(\sqrt{\ln
p/n})$. A proof of this lemma can be found in
Appendix~\ref{append:proof_lemma_norm_joint}.
\begin{lemma}\label{lemma:norm_joint}
Assume that Assumption~\ref{assump:tail_1} holds. If $n
> 6 \ln p/(\sigma^2\zeta^2)$, then with probability at least $1-
2p^{-1}$ the following inequality holds:
\[
\gamma_n =\|\nabla L
(\theta^*;\mathbb{X}_n)\|_\infty\le \sigma\sqrt{6\ln p/n}.
\]
\end{lemma}
From~\eqref{equat:joint_hessian} we know that the Hessian $\nabla^2
L(\theta;\mathbb{X}_n)$ is positive semidefinite at any $\theta$. To
derive the estimation error, we also need the following condition
which guarantees the restricted positive definiteness of $\nabla^2
L(\theta;\mathbb{X}_n)$ when $\theta$ is sufficiently close to
$\theta^*$.
\begin{assumption}[\textbf{Locally Restricted Positive Definite Hessian}]\label{assump:positive_definite}
Let $S=\supp(\theta^*)$. There exist constants $r>0$ and $\beta>0$
such that for any $\theta \in \{\|\theta - \theta^*\| < r\}$, the
following inequality holds for any $\vartheta \in \mathcal
{C}_{S}:=\{ \|\theta_{\bar S}\|_1 \le 3 \|\theta_S\|_1\}$:
\[
\vartheta^\top\nabla^2 L(\theta;\mathbb{X}_n) \vartheta \ge \beta
\|\vartheta\|^2.
\]
\end{assumption}
Assumption~\ref{assump:positive_definite} requires that the Hessian
$\nabla^2 L(\theta;\mathbb{X}_n)$ is positive definite in the cone
$\mathcal {C}_S$ when $\theta$ lies in a local ball centered at
$\theta^*$. This condition is specification of the concept
\emph{restricted strong convexity}~\citep{Zhang-Zhang-2012} to our
problem setup. If $X$ is multivariate Gaussian, i.e.,
$\phi(X_s,X_t)= X_sX_t$ and $f(X_s)= - X_s^2$,  it is easy to verify
that this condition can be satisfied when the true precision matrix
is positive definite~\citep{Rothman-2008}.
\begin{remark}[Minimal Representation]\label{remark:minimal}
We say Semi-EFGM has \emph{minimal representation} if there is a
unique parameter vector $\theta$ associate with the
distribution~\eqref{prob:kefgm_distr}. When fix $\theta_{ss}=1$,
this condition equivalently requires that there does not exist a
non-zero $\theta$ such that the linear combination $\sum_{s,t}
\theta_{st} \phi(X_s,X_t)$ is equal to an absolute constant. This
implies that for any $\theta$,
\[
\var_{\theta}\left[\sum_{s,t}\vartheta_{st} \phi(X_s,X_t)\right] =
\vartheta^\top \nabla^2 L(\theta;\mathbb{X}_n) \vartheta>0, \quad
\text{for all non-zero $\vartheta$}.
\]
It follows that there exist constants $r>0$ and $\beta>0$ such that
for any $\theta \in \{\|\theta - \theta^*\| < r\}$,
$\vartheta^\top\nabla^2 L(\theta;\mathbb{X}_n) \vartheta \ge \beta
\|\vartheta\|^2$. Therefore,
Assumption~\ref{assump:positive_definite} is valid when Semi-EFGM
has minimal representation.
\end{remark}
The following result bounds the estimation error of the joint MLE
estimator~\eqref{prob:glm_structure_1} in terms of $\gamma_n$, $r$
and $\beta$.
\begin{lemma}\label{lemma:error_bound_joint}
Assume that the conditions in
Assumption~\ref{assump:positive_definite} hold. Assume that
$\lambda_n \in [2 \gamma_n, c_0 \gamma_n]$ for some $c_0 \ge 2$.
Define $\gamma = 1.5c_0\sqrt{\|\theta^*\|_0}\beta^{-1}\gamma_n$. If
$\gamma <r$, then we have
\[
\|\hat\theta_n - \theta^*\| \le
1.5\sqrt{\|\theta^*\|_0}\beta^{-1}\gamma_n.
\]
\end{lemma}
A proof of this lemma is provided in
Appendix~\ref{append:proof_lemma_error_bound_joint}. The following
theorem is our main result on the estimation error of the joint MLE
estimator~\eqref{prob:glm_structure_1}.
\begin{theorem}\label{thrm:main_joint}
Assume that the conditions in Lemma~\ref{lemma:norm_joint} and
Lemma~\ref{lemma:error_bound_joint} hold. If sample size $n$
satisfies
\[
n > \max\left(\frac{6\ln p}{\sigma^2\zeta^2},
13.5c_0^2r^{-2}\beta^{-2}\sigma^2\|\theta^*\|_0\ln p\right),
\]
then with probability at least $1-2p^{-1}$, the following inequality
holds:
\[
\|\hat\theta_n - \theta^*\| \le
1.5c_0\beta^{-1}\sigma\sqrt{6\|\theta^*\|_0\ln p/n}.
\]
\end{theorem}
\begin{proof}
By using Lemma~\ref{lemma:norm_joint} and the condition $n >
13.5c_0^2 r^{-2}\beta^{-2}\sigma^2\|\theta^*\|_0\ln p$ we have that
with probability at least $1-2p^{-1}$,
\[
\gamma = 1.5c_0\sqrt{\|\theta^*\|_0}\beta^{-1}\gamma_n \le
1.5c_0\beta^{-1} \sigma\sqrt{6\|\theta^*\|_0\ln p/n} < r.
\]
By applying Lemma~\ref{lemma:error_bound_joint} we obtain the
desired result.
\end{proof}
\begin{remark}\label{remark:conv}
The main message Theorem~\ref{thrm:main_joint} conveys is that when
$n=O(\|\theta^*\|_0\ln p)$ is sufficiently large, the estimation
error $\|\hat\theta_n - \theta^*\|$ vanishes at the order of
$O(\sqrt{\|\theta^*\|_0\ln p/n})$. This convergence rate matches the
results obtained in~\citep{Rothman-2008,Ravikumar-EJS-2011} for GGMs
and the results in~\citep{Liu-RankSemiparaGGM,Zou-RankSemiparaGGM}
for Nonparanormal. To our knowledge, this is the first sparse
recovery result for the exponential family graphical models with
general sufficient statistics beyond pairwise product. Note that we
did not make any attempt to optimize the constants in
Theorem~\ref{thrm:main_joint}, which are relatively loose.
\end{remark}
By specifying the conditions under which the assumptions in
Theorem~\ref{thrm:main_joint} hold, we obtain the following corollary.
\begin{corollary}\label{corol:main_joint}
Assume that the mapping function $\phi(\cdot,\cdot)$ is a Mercer
kernel satisfying the condition~\eqref{inequat:integer_bound}. Let
$f(X_s)=-\phi(X_s,X_s)$. Assume that the joint
distribution~\eqref{prob:kdfgm_distr_*} has minimal representation
and $|\theta^*_{st}| \le 0.25$. Then there exist constants
$\sigma,\zeta, r, \beta, c_0
> 0$ such that if
\[
n > \max\left(\frac{6\ln p}{\sigma^2\zeta^2},
13.5c_0^2r^{-2}\beta^{-2}\sigma^2\|\theta^*\|_0\ln p\right),
\]
then with probability at least $1-2p^{-1}$, the following inequality
holds:
\[
\|\hat\theta_n - \theta^*\| \le
1.5c_0\beta^{-1}\sigma\sqrt{6\|\theta^*\|_0\ln p/n}.
\]
\end{corollary}
\begin{proof}
Since $\phi(\cdot,\cdot)$ is a Mercer kernel
satisfying~\eqref{inequat:integer_bound} and
$f(X_s)=-\phi(X_s,X_s)$, from the arguments in
Remark~\ref{remark:general} we know that $Z_{st}:= \phi(X_s,X_t) -
\mathbb{E}_{\theta^*}[\phi(X_s,X_t)]$ are sub-exponential, and thus
Assumption~\ref{assump:tail_1} holds. Since the joint
distribution~\eqref{prob:kdfgm_distr_*} has minimal representation,
from the discussions in Remark~\ref{remark:minimal} we know that
Assumption~\ref{assump:positive_definite} is valid. The corollary
then follows immediately from Theorem~\ref{thrm:main_joint}.
\end{proof}

\subsection{Analysis of the Node-Conditional Estimator}
For the node-conditional estimator~\eqref{prob:glm_structure_0}, we
study the rate of convergence of the parameter estimation error
$\|\hat\theta^n_s - \theta^*_s\|$ as a function of sample size $n$.
Intuitively, as $n\rightarrow \infty$, we expect $\hat\theta^n_s
\rightarrow \theta^*_s$ and thus $\nabla \tilde L(\hat\theta^n_s;
\mathbb{X}_n) \rightarrow \nabla \tilde L(\theta^*_s;
\mathbb{X}_n)$. Since $\hat\theta^n_s$ is the minimizer
of~\eqref{prob:glm_structure_0}, we have $\nabla \tilde
L(\hat\theta^n_s;\mathbb{X}_n) \rightarrow 0$ as $n \rightarrow
\infty$ and $\lambda_n \rightarrow 0$. Therefore it is desired that
$\nabla \tilde L(\theta^*_s;\mathbb{X}_n)$ vanishes as $n$
approaches infinity. Inspired by this intuition and
Proposition~\ref{prop:derivatives}, we  study the concentration
bound of the random variables defined by
\[
\tilde Z_{st}:= \mathbb{E}_{\theta^*_s} [\phi(X_s,X_t) \mid X_{\s}]
- \mathbb{E}_{\theta^*}[\phi(X_s,X_t)],
\]
where the expectation $\mathbb{E}_{\theta^*_s}[\cdot\mid X_{\s}]$ is
taken over the node-conditional
distribution~\eqref{equat:kernel_conditional_distr}. By applying the
``law of the unconscious statistician'' and the rule of iterated
expectation (i.e., $\mathbb{E}[X]=\mathbb{E}[\mathbb{E}[X|Y]]$), we
obtain that $\mathbb{E} [\tilde Z_{st}] =
\mathbb{E}_{\theta^*}[\phi(X_s, X_t)] - \mathbb{E}_{\theta^*}
[\phi(X_s, X_t)]=0$. That is, $\tilde Z_{st}$ are zero-mean random
variables. The following lemma shows that under
Assumption~\ref{assump:tail_1}, $\tilde Z_{st}$ have
exponential-type moment generating function. A proof of this lemma
is given in Appendix~\ref{append:proof_lemma_tail_mgf}.
\begin{lemma}\label{lemma:tail_mgf}
If Assumption~\ref{assump:tail_1} holds, then for any $(s,t)$ we
have that for all $|\eta| \le \zeta$,
\[
\mathbb{E}[\exp\{\eta\tilde Z_{st}\}] \le
\exp\left\{\sigma^2\eta^2/2\right\}.
\]
\end{lemma}
\begin{remark}
This lemma shows that the random variables $\tilde Z_{st}$ all have
the same exponential-type moment generating function as that of
$Z_{st}$. From the discussions in
Remarks~\ref{remark:sub_gaussian},~\ref{remark:sub_exponential}
and~\ref{remark:general} we know that when $Z_{st}$ are sub-Gaussian
or sub-exponential, if Assumption~\ref{assump:tail_1} holds, and
consequently Lemma~\ref{lemma:tail_mgf} holds.
\end{remark}

The following result indicates that under
Assumption~\ref{assump:tail_1}, $\tilde Z_{st}$ has a similar large
deviation property as $Z_{st}$.
\begin{lemma}\label{lemma_tail_0}
If Assumption~\ref{assump:tail_1} holds, then for all index pairs
$(s,t)$ and any $\varepsilon \le \sigma^2\zeta$ we have
\[
\mathbb{P}\left(\left|\frac{1}{n}\sum_{i=1}^n
\mathbb{E}_{\theta^*_s}[\phi(X_s,X^{(i)}_t)\mid X^{(i)}_{\s}] -
\mathbb{E}_{\theta^*} [\phi(X_s,X_t)] \right|
> \varepsilon\right) \le 2
\exp\left\{-\frac{n\varepsilon^2}{2\sigma^2}\right\}.
\]
\end{lemma}
The proof of this lemma follows the same arguments as that of
Lemma~\ref{lemma_tail}. Let us define $\tilde\gamma_n:=\|\nabla
\tilde L(\theta^*_s;\mathbb{X}_n)\|_\infty$. The following lemma
indicates that under Assumption~\ref{assump:tail_1}, with
overwhelming probability, $\tilde\gamma_n$ approaches zero at the
rate of $O(\sqrt{\ln p/n})$.
\begin{lemma}\label{lemma:norm}
Assume that Assumption~\ref{assump:tail_1} holds. If $n
> 6 \ln p/(\sigma^2\zeta^2)$, then with probability at least $1-
4p^{-2}$ the following inequality holds:
\[
\tilde \gamma_n \le 2\sigma\sqrt{6\ln p/n}.
\]
\end{lemma}
A proof of this lemma can be found in
Appendix~\ref{append:proof_lemma_norm}. Analogous to
Assumption~\ref{assump:positive_definite} for the joint estimator,
we further introduce the following condition which is sufficient to
guarantee the statistical efficiency of the node-conditional
estimator~\eqref{prob:glm_structure_0}.
\begin{assumption}\label{assump:positive_definite_node}
For any node $s$, let $S=\supp(\theta_s^*)$. There exist constants
$\tilde r>0$ and $\tilde\beta>0$ such that for any $\theta_s \in
\{\|\theta_s - \theta_s^*\| < \tilde r\}$, the following inequality
holds for any $\vartheta_s \in \tilde{\mathcal{C}}_{S}:=\{
\|(\theta_s)_{\bar S}\|_1 \le 3 \|(\theta_s)_S\|_1\}$:
\[
\vartheta_s^\top\nabla^2 \tilde L(\theta_s;\mathbb{X}_n) \vartheta_s
\ge \tilde \beta \|\vartheta_s\|^2.
\]
\end{assumption}
\begin{remark}\label{remark:positive_definite_node}
Assumption~\ref{assump:positive_definite_node} requires that the
Hessian $\nabla^2 \tilde L(\theta_s;\mathbb{X}_n)$ is positive
definite in the cone $\tilde{\mathcal {C}}_S$ when $\theta_s$ lies
in a local ball centered at $\theta_s^*$. Specially, when $X$ is
multivariate Gaussian, i.e., $\phi(X_s,X_t)= X_sX_t$ and $f(X_s)= -
X_s^2$, this condition essentially requires that the design matrix
$A^n_{s} = \frac{1}{n}\sum_{i=1}^n X^{(i)}_{\s} (X^{(i)}_{\s})^\top$
is positive definite. In this case, if the precision matrix is
positive definite, then it is known from the compressed sensing
literature~\citep[see][for
example]{Baraniuk-SimpleRIP-2008,Candes-D-RIP-2011} that with
overwhelming probability, $A^n_s$ is positive definite provided that
the sample size $n = O(\ln p)$ is sufficiently large. More
generally, it can be verified that $\mathbb{E}[\nabla^2 \tilde
L(\theta_s;\mathbb{X}_n)]$ is the sub-matrix of $\nabla^2
L(\theta;\mathbb{X}_n)$ associated with the pairs $(s,t)_{t\in
V\s}$. Therefore, if the whole Hessian matrix $\nabla^2
L(\theta;\mathbb{X}_n)$ is positive definite at any $\theta$, then
$\mathbb{E}[\nabla^2 \tilde L(\theta_s;\mathbb{X}_n)]$ is also
positive definite. By using weak law of large number we obtain that
Assumption~\ref{assump:positive_definite_node} holds with high
probability when $n$ is sufficiently large.
\end{remark}
The following result establishes the estimation error of the
node-conditional estimator~\eqref{prob:glm_structure_0} in terms of
$\tilde \gamma_n$, $\tilde r$ and $\tilde \beta$.
\begin{lemma}\label{lemma:error_bound}
Assume that the conditions in
Assumption~\ref{assump:positive_definite_node} hold. Assume that
$\lambda_n \in [2 \tilde\gamma_n, \tilde c_0 \tilde\gamma_n]$ for
some $\tilde c_0 \ge 2$. Define $\tilde\gamma = 1.5\tilde c_0
\sqrt{\|\theta_s^*\|_0}\tilde\beta^{-1}\tilde\gamma_n$. If
$\tilde\gamma < \tilde r$, then we have
\[
\|\hat\theta^n_s - \theta_s^*\| \le 1.5\tilde
c_0\sqrt{\|\theta_s^*\|_0}\tilde\beta^{-1}\tilde\gamma_n.
\]
\end{lemma}
The proof of this lemma mirrors that
of~Lemma~\ref{lemma:error_bound_joint}. By combining
Lemma~\ref{lemma:norm} and Lemma~\ref{lemma:error_bound}, we
immediately obtain the following main result on the convergence rate
of $\|\hat \theta^n_s - \theta^*_s\|$.
\begin{theorem}\label{thrm:main}
Assume that the conditions in Lemma~\ref{lemma:norm} and
Lemma~\ref{lemma:error_bound} hold. If sample size $n$ satisfies
\[
n > \max\left(\frac{6\ln p}{\sigma^2\zeta^2}, 54\tilde c_0^2\tilde
r^{-2}\tilde\beta^{-2}\sigma^2\|\theta_s^*\|_0\ln p\right),
\]
then with probability at least $1-4p^{-2}$, the following inequality
holds:
\[
\|\hat\theta^n_s - \theta^*_s\| \le 3\tilde c_0
\tilde\beta^{-1}\sigma\sqrt{6\|\theta^*_s\|_0\ln p/n}.
\]
\end{theorem}
\begin{proof}
By using Lemma~\ref{lemma:norm} and the condition $n > 54\tilde
c^2_0 \tilde r^{-2}\tilde \beta^{-2}\sigma^2\|\theta_s^*\|_0\ln p$
we have that with probability at least $1-4p^{-2}$,
\[
\tilde\gamma = 1.5\tilde
c_0\sqrt{\|\theta^*_s\|_0}\tilde\beta^{-1}\tilde\gamma_n \le 3\tilde
c_0\tilde\beta^{-1} \sigma\sqrt{6\|\theta_s^*\|_0\ln p/n} < \tilde
r.
\]
By applying Lemma~\ref{lemma:error_bound} we obtain the desired
result.
\end{proof}

\begin{remark}
Note that we did not make any attempt to optimize the constants in
the presented results above, which are relatively loose. Therefore
in the discussion, we shall ignore the constants, and focus on the
main messages these results convey. Theorem~\ref{thrm:main}
indicates that with overwhelming probability, the estimation error
$\|\hat\theta^n_s - \theta^*_s\| = O(\sqrt{d\ln p/n})$ where $d$ is
the degree of the underlying graph, i.e., $d = \max_{s \in V}
\|\theta^*_s\|_0$. We may combine the estimation errors from all the
nodes as a global measurement of accuracy. Let $\theta^*$ (or
$\hat\theta^n$) be a matrix stacked by the columns $\theta^*_s$ (or
$\hat\theta^n_s$). By Theorem~\ref{thrm:main} and union of
probability we obtain that $\|\hat\theta - \theta^*\|_{2,\infty} =
O(\sqrt{d\ln p/n})$ holds with probability at least $1-4p^{-1}$.
This estimation error bound is analogous to those specifically
derived for GGMs with neighborhood-selection-type
estimators~\citep{Yuan-JMLR-2010}.
\end{remark}
By specifying the conditions under which the assumptions in
Theorem~\ref{thrm:main} hold, we obtain the following corollary.
\begin{corollary}\label{corol:main}
Assume that the mapping function $\phi(\cdot,\cdot)$ is a Mercer
kernel satisfying~\eqref{inequat:integer_bound}. Let
$f(X_s)=-\phi(X_s,X_s)$. Assume that the joint
distribution~\eqref{prob:kdfgm_distr_*} has minimal representation
and $|\theta^*_{st}|<0.25$. If $n$ is sufficiently large, then with
overwhelming probability the following inequality holds:
\[
\|\hat\theta^n_s - \theta^*_s\| = O(\sqrt{\|\theta_s^*\|_0\ln p/n}).
\]
\end{corollary}
\begin{proof}
Since $\phi(\cdot,\cdot)$ is a Mercer kernel
satisfying~\eqref{inequat:integer_bound} and
$f(X_s)=-\phi(X_s,X_s)$, from Remark~\ref{remark:general} we know
that $Z_{st}:= \phi(X_s,X_t) - \mathbb{E}_{\theta^*}[\phi(X_s,X_t)]$
are sub-exponential, and thus Assumption~\ref{assump:tail_1} holds.
Since the joint distribution~\eqref{prob:kdfgm_distr_*} has minimal
representation, from the discussions in Remark~\ref{remark:minimal}
we know that Hessian $\nabla^2 L(\theta;\mathbb{X}_n)$ is positive
definite at any $\theta$. Therefore, based on the discussions in
Remark~\ref{remark:positive_definite_node} we obtain that if $n$ is
sufficiently large, then with overwhelming probability
Assumption~\ref{assump:positive_definite_node} holds. The corollary
then follows immediately from Theorem~\ref{thrm:main}.
\end{proof}

\section{Experiment}
\label{sect:experiment}

We evaluate the performance of Semi-EFGMs for graphical models
learning on synthetic and real data sets. We first
investigate support recovery accuracy using simulation data (for
which we know the ground truth), and then we apply our method to
the analysis of a stock price data. But, before presenting the
experimental results, we first need to discuss some implementation issues of
the proposed estimators.

\subsection{Implementation Issues}

As discussed in \S\ref{ssect:joint_parameter_estimation} and
\S\ref{ssect:nodewise_parameter_estimation}, in order to estimate
the gradient of the likelihood functions $L(\theta;\mathbb{X}_n)$
and $\tilde L(\theta_s;\mathbb{X}_n)$, we need to iteratively use
sampling methods to calculate the involved expectation terms. This
can be quite time consuming when the dimension $p$ is large. In our
empirical study, we are particularly interested in Semi-EFGMs with
Mercer kernel mapping $\phi$. For this subclass of Semi-EFGMs,
instead of using the generic sampling based optimization algorithms,
we propose to relax the estimators~\eqref{prob:glm_structure_1}
and~\eqref{prob:glm_structure_0} such that the sampling step can be
avoided during the optimization.

Given a Mercer kernel $\phi$, it is possible to find a space
$\mathcal {F}$ and a map $\varphi$ from $\mathcal {X}$ to $\mathcal
{F}$, such that $\phi(X_s,X_t)=\varphi(X_s)^\top \varphi(X_t)$ is
the dot produce in $\mathcal {F}$ between $\varphi(X_s)$ and
$\varphi(X_t)$. The space $\mathcal {F}$ is usually referred to as
the feature space and the map $\varphi$ as the feature map.\\

\noindent \textbf{Relax the Joint
Estimator~\eqref{prob:glm_structure_1}.} Assume that the feature
space $\mathcal {F}$ has finite dimension $m$. Let
$\varphi(X)=[\varphi(X_1),...,\varphi(X_p)] \in \mathcal {F}^p
\subseteq \mathbb{R}^{pm}$ be the expanded random vector. Provided
that $X$ has the distribution~\eqref{prob:kefgm_distr} with
$f(X_s)=-\phi(X_s,X_s)$, the joint distribution of the random vector
$\varphi(X)$ is written by
\begin{equation}\label{equat:kernel_ggm}
\mathbb{P}(\varphi(X);\Omega) \propto \exp\left\{ - \varphi(X)^\top
\Omega \varphi(X) \right\}, \quad \varphi(X) \in \mathcal {F}^p,
\end{equation}
where $\Omega = \Theta \otimes I_{m\times m}$ and $\Theta$ is
coefficient matrix with $\Theta_{ss}= 1$ and $\Theta_{st} =
-\theta_{st}$, $s\neq t$. Typically, the
distribution~\eqref{equat:kernel_ggm} has no close-form
log-partition function. Ideally, if $\mathcal {F}^p =
\mathbb{R}^{pm}$ and $\Theta \succ 0$ (which implies $\Omega \succ
0$), then $\varphi(X)$ is multivariate Gaussian with distribution
\[
\mathbb{P}(\varphi(X);\Omega) = \frac{1}{\sqrt{\pi^{pm} (\det
\Omega)^{-1}}} \exp\left\{ - \varphi(X)^\top \Omega \varphi(X)
\right\}.
\]
Since $\phi(X_s,X_t)=\varphi(X_s)^\top\varphi(X_t)$ and using the
fact $\det\Omega= (\det\Theta)^m$, we can re-write the preceding
distribution in terms of $\Theta$ as
\begin{equation}\label{equat:kernel_ggm_1}
\mathbb{P}(\varphi(X);\Theta) = \frac{1}{\sqrt{\pi^{pm} (\det
\Theta)^{-m}}} \exp\left\{ - \Tr(\Theta^\top \Phi(X)) \right\}.
\end{equation}
Recall that $\Phi(X)$ denotes the kernel matrix with elements
$\Phi_{st} = \phi(X_s,X_t)$. However, this above ideal formulation
does not hold in the general cases where $\mathcal
{F}^p\neq\mathbb{R}^{pm}$. In these cases, in order to enjoy the
close-form distribution~\eqref{equat:kernel_ggm_1}, we may relax the
domain of $\varphi(X)$ from $\mathcal {F}^p$ to $\mathbb{R}^{pm}$
and fit the samples to the distribution~\eqref{equat:kernel_ggm_1}.
After such a relaxation, the joint
estimator~\eqref{prob:glm_structure_1} reduces to the following
$\ell_1$-norm penalized log-determinant program:
\begin{equation}\label{prob:log_determinant_program}
\hat\Theta_n = \argmin_{\Theta \succ 0} \left\{ - \frac{m}{2}
\log\det\Theta + \Tr(\Theta^\top \Phi_n) + \lambda_n |\Theta^-|_1
\right\},
\end{equation}
where $\Phi_n = \frac{1}{n}\sum_{i=1}^n \Phi(X^{(i)})$. There exist
a variety of optimization algorithms addressing this convex
formulation~\citep{Aspremont-2008,Friedman-Glasso-2008,Schmidt-2009,Lu-VSM-2009,Wang-PPA-2010,Yuan-ADM-2012,Yuan-TPAMI-2013}.
In our implementation, we resort to a smoothing and proximal
gradient method from~\citep{Lu-VSM-2009} which has been proved to be
efficient and accurate in practice. Note that the dimension $m$ in~\eqref{prob:log_determinant_program} is
unknown  and  can be
treated as a tuning parameter.\\

\noindent \textbf{Relax the Node-Conditional
Estimator~\eqref{prob:glm_structure_0}.} We may apply a similar
relaxation trick as discussed above to the node-conditional
estimator~\eqref{prob:glm_structure_0}. Since
$\phi(X_s,X_t)=\varphi(X_s)^\top\varphi(X_t)$, we may re-write the
node-conditional distribution~\eqref{equat:kernel_conditional_distr}
in terms of $\varphi(X)$ as
\[
\mathbb{P}(\varphi(X_s) \mid X_{\s};\theta^*_s) \propto \exp\left\{
- \left\|\varphi(X_s) - \frac{1}{2}\sum_{t \in N(s)}\theta^*_{st}
\varphi(X_t)\right\|^2\right\}.
\]
Ideally, if $\mathcal {F}^p = \mathbb{R}^{pm}$, then up to a const,
the node-conditional likelihood can be expressed as:
\[
\tilde L (\theta_s; \mathbb{X}_n) = \frac{1}{n}\sum_{i=1}^n
\left\{\frac{1}{4}\sum_{t,u \neq s}\theta_{st}\theta_{su}
\phi(X^{(i)}_t, X^{(i)}_u) - \sum_{t \neq s}\theta_{st}
\phi(X^{(i)}_s, X^{(i)}_t)\right\},
\]
which is quadratic with respect to $\theta_s$. For the general cases
where $\mathcal {F}^p\neq\mathbb{R}^{pm}$, we may relax the domain
of $\varphi(X)$ from $\mathcal {F}^p$ to $\mathbb{R}^{pm}$ so that
we can still enjoy the above quadratic formulation. By using this
relaxation, the estimator~\eqref{prob:glm_structure_0} becomes a
Lasso problem
\begin{equation}\label{prob:Lasso}
\hat\theta^n_s= \argmin_{\theta_s} \left\{\frac{1}{4}\theta_s^\top
\Phi^n_{\s} \theta_s - (\Phi^n_s)^\top \theta_s + \lambda_n
\|\theta_s\|_1 \right\},
\end{equation}
where $\Phi^n_{\s}$ denotes the sub-matrix of $\Phi_n$ associated
with the nodes $V\s$ and $\Phi^n_s$ is the row of $\Phi_n$
associated with node $s$ (with the diagonal element excluded). The
above estimator can be seen as a kernel extension of the
neighborhood selection method~\citep{Meinshausen-NSLasso-2006} for
GGMs learning. Provided that the evaluation of kernels is
negligible, the solution of the Lasso program~\eqref{prob:Lasso} can
be efficiently found by proximal gradient descent
methods~\citep{Tseng-2008,Beck-2009}. Let $\hat\Theta_n$ be a matrix stacked by the columns $\hat\theta^n_s$. Between $\hat\theta^n_{st}$ and $\hat\theta^n_{ts}$, we take the one with smaller magnitude. This makes resultant $\hat\Theta_n$ a symmetric matrix.

Our numerical experience shows that both relaxed estimators work
well on the used dataset in terms of solution quality.
Computationally, we observe that the solvers for the log-determinant
program~\eqref{prob:log_determinant_program} tend to be slightly
more efficient than those for the node-wise Lasso
program~\eqref{prob:Lasso}. The following reported experimental
results of our model are obtained by using a log-determinant program
solver developed by~\citet{Lu-VSM-2009}.

\subsection{Monte Carlo Simulation}

This is a proof-of-concept experiment. The purpose of this
experiment is to confirm that when the pairwise interactions of the
underlying graphical models are highly nonlinear, our approach with
proper parametric function $\phi$ can be significantly superior to
existing parametric/semiparametric graphical models for inferring
the structure of graphs.\\

\noindent \textbf{Simulated Data}\hspace{0.15in} Our simulation study employs the
following two graphical models which belong to the semiparametric
exponential family~\eqref{prob:kefgm_distr}, with varying structures
of sparsity:
\begin{itemize}
  \item \textbf{Model 1}: In this model, the random variables are uniformly partitioned into 10 groups. For any pair of variables $(X_s,X_t)$ belongs to the same group, we set
  them to be connected with strength $\theta_{st} = 1$, while those pairs of variables from different groups are set to be unconnected.
  \item \textbf{Model 2}: In this model, each parameter $\theta_{st}$ is generated independently and equals to 1 with probability
  $P$ or 0 with probability $1-P$. We will consider the model under different levels of sparsity by adjusting the probability $P$.
\end{itemize}
Model 1 has block structures and Model 2 is an example of graphical
models without any special sparsity pattern. In these two models, we
consider two function families to model the interactions between
pairs $(X_s, X_t)$: the heat kernel $\phi(X_s, X_t) =
\exp\{-(X_s-X_t)^2/\sigma^2\}$ with $\sigma = 1$ and the polynomial
kernel $\phi(X_s,X_t) = (\beta+ \varphi(X_s)^\top
\varphi(X_t))^\alpha$ with $\beta=1$, $\alpha = 2$ and
$\varphi(X_s)$ denotes a $d$-dimensional feature vector (with
unit-length) representation of $X_s$. We set $d=5$ in our study. For
each variate $X_s$, we set $C(X_s,X_s) = -\frac{1}{2}\phi(X_s,X_t)$.
Using Gibbs sampling, we generate a training sample of size $n$ from
the true graphical model, and an independent sample of the same size
from the same distribution for tuning $\lambda_n$ and the parameters
$\sigma$, $\alpha$ and $\beta$ in the function families. We compare
performance for $n=200$, different values of $p\in \{50, 100, 150,
200, 250, 300\}$, and different sparsity levels $P=\{0.02, 0.05,
0.1\}$, replicated 10 times each.\\

\noindent \textbf{Comparison of Models}\hspace{0.15in} We compare the performance of
our estimator to GLasso~\citep{Friedman-Glasso-2008} as a GGMs
estimator and SKEPTIC~\citep{Liu-RankSemiparaGGM} as a Nonparanormal
estimator. In order to apply GLasso to the data with vector-valued
variates, we  treat each dimension of the feature vector
$\varphi(X_s)$ as a sample, and hence we have $200 \times 5 = 1000$
samples which are assumed to be drawn from GGMs. In this setup, GGMs
can be taken as a special case of Semi-EFGMs with linear kernel
$\phi(X_s,X_t) = \varphi(X_s)^\top\varphi(X_t)$. The same treatment
of data is also applied to SKEPTIC. In addition, we
use a version of SKEPTIC with Kendall's tau to infer the
correlation.\\

\noindent \textbf{Evaluation Criterion}\hspace{0.15in} To evaluate the support
recovery performance, we use the standard F-score from the
information retrieval literature~\citep{Rijsbergen-IR-1979}. The
larger the F-score, the better the support recovery performance. The
numerical values over $10^{-3}$ in magnitude are considered to be
nonzero.\\

\noindent \textbf{Results}\hspace{0.15in} Overall, the experimental results on the
simulated data suggest that:
\begin{itemize}
  \item When the pairwise interactions are modeled by the heat kernel function (see the left panels of Figure~\ref{fig:fscore_model_1} and
  Figure~\ref{fig:fscore_model_2}), we have the performance order: Semi-EFGMs $\gg$ Nonparanormal $>$
  GGMs.
  \item When the pairwise interactions are modeled by polynomial kernel function (see the right panels of Figure~\ref{fig:fscore_model_1} and
  Figure~\ref{fig:fscore_model_2}), we have the performance order: Semi-EFGMs $\ge$ Nonparanormal $\ge$ GGMs.
\end{itemize}

\begin{figure}[h!]
\centering
\includegraphics[width=78mm]{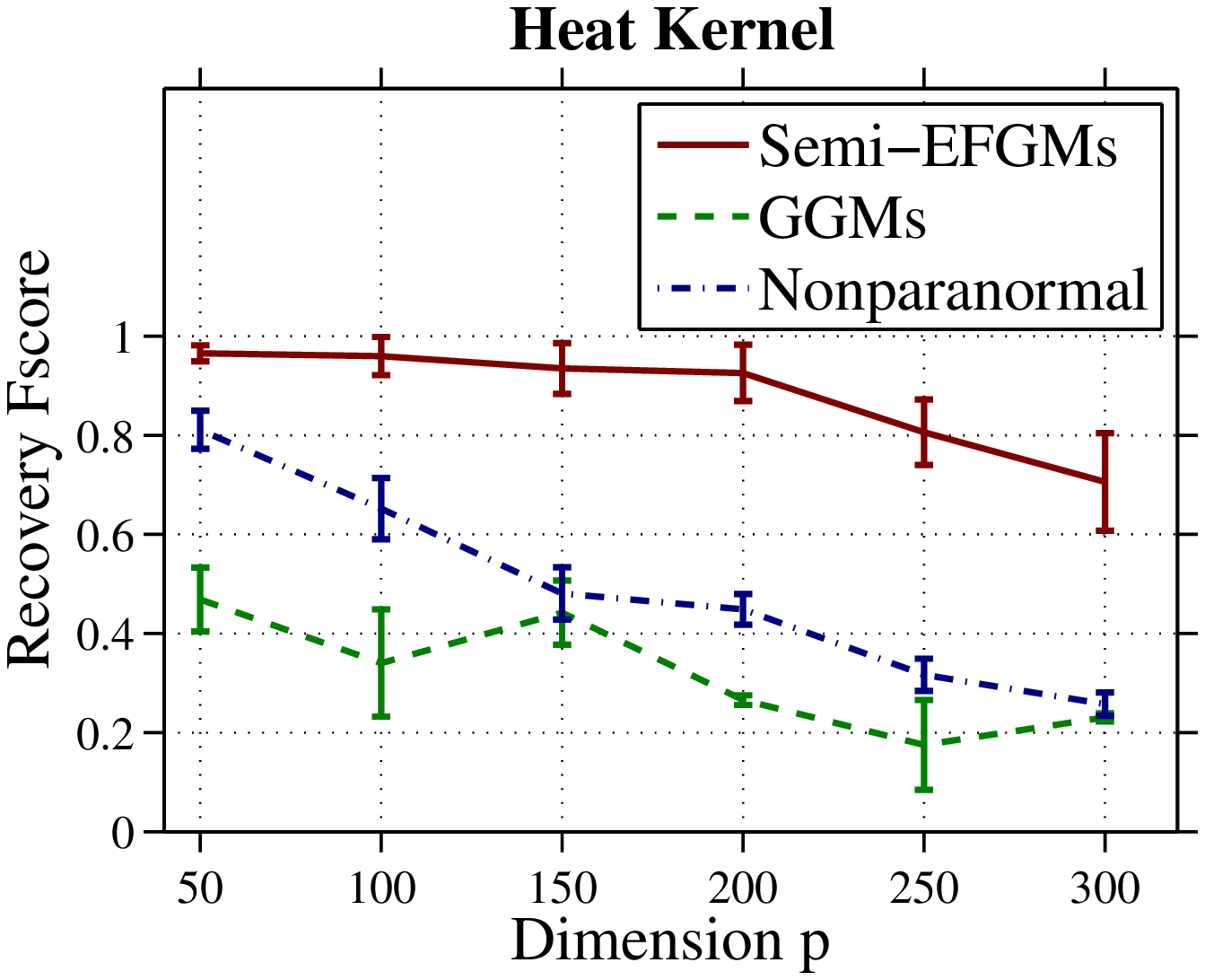}\label{fig:model_heat}
\includegraphics[width=78mm]{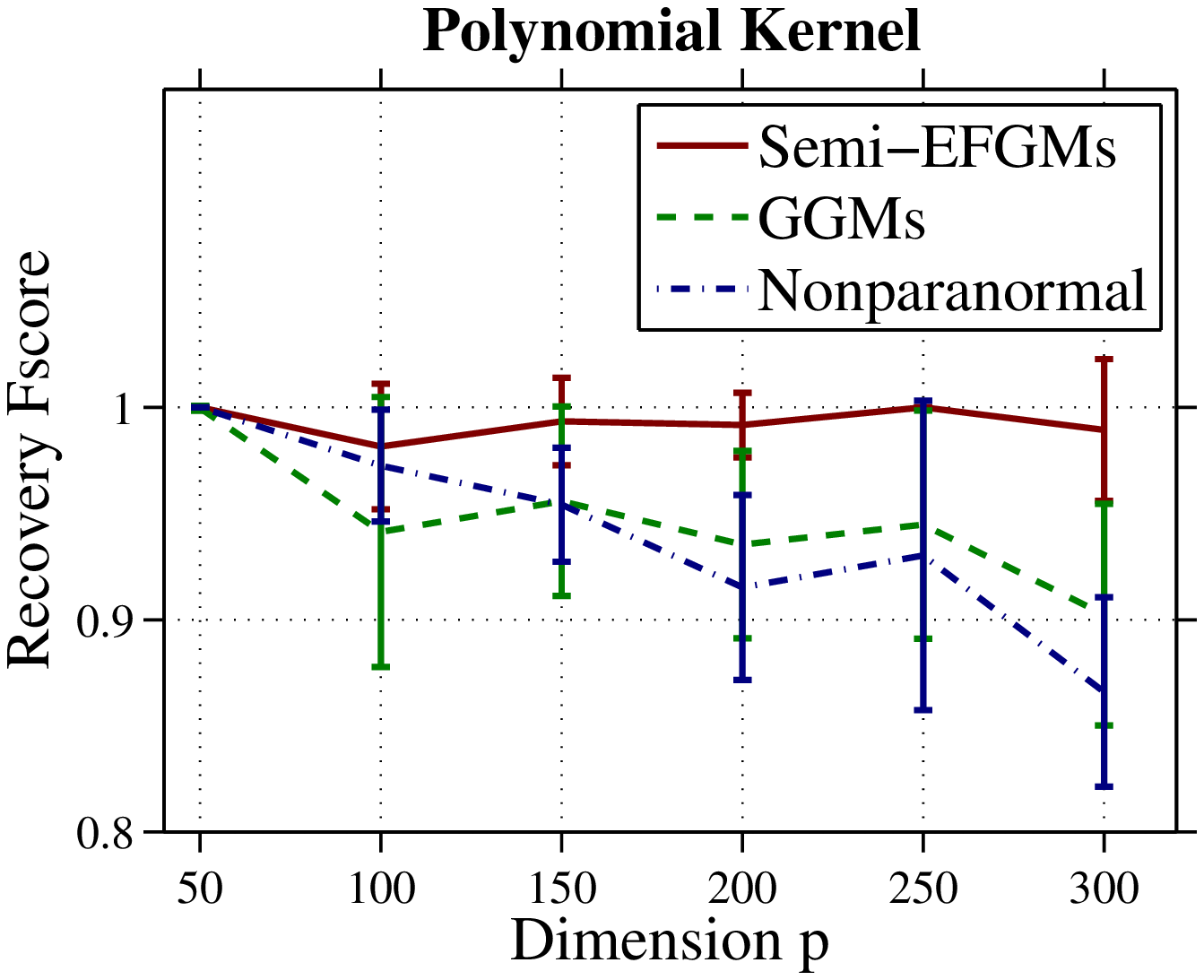}\label{fig:model_poly}
 \caption{Support recovery F-score results on data model 1. \label{fig:fscore_model_1}}
\end{figure}

We now describe in detail these observations.
Figure~\ref{fig:fscore_model_1} shows the support recovery F-scores
on data model 1. From the left panel we can see that Semi-EFGMs work
significantly better than the other two considered methods when the
mutual interactions of variables are modeled by heat kernel
function. In this case, we also observe that Nonparanormal is much
more accurate than GGMs in structure recovery. When polynomial
kernel function is used to define mutual sufficient statistics (see
right panel), Semi-EFGMs is slightly better than the two considered
methods when $p<=100$ and the gap becomes more and more apparent as
$p$ increases. The advantage of Semi-EFGMs is as expected because
this approach explicitly models the nonlinear pairwise interactions
which is hard to be captured by the traditional GGMs and
Nonparanormals. Figure~\ref{fig:fscore_model_2} shows the support
recovery F-scores on data model 2 with different configurations of
kernel function and sparsity level. From the left column we observe
again that Semi-EFGMs significantly outperform the other two
considered methods on heat kernel based models. From the right
column we can see that on polynomial kernel based models, our model
is significantly better than the other two considered methods when
the graph structure is extremely sparse (i.e., $P=0.02$) while it is
slightly better than the other two considered methods when the graph
structure becomes less sparse (i.e., $P=0.05, 0.1$). The numerical
figures to generate Figure~\ref{fig:fscore_model_1}
and~\ref{fig:fscore_model_2} are listed in
Table~\ref{tab:synthetic_results_fscore}. To visually inspect the
support recovery performance of different methods, we show in
Figure~\ref{fig:heatmaps} several selected heatmaps reflecting the
percentage of each graph matrix entry being identified as a nonzero
element. Visual inspection on these heatmaps confirm that Semi-EFGMs
perform favorably in graph structure recovery, especially when heat
kernels are used as pairwise sufficient statistics (see the top two
rows).

\begin{figure}
\centering \subfigure[Sparsity $P = 0.02$]{
\includegraphics[width=78mm]{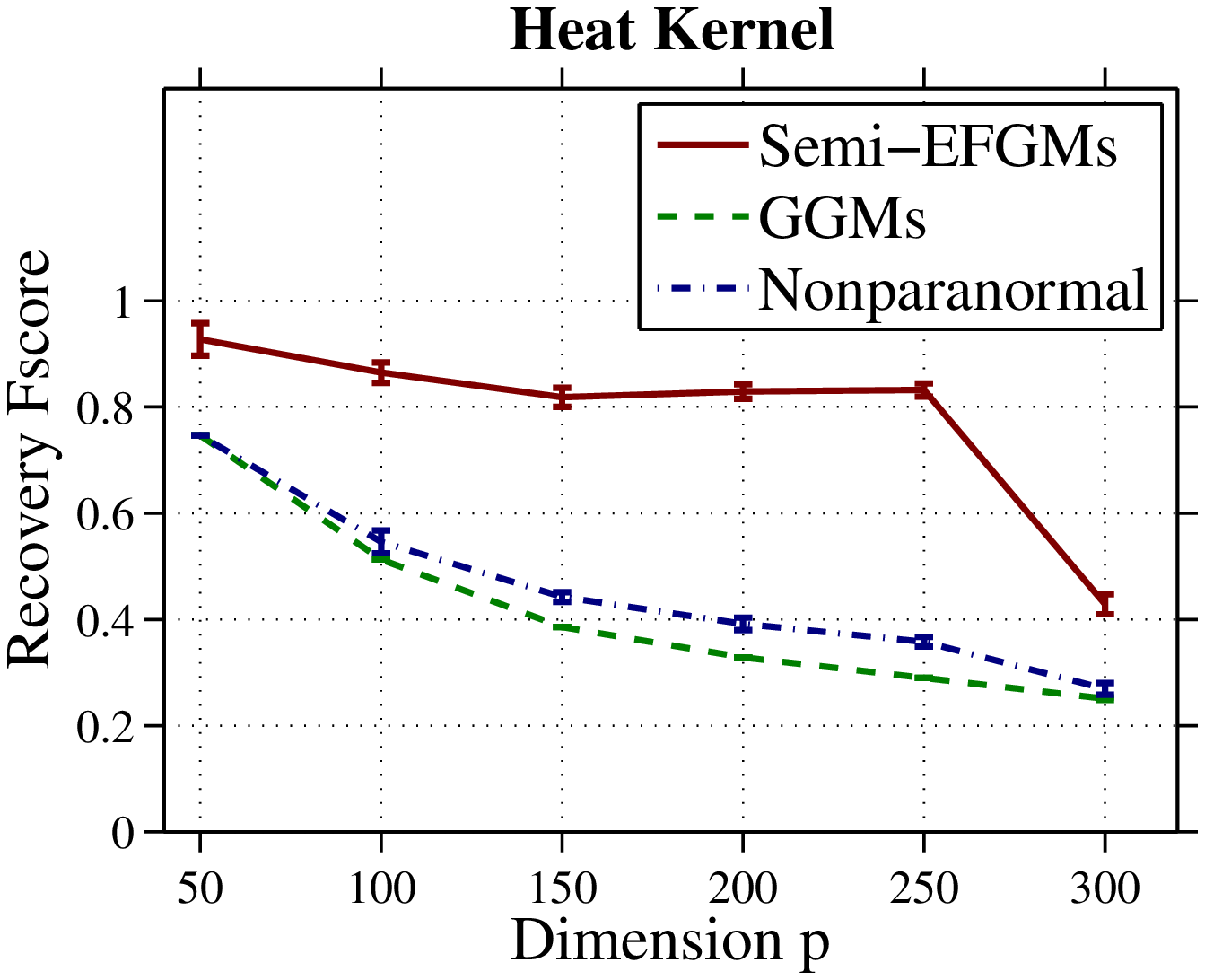}\label{fig:mode2_heat_0.02}
\includegraphics[width=78mm]{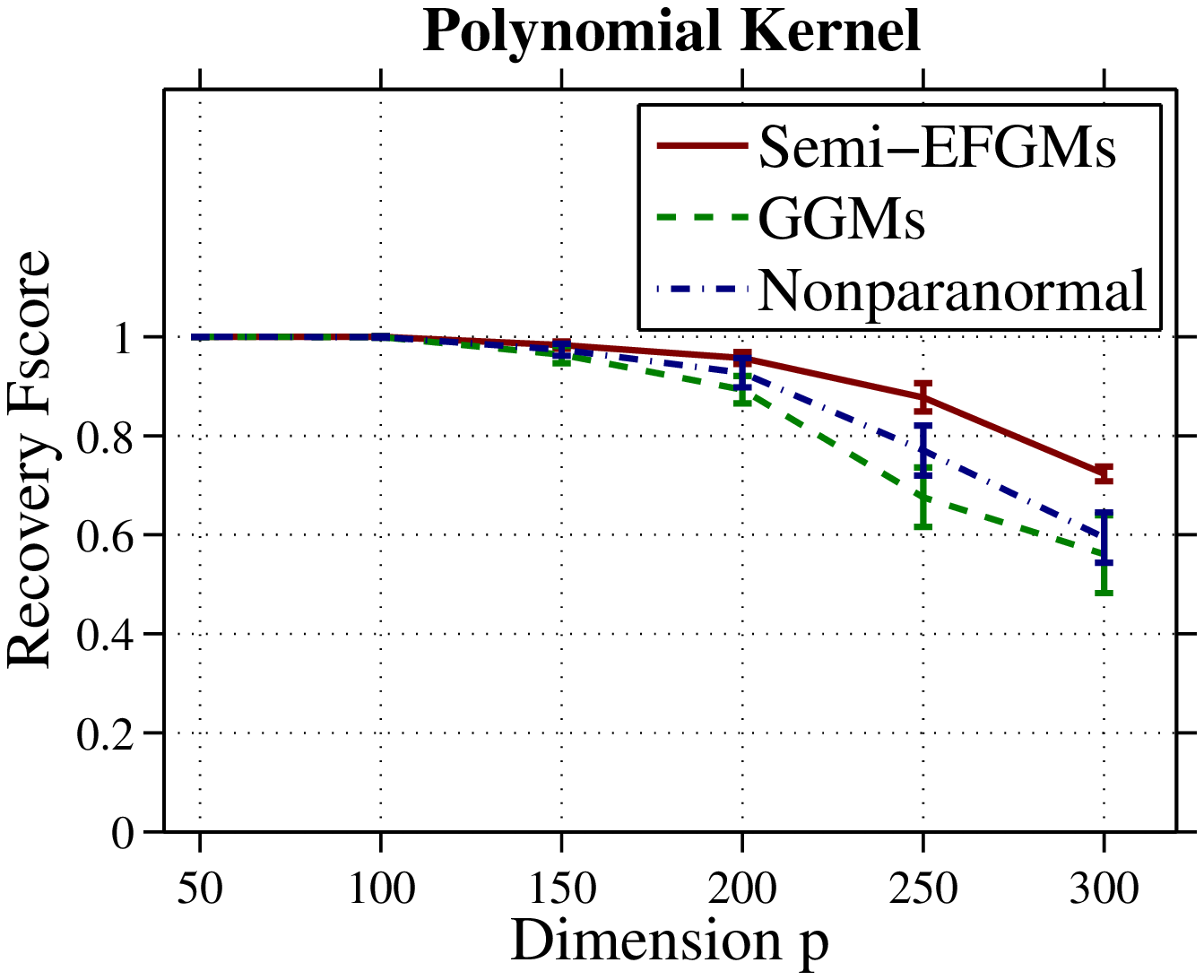}\label{fig:mode2_poly_0.02}
} \subfigure[Sparsity $P = 0.05$]{
\includegraphics[width=78mm]{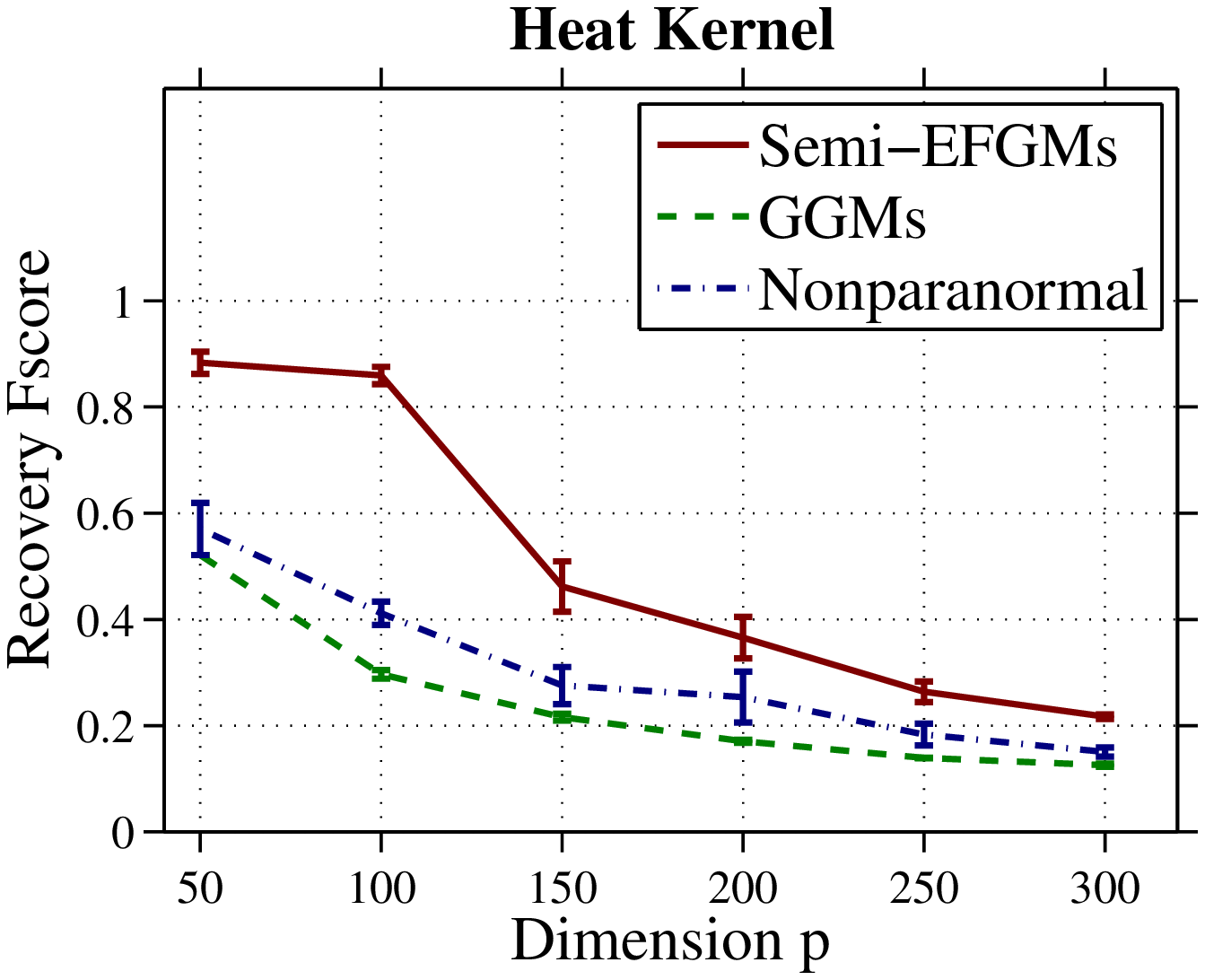}\label{fig:mode2_heat_0.05}
\includegraphics[width=78mm]{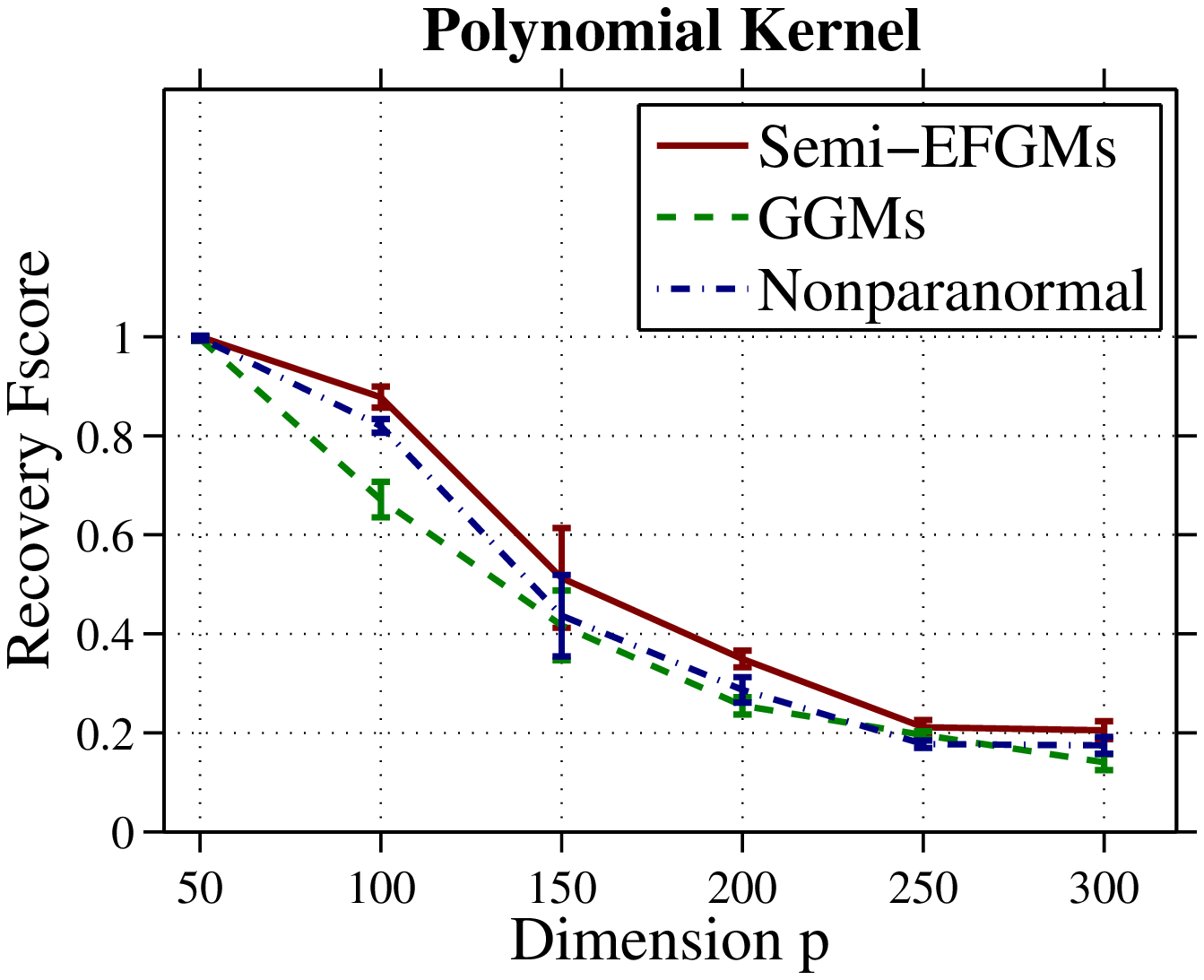}\label{fig:mode2_poly_0.05}
} \subfigure[Sparsity $P = 0.1$]{
\includegraphics[width=78mm]{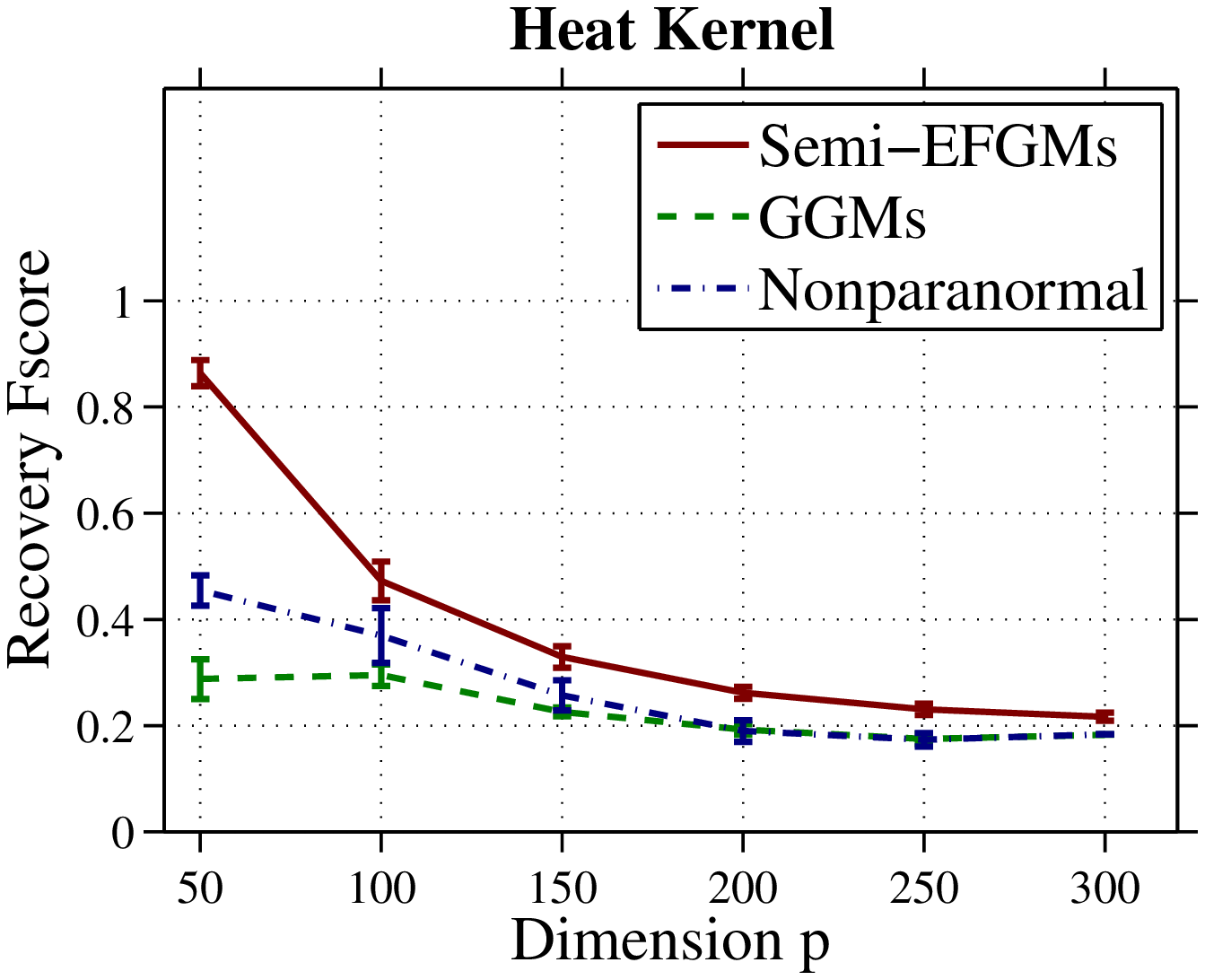}\label{fig:mode2_heat_0.1}
\includegraphics[width=78mm]{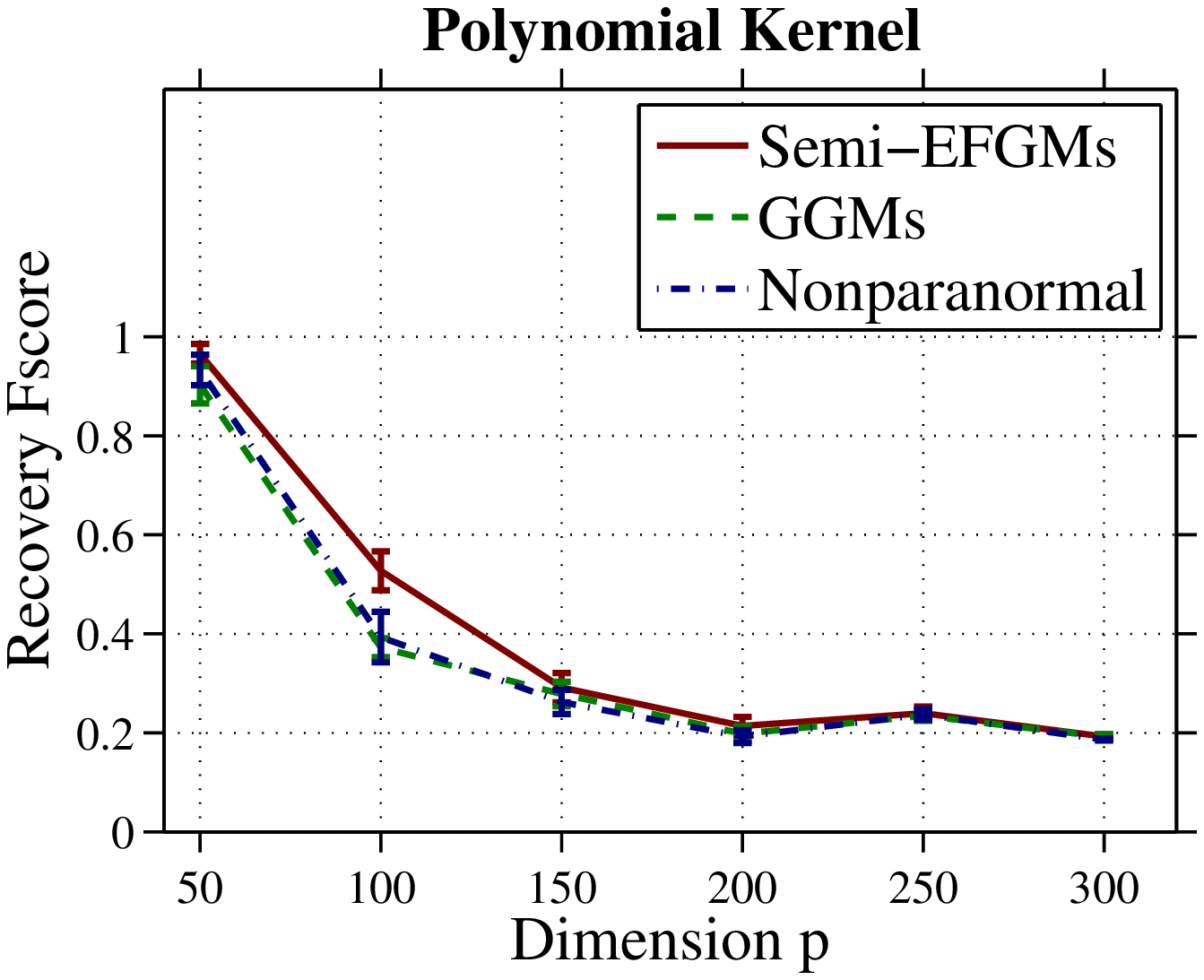}\label{fig:mode2_poly_0.1}
} \caption{Support recovery F-score results on data model
2.\label{fig:fscore_model_2}}
\end{figure}

\begin{table}[h!]
\begin{center}
\caption{Comparison of average (SE) of F-score for synthetic data
over 100 replications. \label{tab:synthetic_results_fscore}}
\begin{tabular}{c c c c c c c}
\hline
Methods & $p=50$ & $100$ & $150$ & $200$ & $250$ & $300$\\
\hline \hline
\multicolumn{7}{c}{Data model 1 with heat kernel} \\
Semi-EFGMs      & \textbf{0.97} ($0.02$)& \textbf{0.96} (0.04) & \textbf{0.93} (0.05) & \textbf{0.93} (0.06) & \textbf{0.81} (0.07) & \textbf{0.71} (0.10) \\
GGMs   & 0.47 (0.06) & 0.34 (0.11) & 0.44 (0.06) & 0.27 (0.01) & 0.18 (0.09)& 0.23 (0.01) \\
Nonparanormal  & 0.81 (0.04) & 0.65 (0.06) & 0.48 (0.05) & 0.45 (0.03) & 0.32 (0.03) & 0.26 (0.02) \\
\hline
\multicolumn{7}{c}{Data model 1 with polynomial kernel} \\
Semi-EFGMs      & \textbf{1.00} ($0.00$)& \textbf{0.98} (0.03) & \textbf{0.99} (0.02) & \textbf{0.99} (0.02) & \textbf{1.00} (0.00) & \textbf{0.99} (0.03) \\
GGMs   & 0.99 (0.01) & 0.94 (0.06) & 0.96 (0.04) & 0.94 (0.04) & 0.94 (0.05) & 0.90 (0.05) \\
Nonparanormal  & \textbf{1.00} (0.00) & 0.97 (0.03) & 0.95 (0.03) & 0.92 (0.04) & 0.93 (0.07) & 0.87 (0.04) \\
\hline
\multicolumn{7}{c}{Data model 2 with heat kernel, sparsity $P = 0.02$} \\
Semi-EFGMs      & \textbf{0.93} ($0.03$)& \textbf{0.86} (0.02) & \textbf{0.82} (0.02) & \textbf{0.83} (0.01) & \textbf{0.83} (0.01) & \textbf{0.43 }(0.02) \\
GGMs   & 0.75 (0.01) & 0.51 (0.01) & 0.39 (0.01) & 0.33 (0.01) & 0.29 (0.01) & 0.25 (0.01) \\
Nonparanormal  & 0.75 (0.01) & 0.55 (0.02) & 0.44 (0.01) & 0.39 (0.01) & 0.36 (0.01)  & 0.27 (0.01)  \\
\hline
\multicolumn{7}{c}{Data model 2 with polynomial kernel, sparsity $P = 0.02$} \\
Semi-EFGMs      & \textbf{1.00} ($0.00$)& \textbf{0.99} (0.01) & \textbf{0.98} (0.01) & \textbf{0.96} (0.01) & \textbf{0.88} (0.03) & \textbf{0.72} (0.02) \\
GGMs   & \textbf{1.00 }(0.00) & \textbf{0.99} (0.01) & 0.96 (0.02) & 0.89 (0.03) & 0.68 (0.06) & 0.56 (0.08)  \\
Nonparanormal  & \textbf{1.00 }(0.00) & \textbf{0.99} (0.01) & 0.97 (0.01) & 0.93 (0.03) & 0.77 (0.05)  & 0.59 (0.05) \\
\hline
\multicolumn{7}{c}{Data model 2 with heat kernel, sparsity $P = 0.05$} \\
Semi-EFGMs  & \textbf{0.88} ($0.02$)& \textbf{0.86} (0.02) & \textbf{0.46} (0.05) & \textbf{0.37} (0.04) & \textbf{0.26} (0.02) & \textbf{0.22 }(0.01) \\
GGMs   & 0.52 (0.01) & 0.30 (0.01) & 0.22 (0.01) & 0.17 (0.01) & 0.14 (0.01)  & 0.13 (0.01)  \\
Nonparanormal  & 0.57 (0.05) & 0.41 (0.02) & 0.28 (0.03) & 0.25 (0.05) & 0.18 (0.02)  &  0.15 (0.01) \\
\hline
\multicolumn{7}{c}{Data model 2 with polynomial kernel, sparsity $P = 0.05$} \\
Semi-EFGMs   & \textbf{1.00} ($0.00$)& \textbf{0.88} (0.02) & \textbf{0.51} (0.10) & \textbf{0.35} (0.01) & \textbf{0.21} (0.02) & \textbf{0.21} (0.02)\\
GGMs   & 0.99 (0.01) & 0.67 (0.04) & 0.42 (0.07) & 0.25 (0.02) & 0.20 (0.01)  & 0.14 (0.05)  \\
Nonparanormal  & 0.99 (0.01) & 0.82 (0.01) & 0.44 (0.08) & 0.29 (0.03) & 0.18 (0.01)  &  0.18 (0.02) \\
\hline
\multicolumn{7}{c}{Data model 2 with heat kernel, sparsity $P = 0.1$} \\
Semi-EFGMs   & \textbf{0.86} ($0.02$)& \textbf{0.47} (0.04) & \textbf{0.33} (0.02) & \textbf{0.26} (0.01) & \textbf{0.23} (0.01) & \textbf{0.22 }(0.01) \\
GGMs   & 0.29 (0.04) & 0.29 (0.02) & 0.23 (0.01) & 0.19 (0.01) & 0.17 (0.01)  & 0.18 (0.01)  \\
Nonparanormal  & 0.45 (0.03) & 0.37 (0.05) & 0.26 (0.03) & 0.19 (0.02) & 0.17 (0.01) & 0.18 (0.01)  \\
\hline
\multicolumn{7}{c}{Data model 2 with polynomial kernel, sparsity $P = 0.1$} \\
Semi-EFGMs  & \textbf{0.97} (0.02)& \textbf{0.53} (0.04) & \textbf{0.29} (0.03) & \textbf{0.21} (0.02) & \textbf{0.24} (0.01) & \textbf{0.19} (0.01) \\
GGMs   & 0.90 (0.04) & 0.37 (0.02) & 0.28 (0.02) & 0.20 (0.02) &  0.23 (0.01) &\textbf{0.19} (0.01)  \\
Nonparanormal  & 0.93 (0.03) & 0.39 (0.05)  & 0.26 (0.02) & 0.19 (0.01) & \textbf{0.24} (0.01) &  \textbf{0.19} (0.01)  \\
\hline
\end{tabular}
\end{center}
\end{table}

\begin{figure}
\centering \subfigure[Data model 1 with heat kernel, $p=100$]{
\includegraphics[width=40mm]{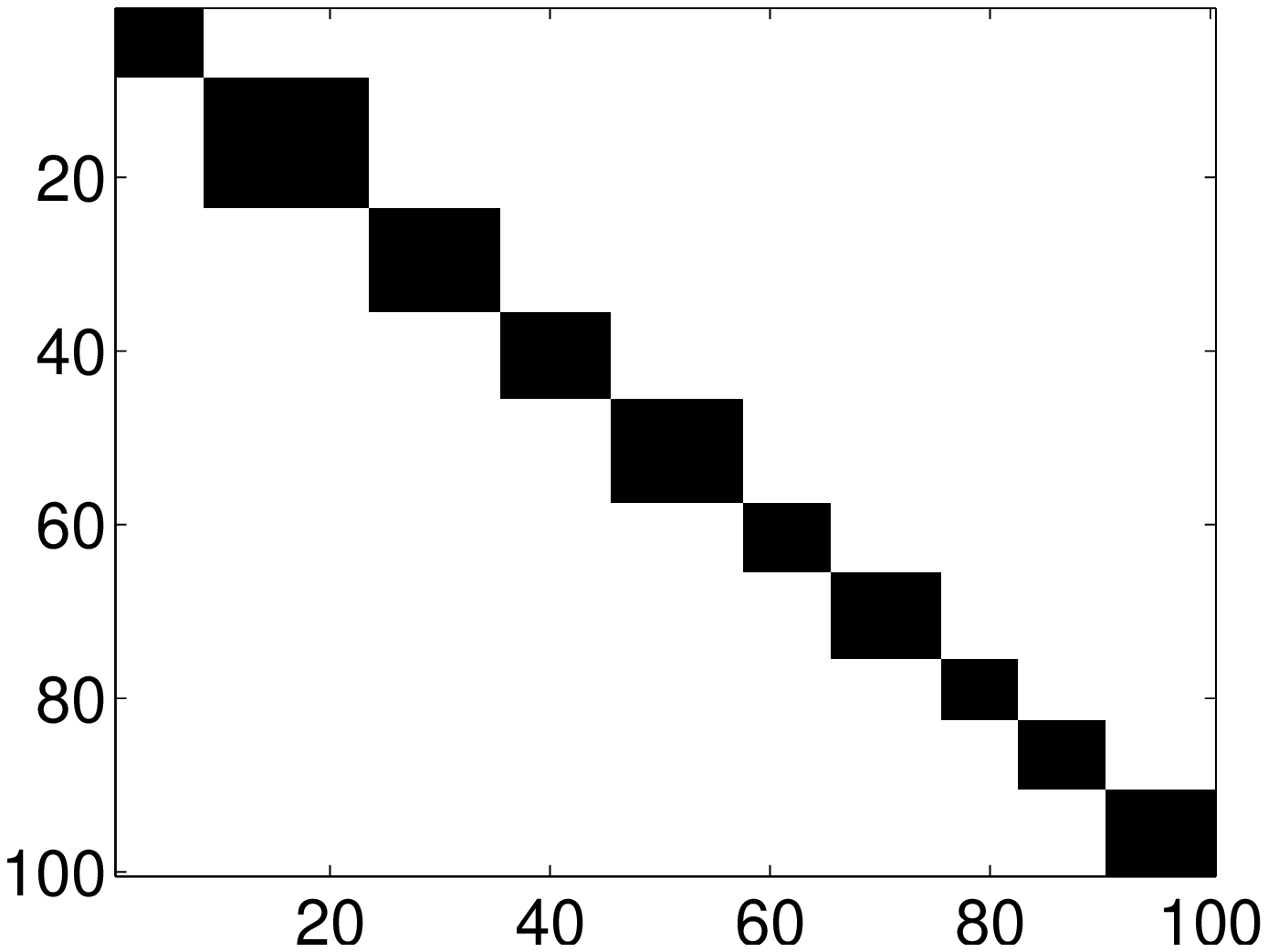}\label{fig:model1_heat_p_100}
\includegraphics[width=40mm]{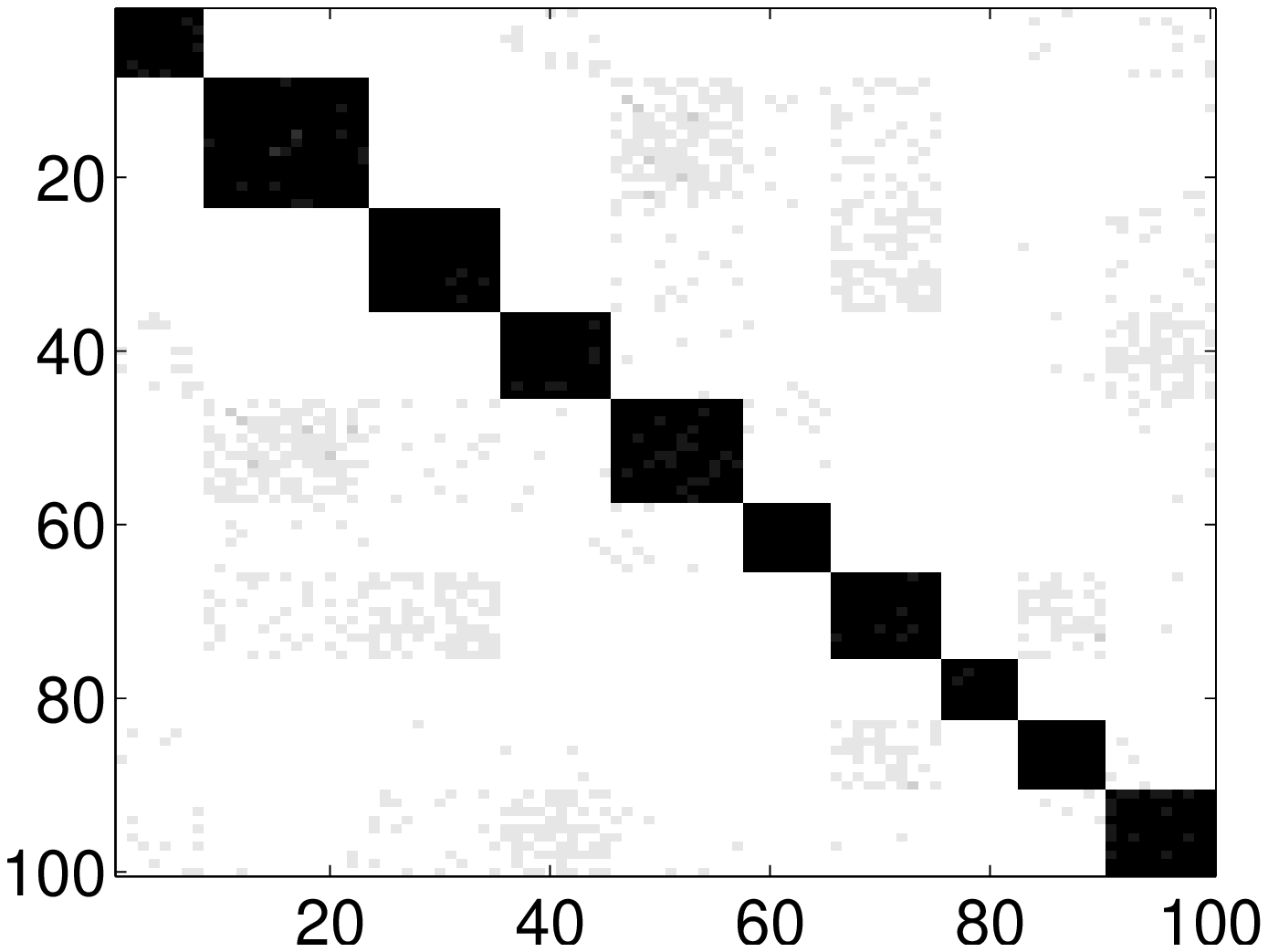}\label{fig:model1_heat_p_100_NLGGM}
\includegraphics[width=40mm]{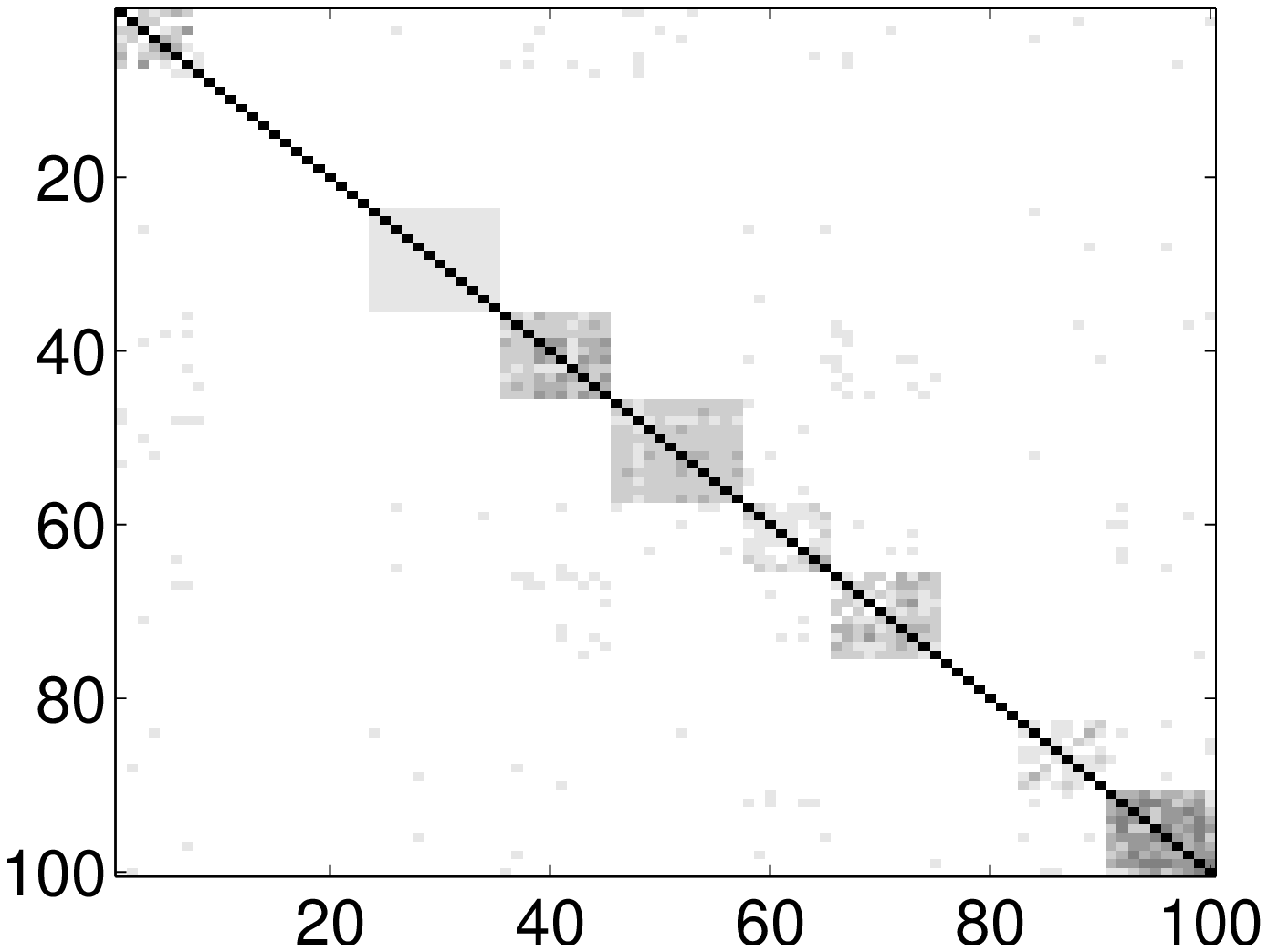}\label{fig:model1_heat_p_100_GLasso}
\includegraphics[width=40mm]{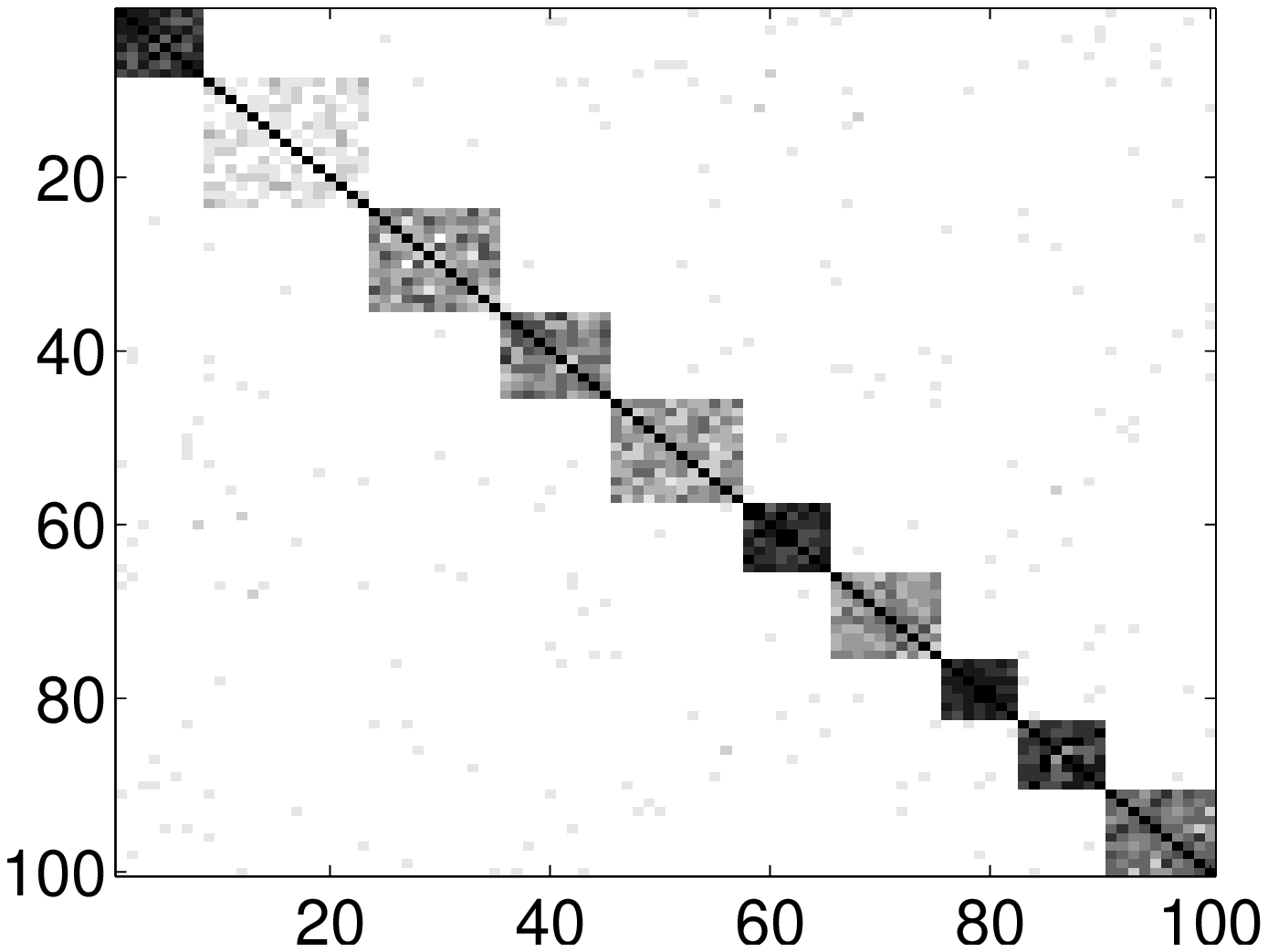}\label{fig:model1_heat_p_100_Nonparanormal}
} \subfigure[Data model 2 ($P=0.02$) with heat kernel, $p=100$]{
\includegraphics[width=40mm]{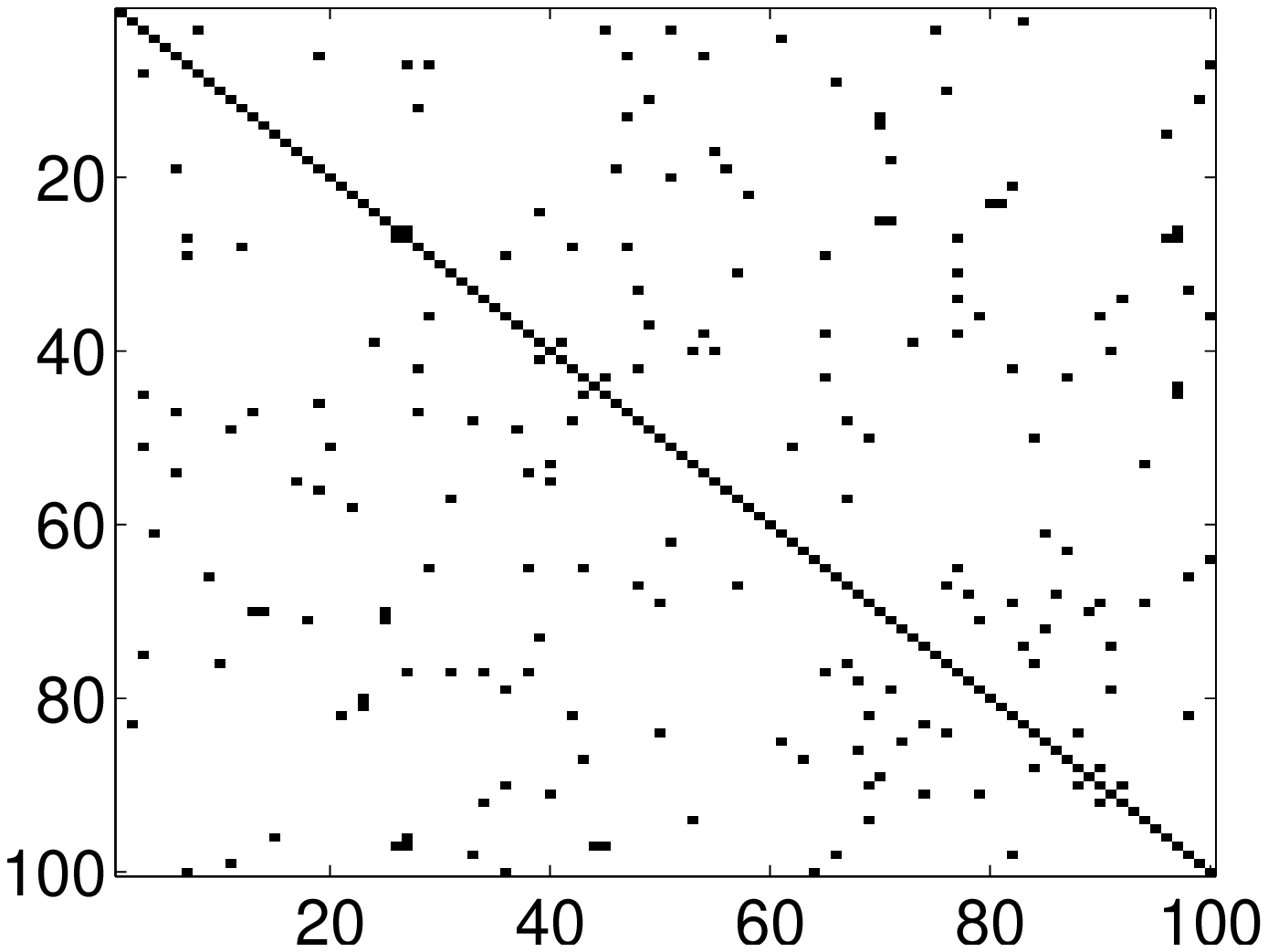}\label{fig:model2_heat_p_100_0.02}
\includegraphics[width=40mm]{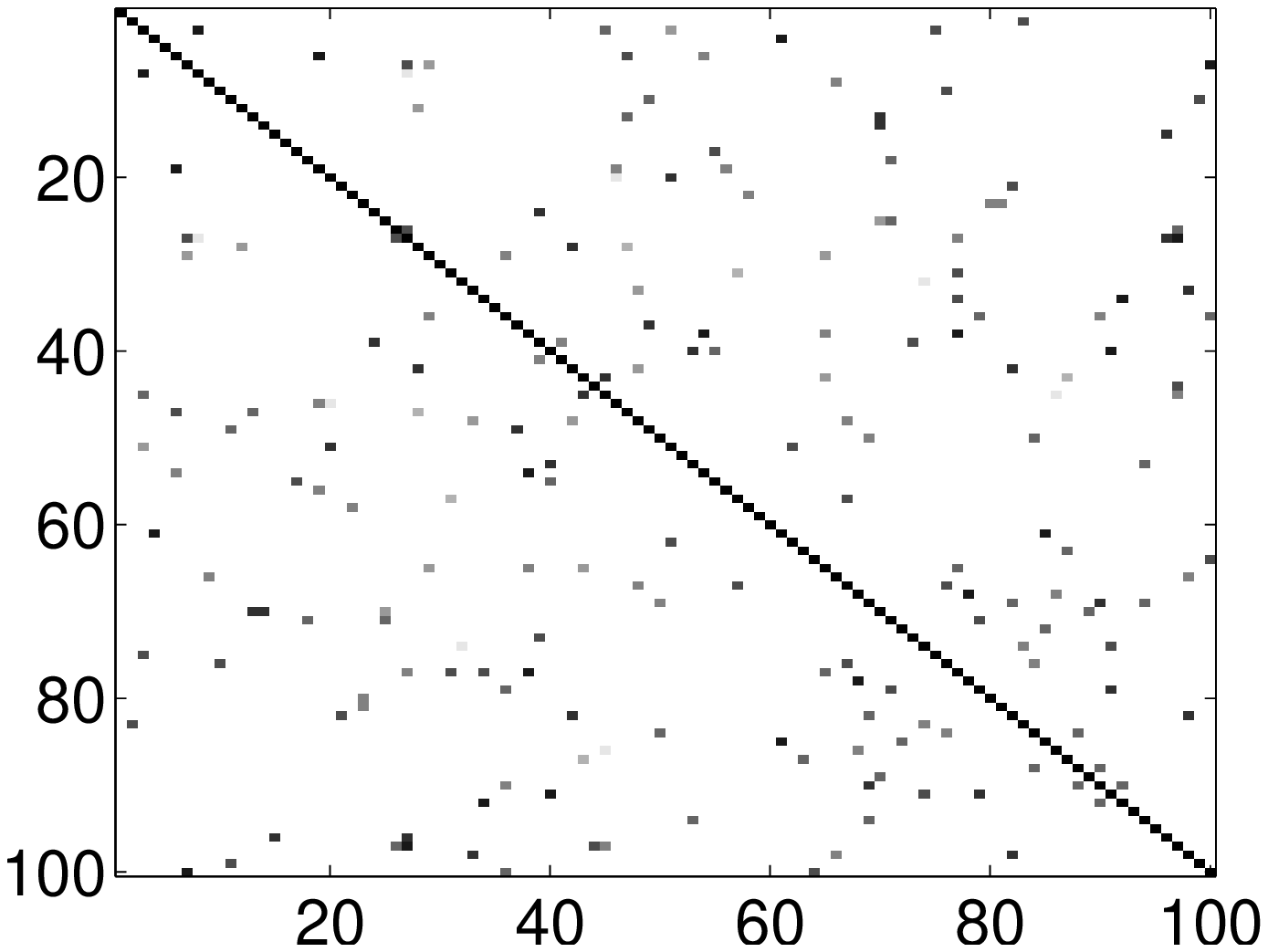}\label{fig:model2_poly_p_100_NLGGM_0.02}
\includegraphics[width=40mm]{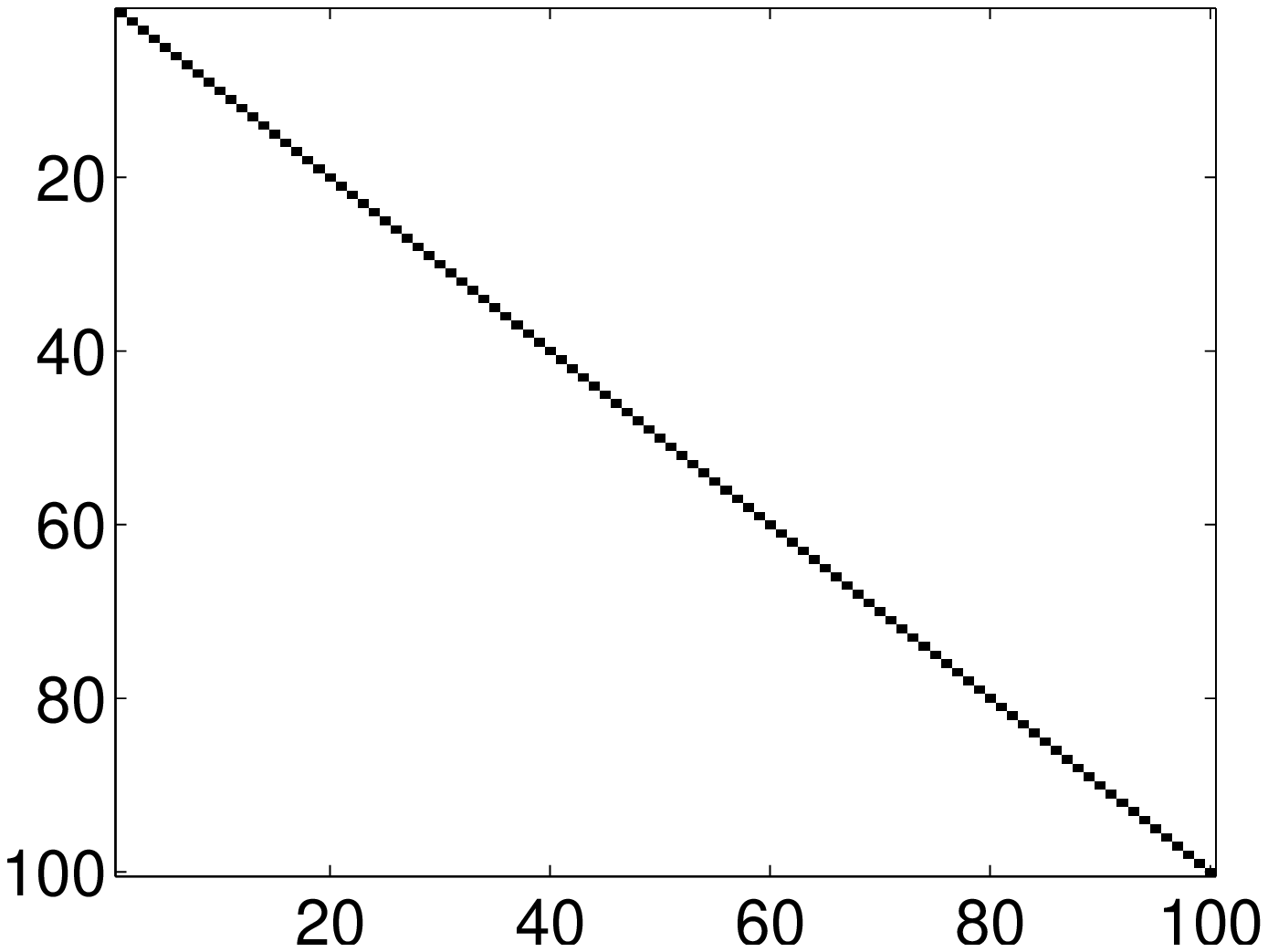}\label{fig:model2_poly_p_100_GLasso_0.02}
\includegraphics[width=40mm]{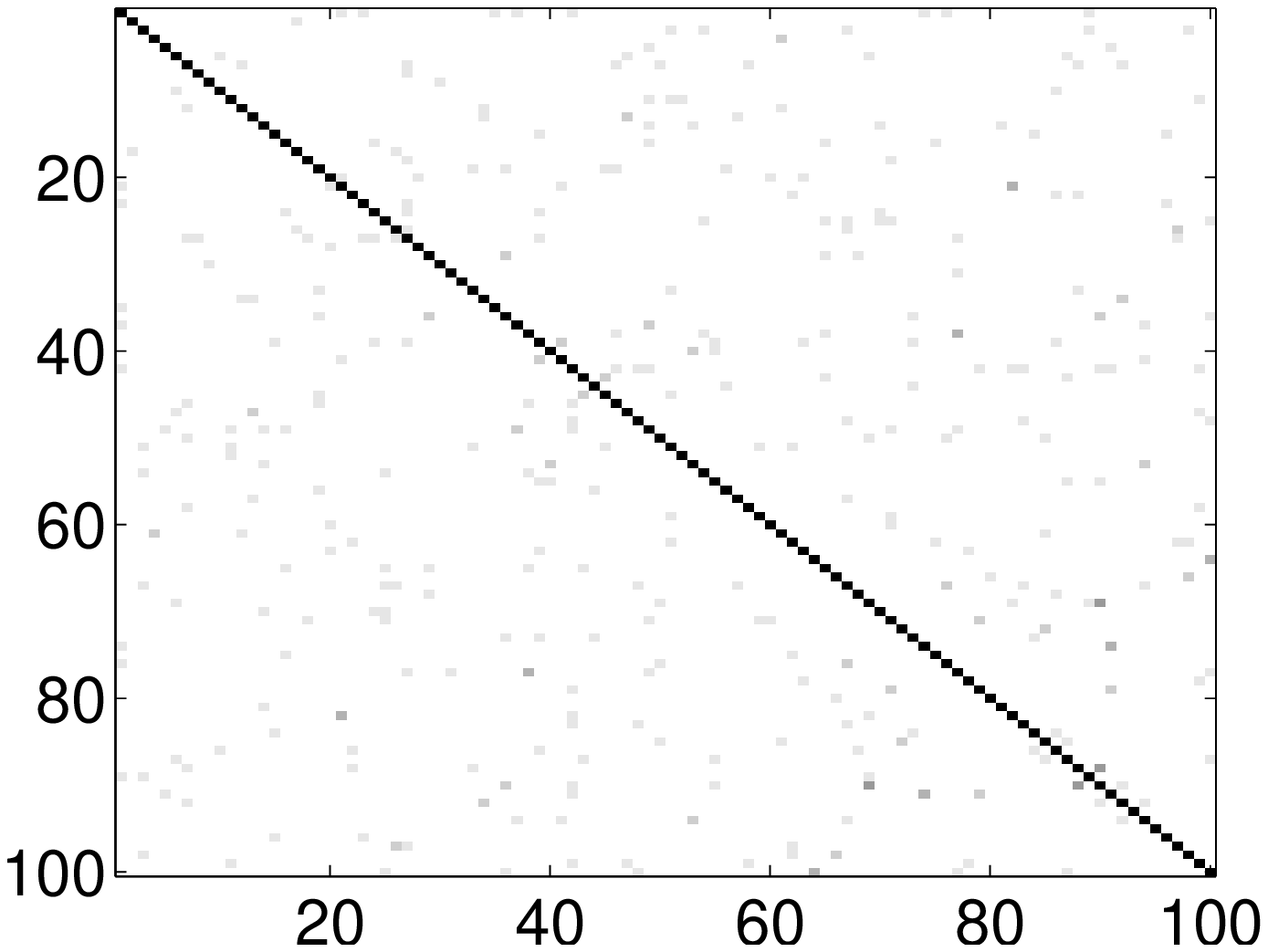}\label{fig:model2_poly_p_100_Nonparanormal_0.02}
} \subfigure[Data model 1 with polynomial kernel, $p=250$]{
\includegraphics[width=40mm]{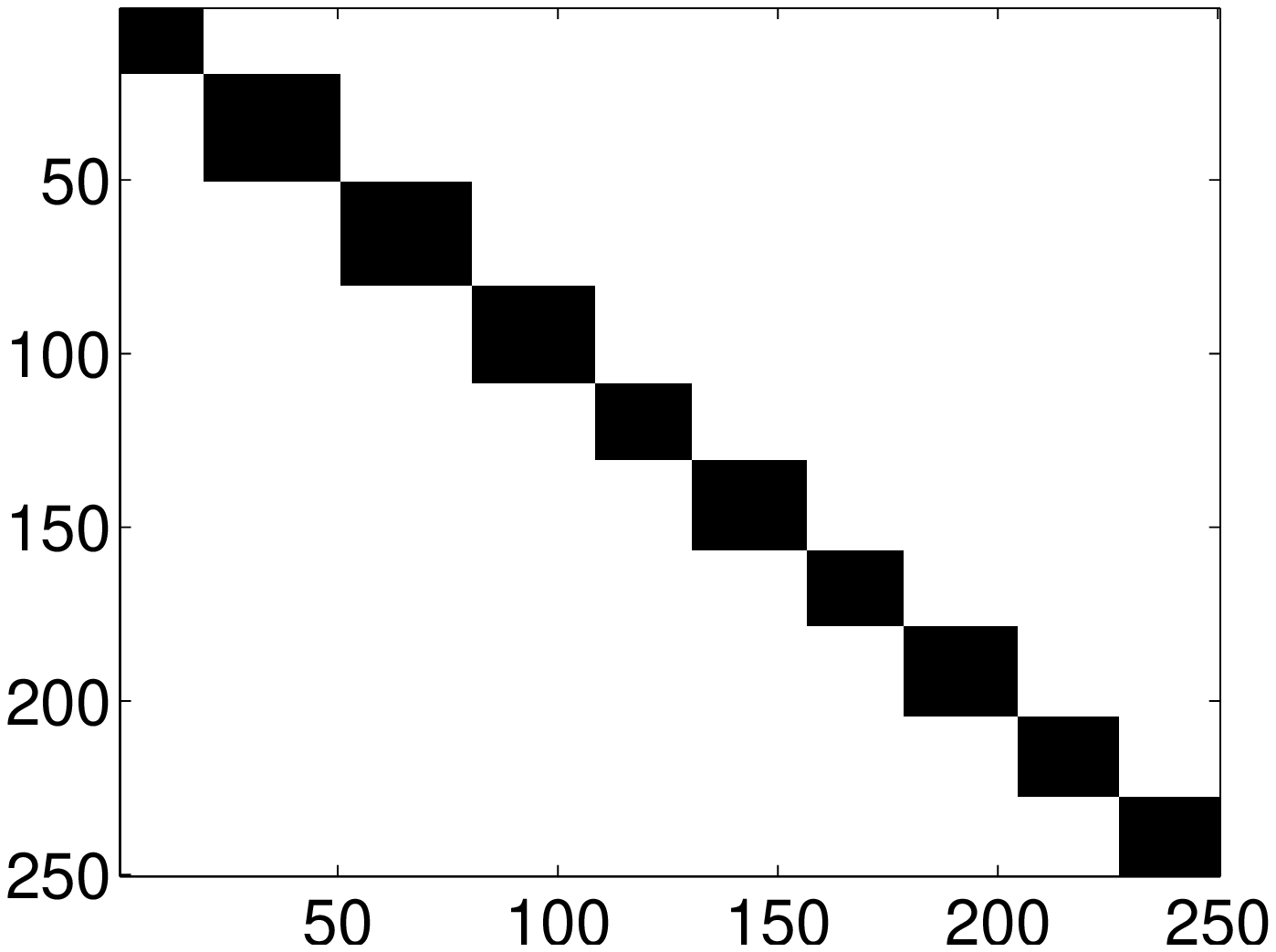}\label{fig:model1_heat_p_250}
\includegraphics[width=40mm]{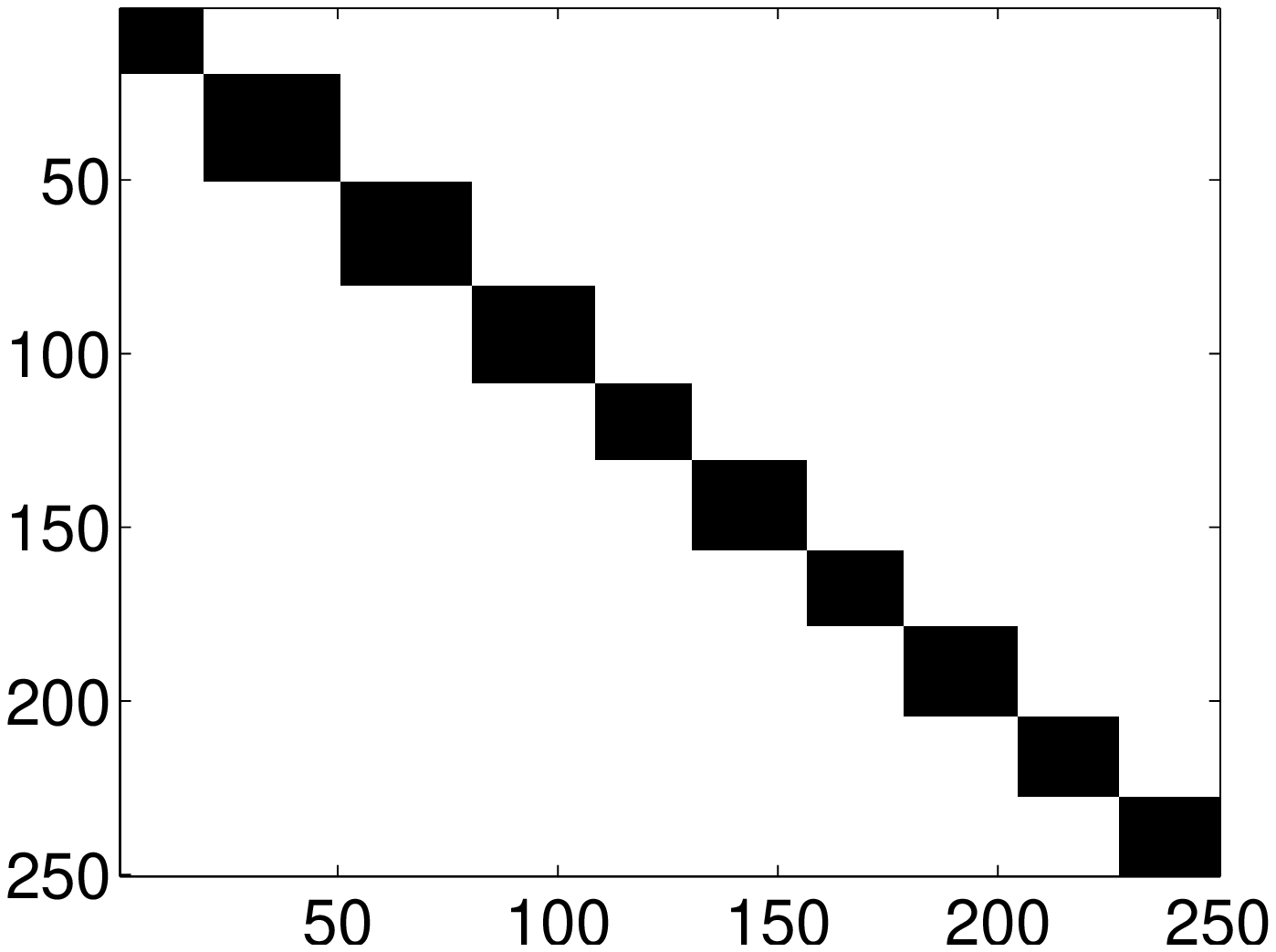}\label{fig:model1_heat_p_250_NLGGM}
\includegraphics[width=40mm]{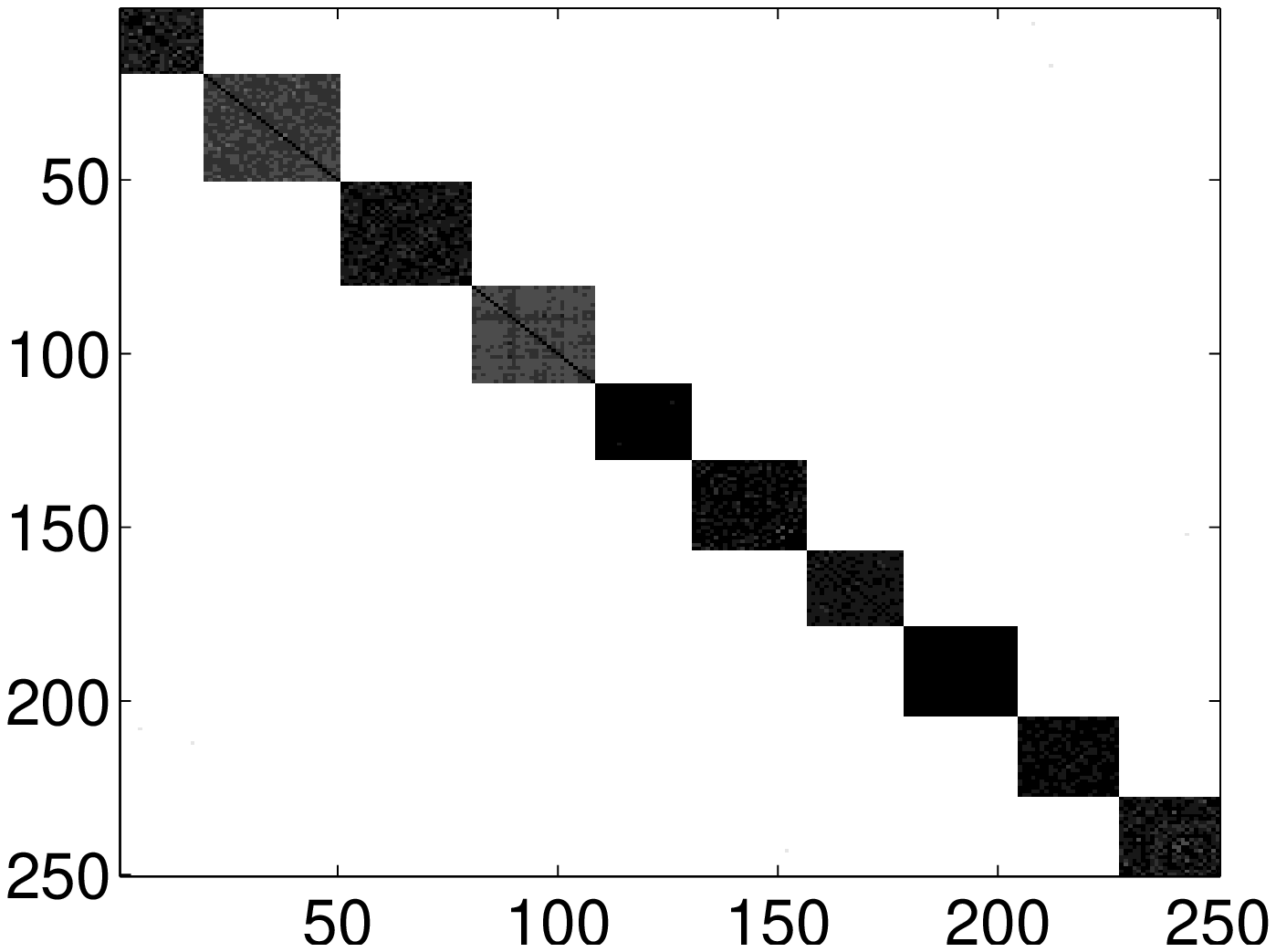}\label{fig:model1_heat_p_250_GLasso}
\includegraphics[width=40mm]{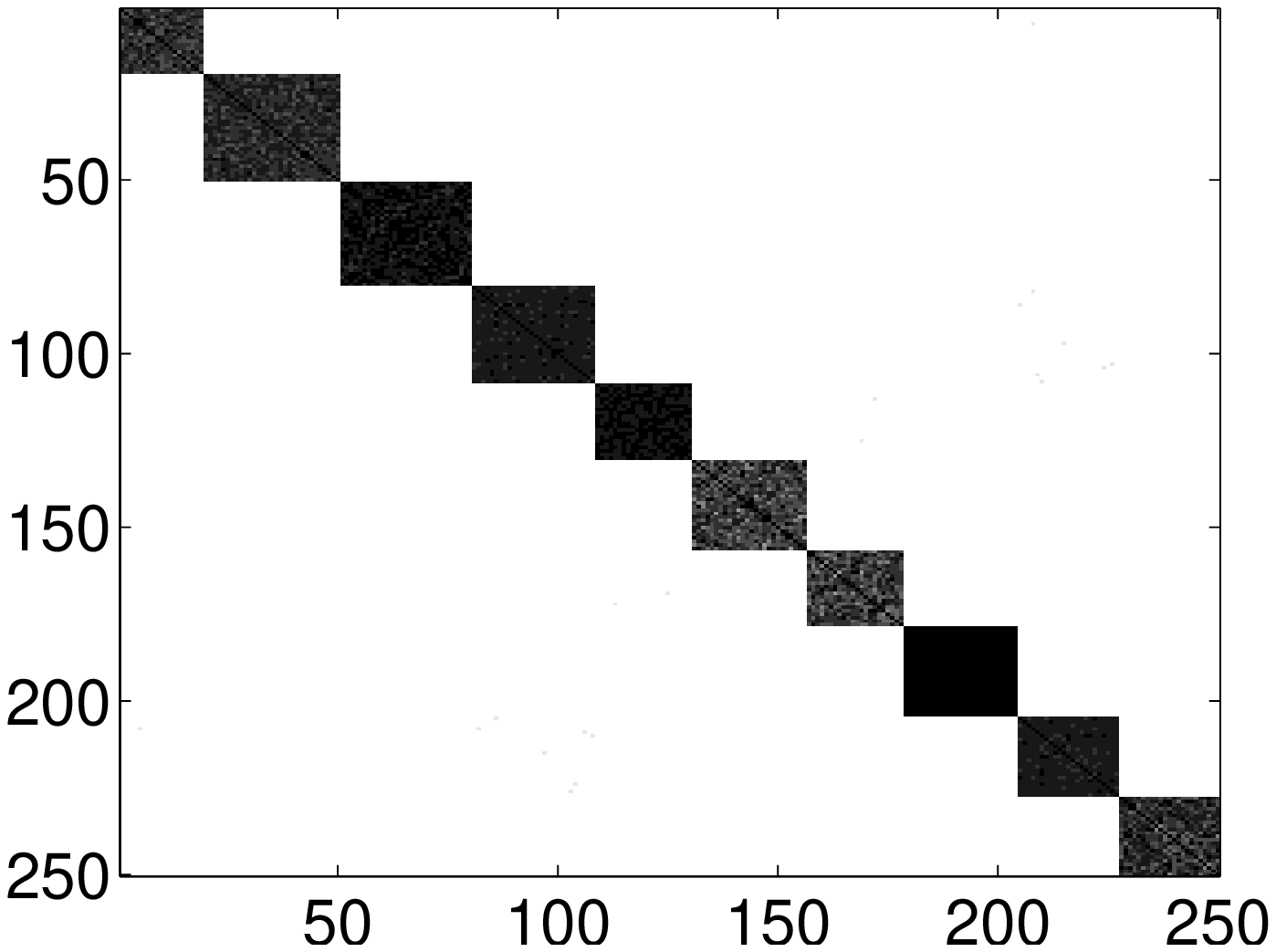}\label{fig:model1_heat_p_250_Nonparanormal}
} \subfigure[Data model 2 ($P = 0.02$) with polynomial kernel,
$p=250$]{
\includegraphics[width=40mm]{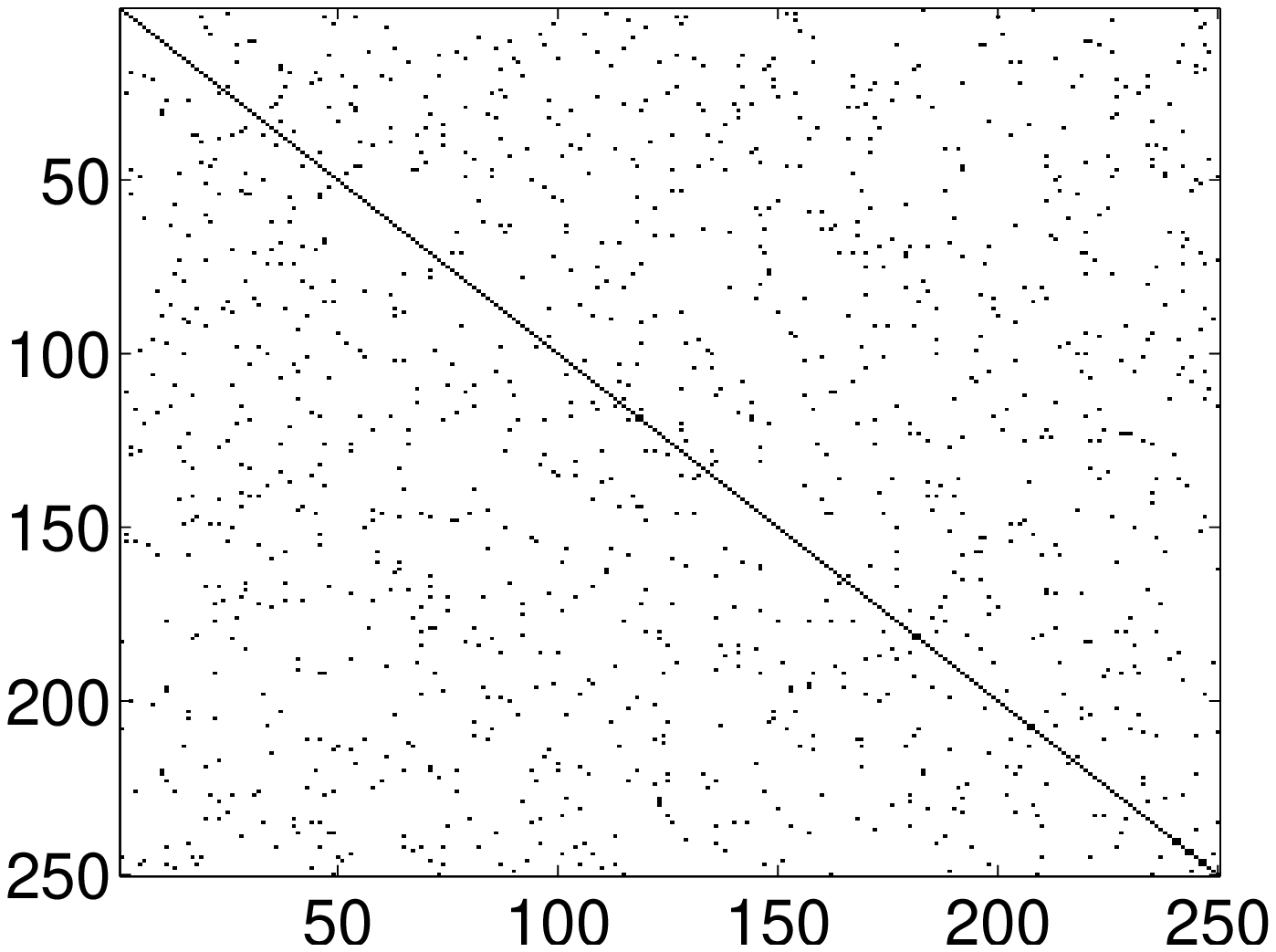}\label{fig:model2_heat_p_250_0.02}
\includegraphics[width=40mm]{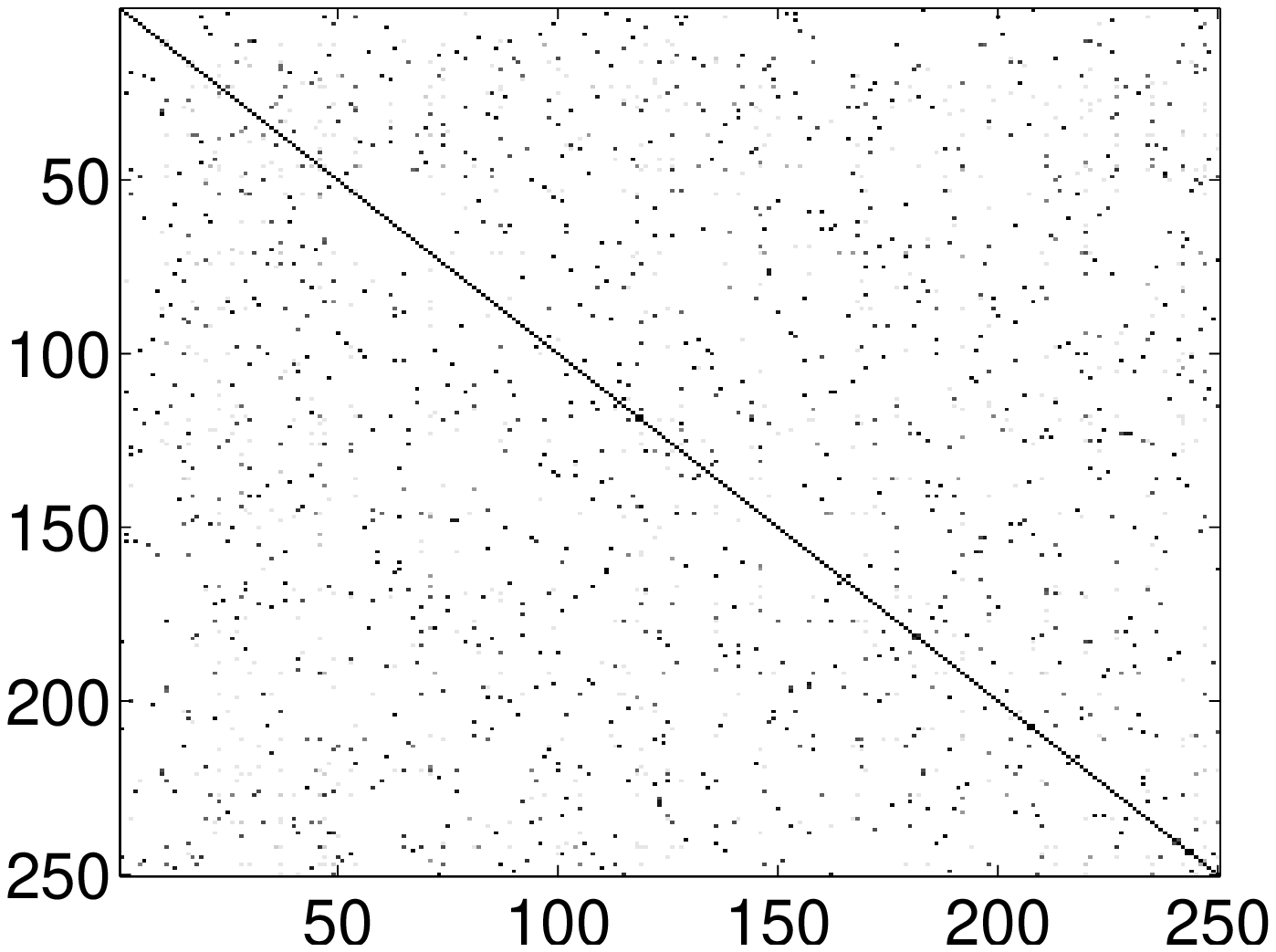}\label{fig:model2_poly_p_250_NLGGM_0.02}
\includegraphics[width=40mm]{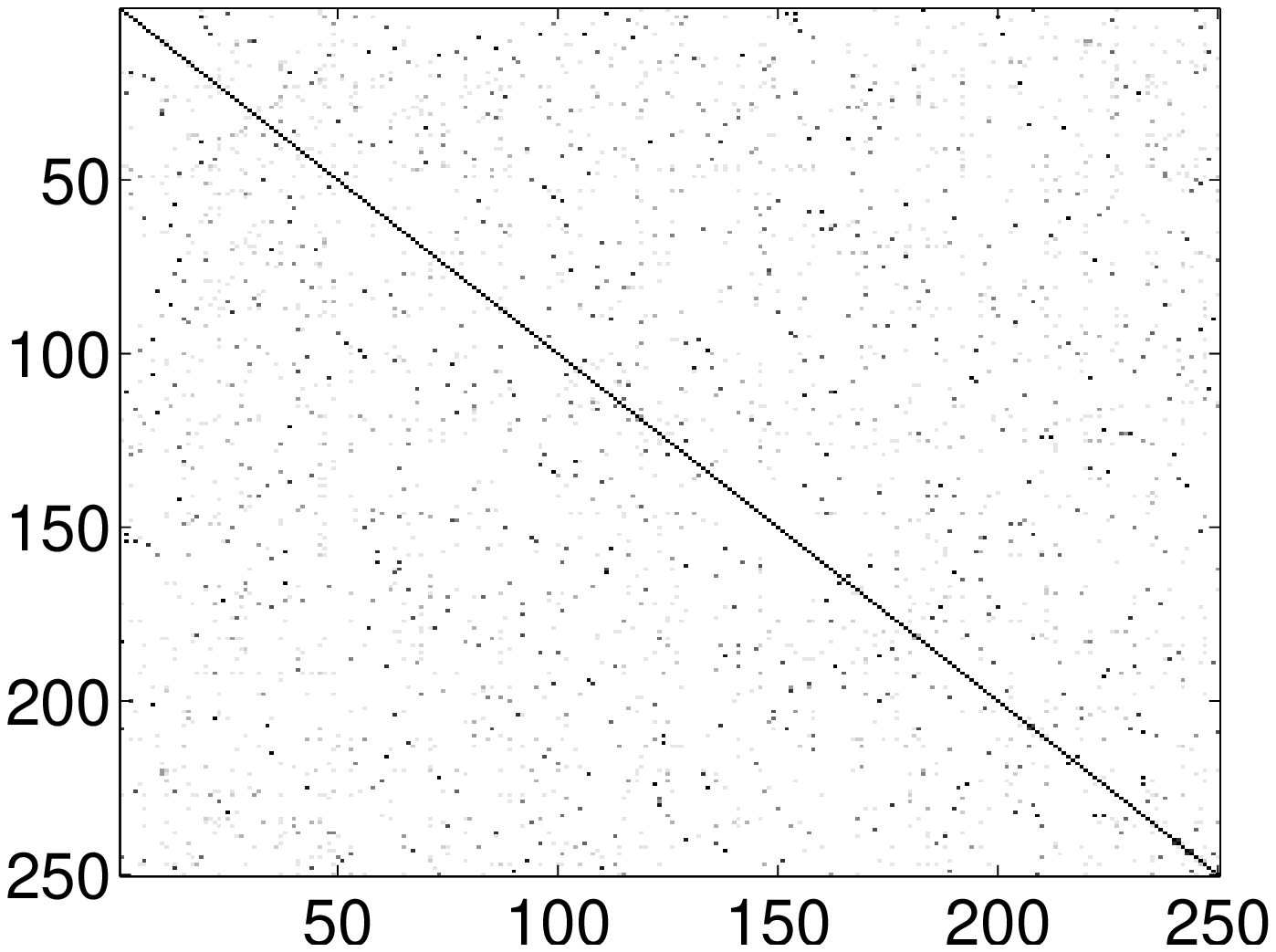}\label{fig:model2_poly_p_250_GLasso_0.02}
\includegraphics[width=40mm]{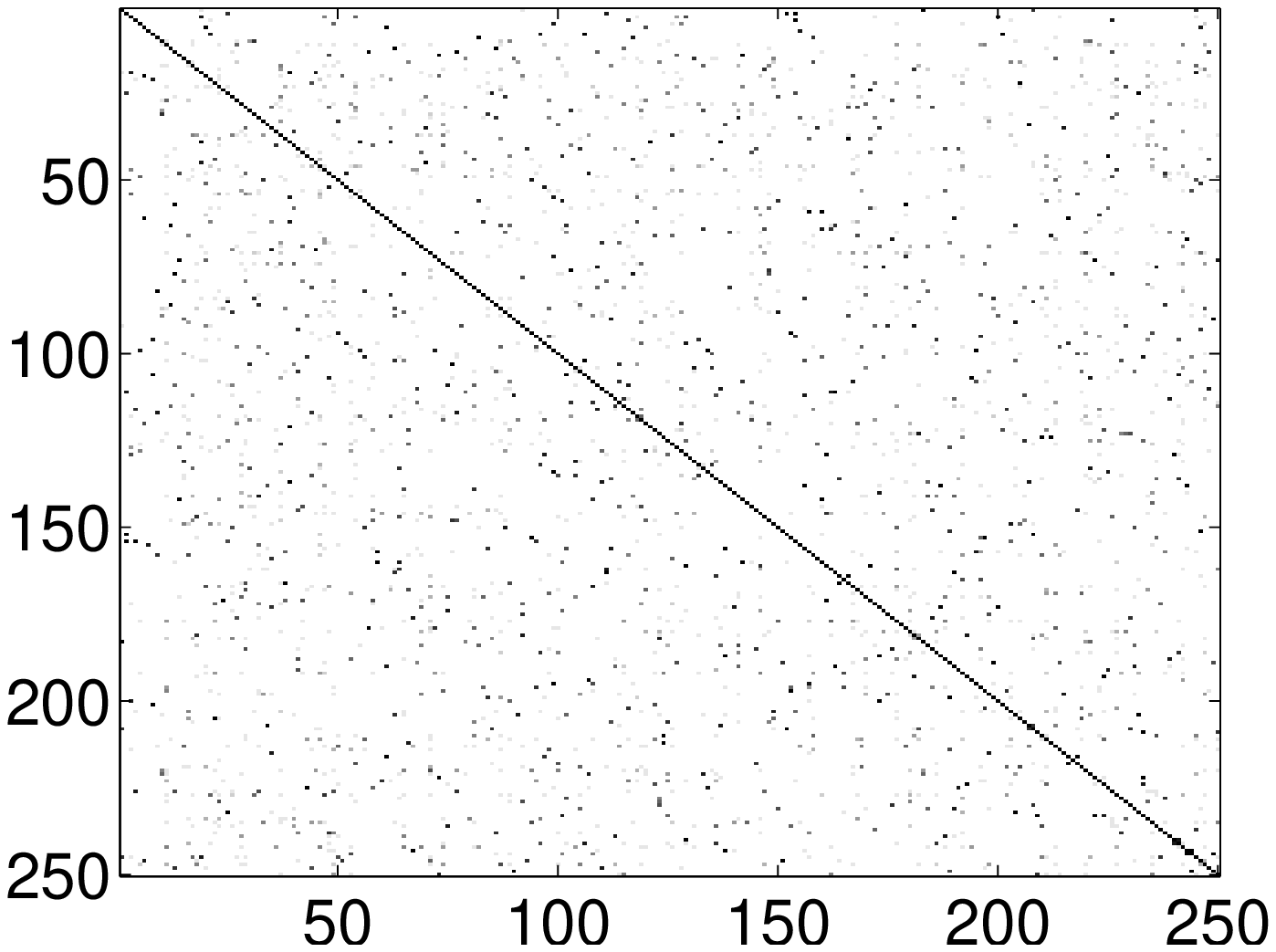}\label{fig:model2_poly_p_250_Nonparanormal_0.02}
}\caption{Heatmaps of the frequency of nonzeros identified for each
entry of the graph matrix out of 10 replications. White indicates 10
zeros identified out of 10 runs, and black indicates 0/10. For each
row, from left to right: \textbf{Ground truth, Semi-EFGMs, GGMs,
Nonparanormal}. \label{fig:heatmaps}}
\end{figure}

\subsection{Stock Price Data}

We further study the performance of Semi-EFGMs on a stock price
data. This data contains the historical prices of \textsf{S\&P500}
stocks over 5 years, from January 1, 2008 to January 1, 2013. By
taking out the stocks with less than 5 years of history, we end up
with 465 stocks, each having daily closing prices over 1,260 trading
days. The prices are first adjusted for dividends and splits and the
used to calculate daily log returns. Each day's return can be
represented as a point in $\mathbb{R}^{465}$. To apply Semi-EFGMs to
this data, we use the polynomial kernel $\phi(X_s,X_t)= ( \beta +
\varphi(X_s)^\top \varphi(X_t))^\alpha$ to model the pairwise
interactions between stocks. The feature vector $\varphi(X_s)$ is
defined as the $5$-day prices of each stock (thus the number of
samples reduces to $n=252$). Since the category information of
\textsf{S\&P500} is available, we  measure the performance by
precision, recall and F-score of the top $k$ links (edges) on the
constructed graph. A link is regarded as \emph{true} if and only if
it connects two nodes belonging to the same category. Note that the
category information is \emph{not} used in any of the graphical
model learning procedures. The parameters $\alpha$, $\beta$ and
$\lambda_n$ are tuned with cross-validation. Figure~\ref{fig:stocks}
shows the curves of precision, recall and F-score as functions of
$k$. It can be seen that Semi-EFGMs significantly outperform the
GGMs and Nonparanormal for identifying correct category links. This
result suggests that the interactions among the \textsf{S\&P500}
stocks is potentially highly nonlinear. We can also observe from
Figure~\ref{fig:stocks} that Nonparanormal is comparable or slightly
inferior to GGMs on this data.

\begin{figure}[h!]
\begin{center}
\mbox{
\includegraphics[width=2.2in]{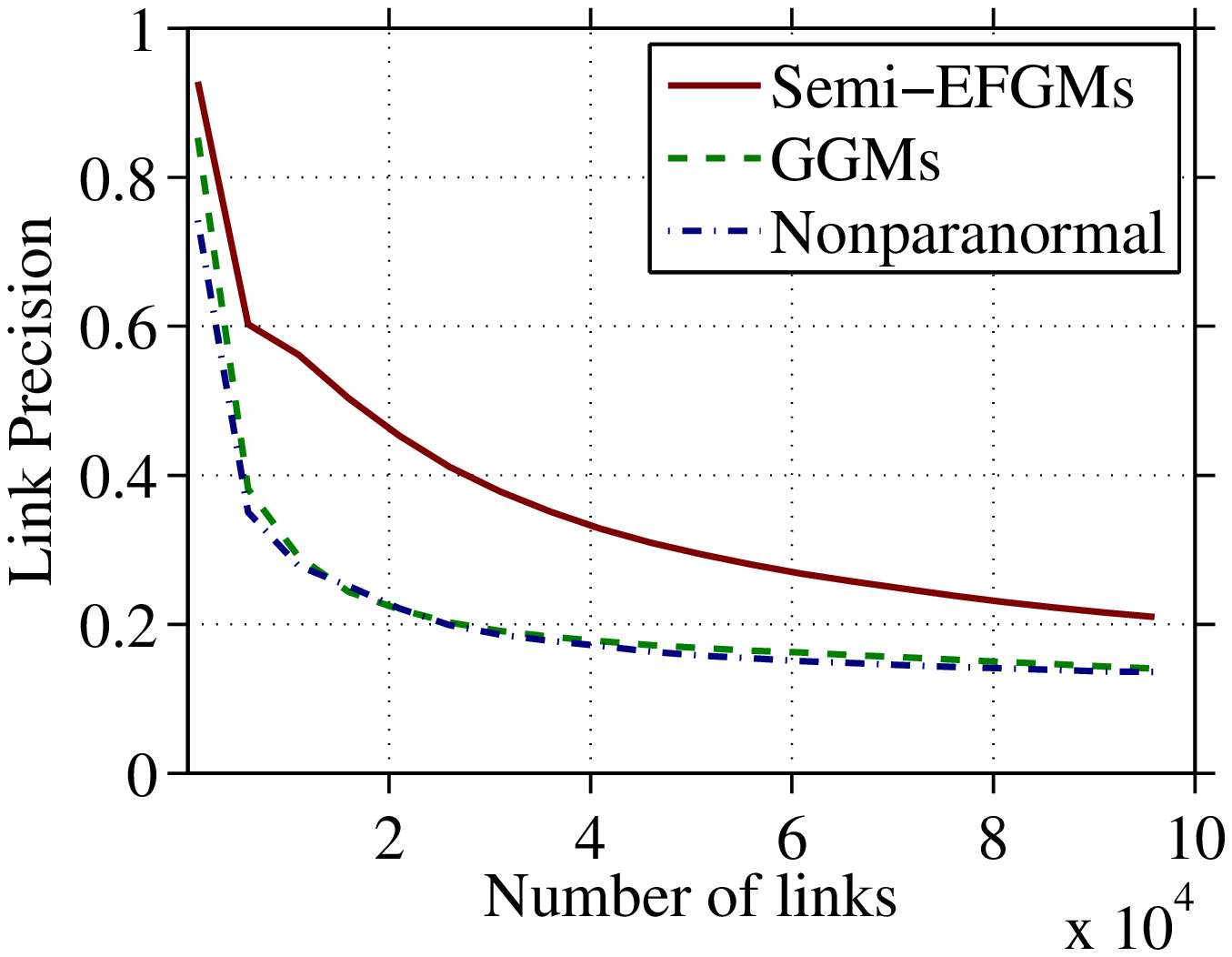}\label{fig:precision}\hspace{-0.1in}
\includegraphics[width=2.2in]{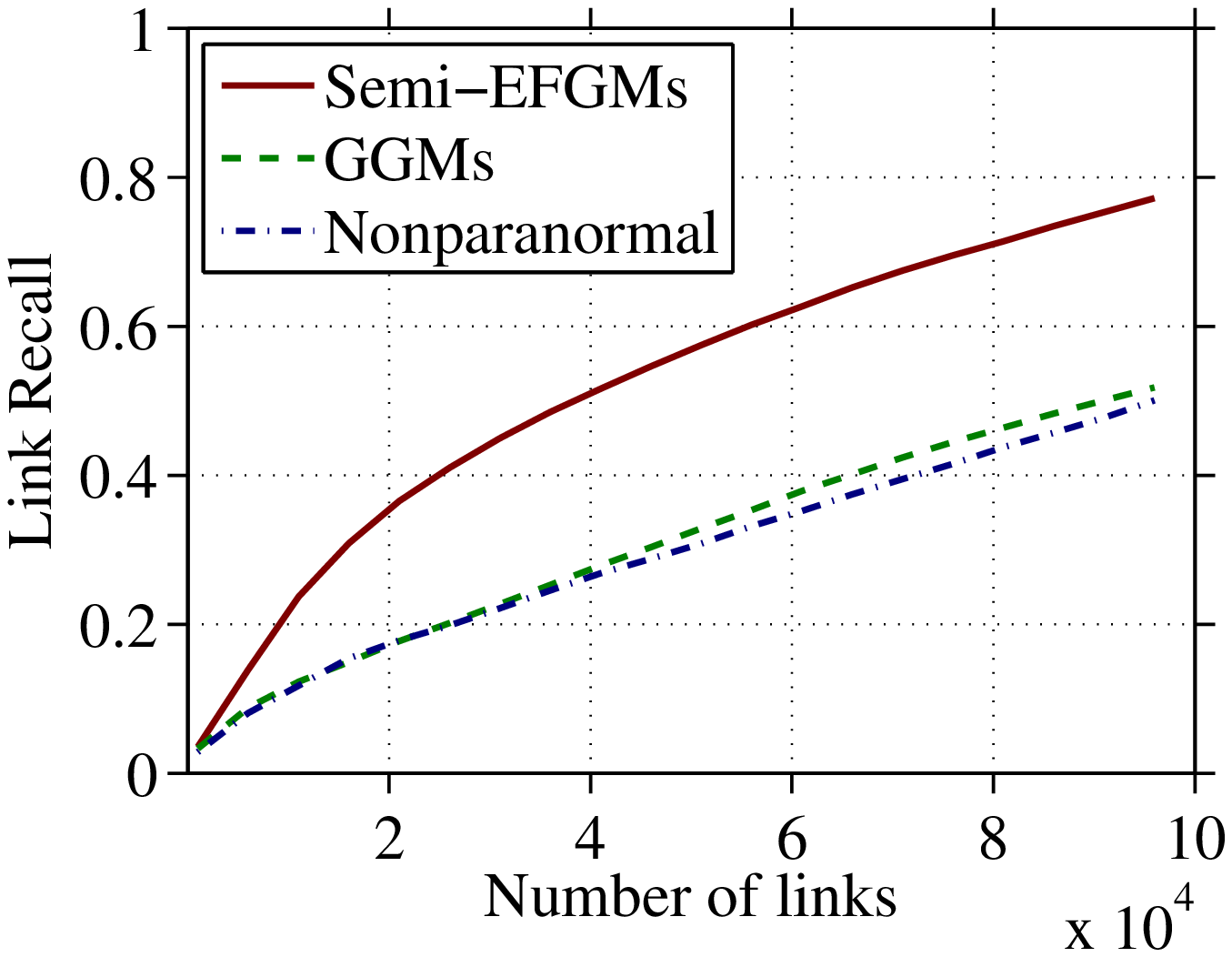}\label{fig:recall}\hspace{-0.1in}
\includegraphics[width=2.2in]{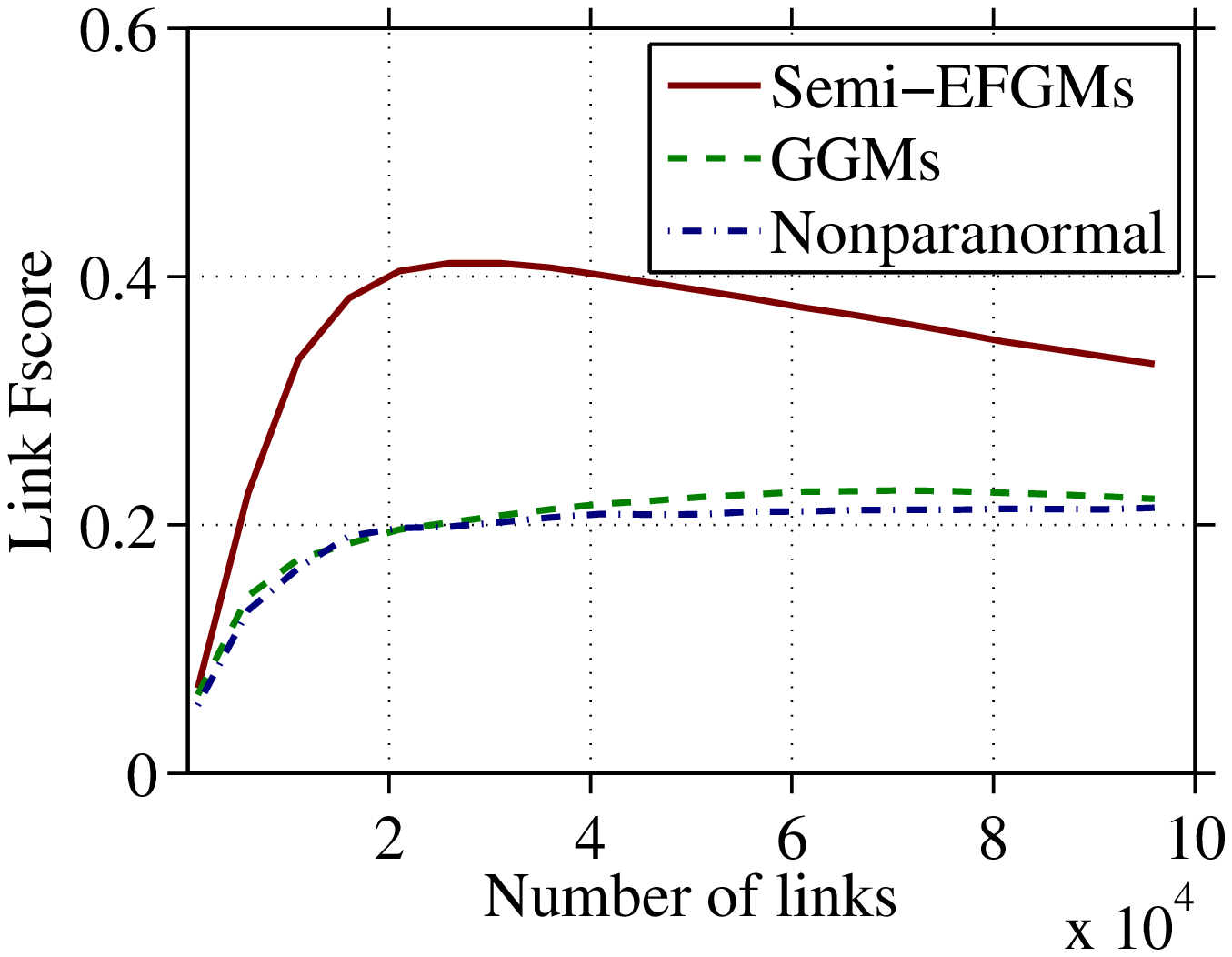}\label{fig:fscore}
}
\end{center}
\vspace{-0.25in}
 \caption{Category link
Precision, Recall and F-score curves of the considered methods on
the stock data \textsf{S\&P500}.}\label{fig:stocks}
\end{figure}

\vspace{-0.1in}
\section{Conclusion}\label{sect:conclusion}

In this paper, we propose Semi-EFGMs as a novel class of
semiparametric exponential family graphical models. The main idea is
to use a parametric nonlinear mapping families, e.g., Mercer
kernels, to compute pairwise sufficient statistics. This allows us
to capture complex interactions among variables which are not
uncommon in modern engineering applications. We investigate two
types of estimators, an $\ell_1$-regularized joint MLE estimator and
an $\ell_1$-regularized node-conditional MLE estimator, for model
parameters learning. Theoretically, we prove that under proper
conditions, our proposed estimators are consistent in parameter
estimation and the rates of convergence are optimal.
Computationally, we show that with proper relaxations, the
proposed estimators can be efficiently optimized via off-the-shelf
GGMs solvers. Empirically, we demonstrate the advantage of
Semi-EFGMs over the state-of-the-art parametric/semiparametric
methods when applied to synthetic and real data. To conclude,
Semi-EFGMs are statistically and computationally suitable for
learning pairwise graphical models with nonlinear sufficient
statistics. In the current model, we assume that the bivariate
mapping $\phi$ is known up to the tunable parameters. In future
work, we will investigate a more general model where $\phi$ admits a
linear combination of basis functions, e.g., over RKHS, so that the
sufficient statistics can be automatically learned in a data-driven
fashion.

\vspace{-0.1in}
\section*{Acknowledgment}
\vspace{-0.1in}

Xiao-Tong Yuan  was a postdoctoral research associate supported by
NSF-DMS 0808864 and NSF-EAGER 1249316. Ping Li is supported by
ONR-N00014-13-1-0764, AFOSR-FA9550-13-1-0137, and NSF-BigData
1250914. Tong Zhang is supported by NSF-IIS 1016061, NSF-DMS 1007527, and NSF-IIS 1250985.

\newpage

\appendix

\section{Proofs of Lemmas}
\subsection{Proof of Lemma~\ref{lemma_tail}}
\label{append:proof_lemma_tail}
\begin{proof}
Since $X^{(i)}$ are i.i.d. samples of $X$, we have that
$Z_{st}^{(i)} = \phi(X^{(i)}_s, X^{(i)}_t) - \mathbb{E}_{\theta^*}
[\phi(X_s, X_t)]$ are also i.i.d. samples of $Z_{st}$. We use the
exponential Markov inequality for the sum $Z = \sum_{i=1}^n
Z_{st}^{(i)}$ and with a parameter $\eta>0$
\[
\mathbb{P} \left(Z > \epsilon\right) = \mathbb{P} \left(\exp\{\eta
Z\}
> \exp\{\eta\epsilon\}\right) \le
\frac{\mathbb{E}[\exp\{\eta Z\}]}{\exp\{\eta\epsilon\}} =
\frac{\prod_{i=1}^n\mathbb{E}\left[\exp\left\{\eta
Z^{(i)}_{st}\right\}\right]}{\exp\{\eta \epsilon\}}.
\]
If $\eta\le\zeta$, Assumption~\ref{assump:tail_1} yields
\[
\mathbb{P} \left(Z > n\varepsilon\right) \le
\frac{\exp\left\{n\sigma^2\eta^2/2\right\}}{\exp\{\eta n\varepsilon\}} =
\exp\left\{-\eta n\varepsilon + n\sigma^2\eta^2/2\right\},
\]
whose minimum is attained at $\eta = \min\left(\frac{\varepsilon}{\sigma^2},\zeta\right)$. Thus, for any $\varepsilon\le\sigma^2\zeta$, we have
\[\mathbb{P}\left(Z > n\varepsilon\right) \le
\exp\left\{-\frac{n\varepsilon^2}{2\sigma^2}\right\}.
\]
Repeating this argument for $-Z^{(i)}_{st}$ instead of
$Z^{(i)}_{st}$, we obtain the same bound for $\mathbb{P}(-Z >
n\varepsilon)$. Combining these two bounds yields
\[
\mathbb{P}\left(\left|\frac{1}{n}\sum_{i=1}^n
\phi(X^{(i)}_s,X^{(i)}_t) -
\mathbb{E}_{\theta^*}[\phi(X_s,X_t)]\right|
> \varepsilon\right) = \mathbb{P}\left(|Z| > n\varepsilon\right) \le 2
\exp\left\{-\frac{n\varepsilon^2}{2\sigma^2}\right\}.
\]
This completes the proof.
\end{proof}

\subsection{Proof of Lemma~\ref{lemma:norm_joint}}
\label{append:proof_lemma_norm_joint}
\begin{proof}
From the gradient term~\eqref{equat:joint_derivatives} and
Lemma~\ref{lemma_tail} we have that for any index pair $(s,t)$ and
$\varepsilon < \sigma^2\zeta$
\begin{eqnarray}
\mathbb{P}\left(\left|\frac{\partial L
(\theta^*;\mathbb{X}_n)}{\partial \theta^*_{st}}\right| >
\varepsilon\right) &=& \mathbb{P}\left(\left|\frac{1}{n}\sum_{i=1}^n
\phi(X^{(i)}_s,X^{(i)}_t)- \mathbb{E}_{\theta^*} [\phi(X_s,X_t)]
\right| > \varepsilon\right) \le  2
\exp\left\{-\frac{n\varepsilon^2}{2\sigma^2}\right\}\nonumber.
\end{eqnarray}
By the union  bound we obtain
\begin{eqnarray}
\mathbb{P}(\|\nabla L(\theta^*;\mathbb{X}_n))\|_\infty >
\varepsilon) &\le&
2p^2\exp\left\{-\frac{n\varepsilon^2}{2\sigma^2}\right\} \nonumber.
\end{eqnarray}
Let us choose $\varepsilon = \sigma\sqrt{6\ln p/n}$. Since $n > 6
\ln p/(\sigma^2\zeta^2)$, we have $ \varepsilon < \sigma^2\zeta$.
Therefore we obtain that with probability at least $1-2p^{-1}$,
\[
\|\nabla L(\theta^*;\mathbb{X}_n)\|_\infty \le \sigma \sqrt{6\ln
p/n}.
\]
This completes the proof.
\end{proof}

\subsection{Proof of Lemma~\ref{lemma:error_bound_joint}}
\label{append:proof_lemma_error_bound_joint}

\begin{proof}
Let $\Delta \theta = \hat\theta_n - \theta^*$ and we define
$\Delta\tilde\theta = t\Delta\theta$ where we pick $t=1$ if
$\|\Delta\theta\|<r$ and $t \in (0,1)$ with
$\|\Delta\tilde\theta\|=r$ otherwise. By definition, we have
$\|\Delta\tilde\theta\| < r$. We now claim that $\|\Delta
\tilde\theta_{\bar S}\|_1 \le 3 \|\Delta\tilde\theta_S\|_1$. Indeed,
since $\theta^*_{\bar S} = 0$, we have
\begin{eqnarray}\label{equat:lemma_theta_joint_bound_1}
\|\theta^* + \Delta \tilde\theta\|_1 - \|\theta^*\|_1 = \|(\theta^*
+ \Delta\tilde\theta)_S\|_1 + \|\Delta\tilde\theta_{\bar S}\|_1 -
\|\theta^*_S\|_1 \ge \|\Delta\tilde\theta_{\bar S}\|_1 -
\|\Delta\tilde\theta_S\|_1.
\end{eqnarray}
From the convexity of function $L(\theta;\mathbb{X}_n)$ and $\lambda
\ge 2\gamma_n = 2\|\nabla L(\theta^*;\mathbb{X}_n)\|_\infty$ we have
\begin{equation}\label{equat:lemma_theta_joint_bound_2}
L(\theta^* + \Delta \tilde\theta;\mathbb{X}_n) -
L(\theta^*;\mathbb{X}_n) \ge  \langle \nabla
L(\theta^*;\mathbb{X}_n), \Delta \tilde\theta\rangle \ge -\|\nabla
L(\theta^*;\mathbb{X}_n)\|_\infty \|\Delta\tilde\theta\|_1 \ge -
\frac{\lambda_n }{2} \|\Delta\tilde\theta\|_1.
\end{equation}
Due to the optimality of $\hat\theta_n$ and the convexity of
$L(\theta;\mathbb{X}_n)$, it holds that
\begin{equation}\label{equat:lemma_theta_joint_bound_3}
L(\theta^*+ \Delta\tilde\theta;\mathbb{X}_n) + \lambda_n \|\theta^*
+ \Delta\tilde\theta\|_1 \le L(\theta^*;\mathbb{X}_n) + \lambda_n
\|\theta^*\|_1.
\end{equation}
By combining the proceeding three
inequalities~\eqref{equat:lemma_theta_joint_bound_1},~\eqref{equat:lemma_theta_joint_bound_2}
and~\eqref{equat:lemma_theta_joint_bound_3}, we obtain that
\begin{eqnarray}
0 &\ge& L(\theta^* + \Delta \tilde\theta;\mathbb{X}_n) +
\lambda_n\|\theta^* + \Delta \tilde\theta\|_1 -
L(\theta^*;\mathbb{X}_n) -
\lambda_n \|\theta^*\|_1 \nonumber \\
&\ge& -\frac{\lambda_n}{2} (\|\Delta\tilde\theta_S\|_1 +
\|\Delta\tilde\theta_{\bar S}\|_1) + \lambda_n
(\|\Delta\tilde\theta_{\bar S}\|_1 - \|\Delta\tilde\theta_S\|_1),
\nonumber
\end{eqnarray}
which implies $\|\Delta\tilde\theta_{\bar S}\|_1 \le 3
\|\Delta\tilde\theta_S\|_1$. From second-order Taylor expansion we
know that there exists a real number $\xi \in [0,1]$ such that
\begin{equation}
L(\theta^* + \Delta \tilde\theta;\mathbb{X}_n) =
L(\theta^*;\mathbb{X}_n) +  \langle \nabla L(\theta^*;\mathbb{X}_n),
\Delta \tilde\theta\rangle + \frac{1}{2}\tilde\Delta\theta^\top
\nabla^2 L(\theta^* + \xi\Delta
\tilde\theta;\mathbb{X}_n)\tilde\Delta\theta. \nonumber
\end{equation}
By using Assumption~\ref{assump:positive_definite} (note that
$\|\xi\tilde\Delta\theta\| \le \|\tilde\Delta\theta\|< r$)
and~\eqref{equat:lemma_theta_joint_bound_2} we have
\begin{equation}\label{equat:lemma_theta_joint_bound_4}
L(\theta^* + \Delta \tilde\theta;\mathbb{X}_n) -
L(\theta^*;\mathbb{X}_n) \ge  \langle \nabla
L(\theta^*;\mathbb{X}_n), \Delta \tilde\theta\rangle +
\frac{\beta}{2}\|\tilde\Delta\theta\|^2 \ge -\frac{\lambda_n }{2}
\|\Delta\tilde\theta\|_1 + \frac{\beta}{2}\|\tilde\Delta\theta\|^2.
\end{equation}
By combining the
inequalities~\eqref{equat:lemma_theta_joint_bound_1},~\eqref{equat:lemma_theta_joint_bound_3}
and~\eqref{equat:lemma_theta_joint_bound_4}, we obtain that
\begin{eqnarray}
0 &\ge& L(\theta^* + \Delta\tilde\theta;\mathbb{X}_n) +
\lambda_n\|\theta^* + \Delta\tilde\theta\|_1 -
L(\theta^*;\mathbb{X}_n) -
\lambda_n \|\theta^*\|_1 \nonumber \\
&\ge& -\frac{\lambda_n }{2} \|\Delta\tilde\theta\|_1  +
\frac{\beta}{2}\|\tilde\Delta\theta\|^2+
\lambda_n(\|\tilde\Delta\theta_{\bar S}\|_1 -
\|\tilde\Delta\theta_S\|_1)
\nonumber \\
&\ge& \frac{\lambda_n}{2} (\|\Delta\tilde\theta_{\bar S}\|_1 - 3
\|\Delta\tilde\theta_S\|_1) + \beta\|\Delta\tilde\theta\|^2
\nonumber \\
&\ge& - 1.5 \lambda_n \|\Delta\tilde\theta_S\|_1 +
\beta\|\Delta\tilde\theta\|^2 \ge - 1.5 \lambda_n
\sqrt{|S|}\|\Delta\tilde\theta\| + \beta\|\Delta\tilde\theta\|^2,
\nonumber
\end{eqnarray}
which implies that
\[
\|\Delta\tilde\theta\| \le 1.5 \lambda_n \beta^{-1}
\sqrt{\|\theta^*\|_0} \le 1.5 c_0 \sqrt{\|\theta^*\|_0} \beta^{-1}
\gamma_n =\gamma.
\]
Since $\gamma<r$, we claim that $t=1$ and thus
$\Delta\tilde\theta=\Delta\theta$. Indeed, if otherwise $t<1$, then
$\|\Delta\tilde\theta\|=r>\gamma$ which contradicts the above
inequality. This completes the proof.
\end{proof}

\subsection{Proof of Lemma~\ref{lemma:tail_mgf}}
\label{append:proof_lemma_tail_mgf}

\begin{proof}
Note that for any $\eta$, $\exp\{\eta x\}$ is convex with respect to
$x$. By applying Jensen's inequality we have
\[
\exp\left\{\eta \mathbb{E}_{\theta^*_s} [\phi(X_s,X_t) \mid
X_{\s}]\right\} \le \mathbb{E}_{\theta^*_s} [\exp\left\{\eta
\phi(X_s,X_t) \right\}\mid X_{\s}].
\]
By taking the expectation $\mathbb{E}_{\theta^*_{\s}}[\cdot]$ with
respect to the marginal distribution of $X_{\s}$, and using the rule
of iterated expectation, we obtain
\[
\mathbb{E}_{\theta^*_{\s}}\left[\exp\left\{\eta
\mathbb{E}_{\theta^*_s} [\phi(X_s,X_t) \mid X_{\s}]\right\}\right]
\le \mathbb{E}_{\theta^*_{\s}}\left[ \mathbb{E}_{\theta^*_s}
[\exp\left\{\eta \phi(X_s,X_t) \right\}\mid X_{\s}]\right] =
\mathbb{E}_{\theta^*}[\exp\{\eta \phi(X_s,X_t)\}].
\]
By using the ``law of the unconscious statistician'' and the above
inequality we obtain
\[
\mathbb{E}[\exp\{\eta\tilde Z_{st}\}] \le \mathbb{E}[\exp\{\eta
Z_{st}\}] \le \exp\left\{\sigma^2 \eta^2/2\right\},
\]
where the last inequality follows from
Assumption~\ref{assump:tail_1}. This completes the proof.
\end{proof}

\subsection{Proof of Lemma~\ref{lemma:norm}}
\label{append:proof_lemma_norm}

\begin{proof}
Recall the formulation of gradient $\nabla \tilde
L(\theta_s;\mathbb{X}_n)$ in~\eqref{equat:derivatives}. For any node
$t \in V\s$, we have
\begin{eqnarray}
&&\left|\frac{\partial \tilde L(\theta_s; \mathbb{X}_n)}{\partial
\theta_{st}}\right| \nonumber \\
&=& \left|\frac{1}{n}\sum_{i=1}^n - \phi(X^{(i)}_s,X^{(i)}_t) +
\mathbb{E}_{\theta^*_s}[\phi(X_s,X^{(i)}_t)\mid X^{(i)}_{\s}]\right| \nonumber \\
&\le& \left|\frac{1}{n}\sum_{i=1}^n \phi(X^{(i)}_s,X^{(i)}_t) -
\mathbb{E}_{\theta^*} [\phi(X_s,X_t)] \right| + \left|
\frac{1}{n}\sum_{i=1}^n
\mathbb{E}_{\theta^*_s}[\phi(X_s,X^{(i)}_t)\mid X^{(i)}_{\s}] -
\mathbb{E}_{\theta^*} [\phi(X_s,X_t)] \right| \nonumber.
\end{eqnarray}
Therefore, for any $\varepsilon \le 2\sigma^2\zeta$,
\begin{eqnarray}
\mathbb{P}\left(\left|\frac{\partial \tilde L(\theta_s;
\mathbb{X}_n)}{\partial \theta_{st}}\right|> \varepsilon \right)
&\le& \mathbb{P}\left(\left|\frac{1}{n}\sum_{i=1}^n
\phi(X^{(i)}_s,X^{(i)}_t)- \mathbb{E}_{\theta^*} [\phi(X_s,X_t)] \right| > \frac{\varepsilon}{2}\right) \nonumber \\
&& + \mathbb{P}\left(\left| \frac{1}{n}\sum_{i=1}^n
\mathbb{E}_{\theta^*_s}[\phi(X_s,X^{(i)}_t)\mid X^{(i)}_{\s}] -
\mathbb{E}_{\theta^*} [\phi(X_s,X_t)] \right|> \frac{\varepsilon}{2} \right) \nonumber \\
&\le&  4 \exp\left\{-\frac{n\varepsilon^2}{8\sigma^2}\right\}
\nonumber,
\end{eqnarray}
where the last ``$\le$'' follows from Lemma~\ref{lemma_tail} and
Lemma~\ref{lemma_tail_0}. By the union  bound we obtain
\begin{eqnarray}
\mathbb{P}(\|\nabla \tilde L(\theta^*_s;\mathbb{X}_n)\|_\infty >
\varepsilon) &\le&
4p\exp\left\{-\frac{n\varepsilon^2}{8\sigma^2}\right\} \nonumber.
\end{eqnarray}
Let us choose $\varepsilon = 2\sigma\sqrt{6\ln p/n}$. Since $n >
6\ln p/(\sigma^2\zeta^2)$, we have $ \varepsilon < 2\sigma^2\zeta$.
We conclude that with probability at least $1-4p^{-2}$,
\[
\|\nabla \tilde L(\theta^*_s;\mathbb{X}_n)\|_\infty \le
2\sigma\sqrt{6\ln p/n}.
\]
This proves the desired bound.
\end{proof}

\end{document}